\documentclass[11pt,twoside]{article}
\usepackage{fullpage}
\usepackage{epsf}
\usepackage{fancyhdr}
\usepackage{graphics}
\usepackage{graphicx}
\usepackage{psfrag}
\usepackage{microtype}
\usepackage{subfigure}
\usepackage{algorithmic}
\usepackage{color,xcolor}
\usepackage[linesnumbered,ruled]{algorithm2e}
\DontPrintSemicolon
\usepackage{color}
\usepackage{tabularx}
\usepackage{amsthm}
\usepackage{amsfonts}
\usepackage{amsmath}
\usepackage{amssymb,bbm}
\usepackage{stackengine}
\usepackage{footnote}
\makesavenoteenv{tabular}
\makesavenoteenv{table}

\newcommand{\filtration}{\mathcal{F}}

\usepackage{thmtools, thm-restate}
\newcommand{\asymconditioning}{\kappa}

\newcommand{\transition}{\mathcal{T}}

\newcommand{\stepsize}{\eta}

\newcommand{\hc}{\mathsf{H}}
\newcommand{\law}{\mathcal{L}}

\newcommand{\real}{\ensuremath{\mathbb{R}}}
\newcommand{\complex}{\ensuremath{\mathbb{C}}}

\newcommand{\order}[1]{\ensuremath{\mathcal{O}\parenth{#1}}}

\newcommand{\thetastar}{\ensuremath{{\theta^*}}}
\newcommand{\thetahat}{\ensuremath{\widehat{\theta}}}
\newcommand{\thetatil}{\ensuremath{\widetilde{\theta}}}
\newcommand{\thetabar}{\ensuremath{\bar{\theta}}}

\newcommand{\parenth}[1]{\left( #1 \right)}

\newcommand{\abss}[1]{\left| #1 \right |}

\newcommand{\ball}{\ensuremath{\mathbb{B}}}
\newcommand{\sphere}{\ensuremath{\mathbb{S}}}

\newcommand{\mydefn}{\ensuremath{:=}}

\newcommand{\defn}{:=}

\newcommand{\matsnorm}[2]{|\!|\!| #1 | \! | \!|_{{#2}}}
\newcommand{\vecnorm}[2]{\left\| #1\right\|_{#2}}

\newcommand{\opnorm}[1]{\ensuremath{\matsnorm{#1}{\tiny{\mbox{op}}}}}

\newcommand{\inprod}[2]{\ensuremath{\langle #1 , \, #2 \rangle}}

\newcommand{\Exs}{\ensuremath{{\mathbb{E}}}}
\newcommand{\Prob}{\ensuremath{{\mathbb{P}}}}

\DeclareMathOperator{\diag}{diag}
\DeclareMathOperator{\var}{var}
\DeclareMathOperator{\cov}{cov}
\DeclareMathOperator{\trace}{trace}

\newtheoremstyle{named}{}{}{\itshape}{}{\bfseries}{.}{.5em}{\thmnote{#3's }#1}
\theoremstyle{named}

\theoremstyle{plain}

\newtheorem{theorem}{Theorem}
\newtheorem{proposition}{Proposition}
\newtheorem{lemma}{Lemma}

\newtheorem{corollary}{Corollary}

\newlength{\widebarargwidth}
\newlength{\widebarargheight}
\newlength{\widebarargdepth}

\makeatletter
\long\def\@makecaption#1#2{
        \vskip 0.8ex
        \setbox\@tempboxa\hbox{\small {\bf #1:} #2}
        \parindent 1.5em  
        \dimen0=\hsize
        \advance\dimen0 by -3em
        \ifdim \wd\@tempboxa >\dimen0
                \hbox to \hsize{
                        \parindent 0em
                        \hfil
                        \parbox{\dimen0}{\def\baselinestretch{0.96}\small
                                {\bf #1.} #2
                                }
                        \hfil}
        \else \hbox to \hsize{\hfil \box\@tempboxa \hfil}
        \fi
        }
\makeatother

\long\def\comment#1{}
\definecolor{battleshipgrey}{rgb}{0.52, 0.52, 0.51}
\definecolor{darkgray}{rgb}{0.66, 0.66, 0.66}
\definecolor{darkgreen}{rgb}{0.0, 0.2, 0.13}
\definecolor{darkspringgreen}{rgb}{0.09, 0.45, 0.27}
\definecolor{dukeblue}{rgb}{0.0, 0.0, 0.61}
\definecolor{olivedrab7}{rgb}{0.24, 0.2, 0.12}
\definecolor{darkblue}{rgb}{0.0, 0.0, 0.55}
\definecolor{darkscarlet}{rgb}{0.34, 0.01, 0.1}
\definecolor{candyapplered}{rgb}{1.0, 0.03, 0.0}
\definecolor{ao(english)}{rgb}{0.0, 0.5, 0.0}
\definecolor{applegreen}{rgb}{0.55, 0.71, 0.0}

\usepackage[numbers]{natbib}
\usepackage{url}
\usepackage[colorlinks,linkcolor=magenta,citecolor=blue, pagebackref=true,backref=true]{hyperref}

\usepackage{nicefrac}
\usepackage{comment}

\usepackage{chngpage}

 \usepackage{tabularx}%

\usepackage{enumitem}
\usepackage{booktabs}
\usepackage{caption}

\usepackage{bm,bbm}
\usepackage{mathtools}

\newcommand{\spectralgap}{{\lambda^*}}
\newtheorem{assumption}{Assumption}

\theoremstyle{definition}
\newtheorem{example}{Example}
\newcommand{\Wass}{\ensuremath{\mathcal{W}}}

\newcommand{\bvec}{\ensuremath{b}}

\newcommand{\Amat}{\ensuremath{A}}
\newcommand{\Umat}{\ensuremath{U}}
\newcommand{\IdMat}{\ensuremath{I_d}}

\newcommand{\Amatbar}{\ensuremath{\bar{\Amat}}}
\newcommand{\bvecbar}{\ensuremath{\bar{\bvec}}}
  
\newcommand{\AmatTil}{\tilde{\Amat}}

\newcommand{\bvectil}{\tilde{\bvec}}

\newcommand{\RealPart}{\ensuremath{\mathrm{Re}}}
\renewcommand{\Vec}{\ensuremath{\mathrm{vec}}}

\newcommand{\lammin}{\ensuremath{\lambda_{\operatorname{min}}}}

\newcommand{\markovtransition}{\mathcal{T}}
\newcommand{\Pmat}{\ensuremath{P}}

\newcommand{\discount}{\ensuremath{\gamma}}
\newcommand{\reward}{\ensuremath{r}}

\newcommand{\SigStar}{\ensuremath{\Sigma^*}}

\newcommand{\NoiseAplain}{\ensuremath{\Xi}}
\newcommand{\NoiseAt}{\ensuremath{\NoiseAplain_t}}
\newcommand{\NoiseA}{\ensuremath{\NoiseAplain_\Amat}}

\newcommand{\noisebplain}{\ensuremath{\xi}}
\newcommand{\noiseb}{\ensuremath{\noisebplain_\bvec}}
\newcommand{\noisebt}{\ensuremath{\noisebplain_t}}

\newcommand{\sigsqA}{\ensuremath{v^2_\Amat}}
\newcommand{\sigsqb}{\ensuremath{v^2_\bvec}}

\newcommand{\coordinate}{a}
\newcommand{\testvector}{v}
\newcommand{\LamMat}{\ensuremath{\Lambda}}
\newcommand{\LamStar}{\ensuremath{\LamMat^*_\stepsize}}

\newcommand{\Term}{\ensuremath{T}}

\makeatletter
\long\def\@makecaption#1#2{
        \vskip 0.8ex
        \setbox\@tempboxa\hbox{\small {\bf #1:} #2}
        \parindent 1.5em  
        \dimen0=\hsize
        \advance\dimen0 by -3em
        \ifdim \wd\@tempboxa >\dimen0
                \hbox to \hsize{
                        \parindent 0em
                        \hfil 
                        \parbox{\dimen0}{\def\baselinestretch{0.96}\small
                                {\bf #1.} #2
                                } 
                        \hfil}
        \else \hbox to \hsize{\hfil \box\@tempboxa \hfil}
        \fi
        }
\makeatother

\newtheorem{assump}{Assumption}

\newcommand{\usedim}{\ensuremath{d}}
\newcommand{\MRPstationary}{\mu}

\newcommand{\goodendex}{\ensuremath{\clubsuit}}

\newcommand{\LinSpace}{\ensuremath{\mathcal{L}_\phi}}
\newcommand{\TDPROJ}{\ensuremath{\Pi_{\LinSpace, \mu}}}

\newcommand{\Vstar}{\ensuremath{V^*}}

\newcommand{\Mgame}{\ensuremath{P}}

\newcommand{\GamStar}{\ensuremath{\Gamma^*}}

\newcommand{\ConfSet}{\ensuremath{\mathcal{E}}}

\newcommand{\Sphere}[1]{\ensuremath{\mathbb{S}}}

\newcommand{\linftygap}{\bar{\lambda}}

\newcommand{\sigmax}{\ensuremath{\sigma_{\mbox{\tiny{max}}}}}

\newcommand{\Abar}{\ensuremath{\bar{A}}}

\begin{document}

\begin{center}
{\bf{\Large{On Linear Stochastic Approximation: Fine-grained\\[.2cm]
      Polyak-Ruppert and Non-Asymptotic Concentration }}}

\vspace*{.2in}
 {\large{
 \begin{tabular}{ccc}
  Wenlong Mou$^{\diamond}$ & Chris Junchi Li$^{\diamond}$ &  Martin J. Wainwright$^{\diamond, \dagger, \ddagger}$ \\
 \end{tabular}
 \begin{tabular}
 {cc}
  Peter Bartlett$^{\diamond, \dagger}$ & Michael I. Jordan$^{\diamond, \dagger}$
 \end{tabular}

}}

\vspace*{.2in}

 \begin{tabular}{c}
 Department of Electrical Engineering and Computer
 Sciences$^\diamond$\\ Department of Statistics$^\dagger$ \\ UC
 Berkeley\\
 \end{tabular}

 \vspace*{.1in}
 \begin{tabular}{c}
 Voleon Group$^\ddagger$
 \end{tabular}

\vspace*{.2in}

\today

\vspace*{.2in}

\begin{abstract}
  We undertake a precise study of the asymptotic and non-asymptotic
  properties of stochastic approximation procedures with
  Polyak-Ruppert averaging for solving a linear system $\bar{A} \theta
  = \bar{b}$. When the matrix $\bar{A}$ is Hurwitz, we prove a central
  limit theorem (CLT) for the averaged iterates with fixed step size
  and number of iterations going to infinity. The CLT characterizes
  the exact asymptotic covariance matrix, which is the sum of the
  classical Polyak-Ruppert covariance and a correction term that
  scales with the step size. Under assumptions on the tail of the
  noise distribution, we prove a non-asymptotic concentration inequality whose
  main term matches the covariance in CLT in any direction, up to
  universal constants. When the matrix $\bar{A}$ is not Hurwitz but
  only has non-negative real parts in its eigenvalues, we prove that
  the averaged LSA procedure actually achieves an $O(1/T)$ rate in
  mean-squared error. Our results provide a more refined understanding
  of linear stochastic approximation in both the asymptotic and
  non-asymptotic settings.  We also show various applications of the
  main results, including the study of momentum-based stochastic
  gradient methods as well as temporal difference algorithms in
  reinforcement learning.
\end{abstract}
\end{center}


\section{Introduction}

Fixed-point algorithms based on stochastic approximation (SA) play a
central role in a wide variety of
disciplines~\citep{robbins1951stochastic,BERTSEKAS-TSITSIKLIS,bottou2016optimization,lai2003stochastic}.
In general, given the goal of solving an underlying deterministic
fixed-point equation, SA methods perform updates based on
randomized approximations to the current residual.  An important
special case is provided by stochastic gradient methods for
optimization, which play an increasingly important role in
large-scale machine learning and
statistics~\citep{nemirovski2009robust, moulines2011non}.

Moving beyond the setting of optimization, there are many other kinds
of problems in which stochastic approximation is a workhorse.  For
example, many problems in reinforcement learning involve the solution
of fixed-point equations, and algorithms like
TD~\citep{sutton1988learning} and Q-learning~\citep{watkins1992q}
solve them via stochastic approximation. Moreover, even for stochastic
optimization, accelerated methods that include momentum terms in their
updates involve non-symmetric operators, and so require more
general SA techniques for their analysis.

The celebrated Polyak-Ruppert averaging
procedure~\citep{polyak1992acceleration,ruppert1988efficient}
stabilizes and accelerates stochastic approximation algorithms by
taking an average over iterates. It is known that for suitably decaying
step sizes, a central limit theorem (CLT) can be established
for the averaged iterates. Moreover, Polyak-Ruppert averaging can achieve an optimal covariance, in the sense of local asymptotic minimaxity. Asymptotic results of this kind have provided the underpinnings for the development of online statistical inference methods. Recently, numerous non-asymptotic
results have also been established in the settings of stochastic
optimization (see Section~\ref{sec:related-works}). Notably, the
papers~\cite{ nemirovski2009robust, moulines2011non} give
non-asymptotic bounds for stochastic gradient methods as applied to
weakly convex or strongly convex objectives; here the main term
depends on the trace of the optimal covariance matrix.

There remains, however, a major mismatch between the classical CLTs
and the non-asymptotic rates. Though the non-asymptotic results are
valid for a finite number of iterations and are more reliable, they do lose some of the quantitative aspects of the CLT results. In particular, bounds on mean square error give much less information than the optimal covariance matrix, and the lack of high-probability bounds make them inapplicable in important applications such as policy evaluation. On the other hand, many important effects can vanish when the asymptotic limit is taken. In general, the trade-off between asymptotic limits and the rate of approach to asymptotic limits can be crucial.  Such trade-offs should reflect the effect of the step size, and provide guidance
for step-size selection.

In this paper, we consider the problem of linear stochastic
approximation, where the goal is to solve a system, $\Amatbar \theta =
\bvecbar$, of linear equations from noisy observations $(\Amat_t,
\bvec_t)_{t = 1}^{\infty}$. This problem is not only of intrinsic interesting, with in areas such as linear regression and TD learning, but it also has significant applications to nonlinear stochastic approximation problems, where analysis generally proceeds via local linearization. 

In this paper, we make three primary contributions.  First, we
characterize the asymptotic covariance for the averaged iterates in
Polyak-Ruppert procedure for constant step size linear stochastic
approximation. In addition to the classical $\Amatbar^{-1} \Sigma
(\Amatbar^{-1})^\top$ term, we find a correction term that depends on
the step size. A central limit theorem is shown for the averaged
a constant step-size procedure. Second, under stronger tail assumptions,
we show a non-asymptotic concentration inequality for the averaged
iterates in any direction, the leading term of which is the asymptotic
covariance at this direction, while other terms keep the optimal
rates. Thus, we achieve the best of both worlds. Finally, we show that even
if the matrix $A$ is not Hurwitz, as long as the real part of
eigenvalues are non-negative, a non-asymptotic second moment bound is
still valid for the Polyak-Ruppert procedure, again yielding a $1 /
\sqrt{T}$ rate. This goes beyond the regime of stable dynamical
systems, and completes the picture of possibilities and
impossibilities for linear stochastic approximation. When applied to
momentum-based stochastic gradient descent (SGD) and temporal difference (TD) learning for value function estimation, our
results capture many interesting phenomena, including the acceleration
effect of momentum-based SGD, instance-dependent $\ell_\infty$ bounds for
policy evaluation with near-optimal rates, and gap-independent results
for the average-reward TD algorithm.

\paragraph{Technical overview:}

Similar to past
work~\cite{polyak1992acceleration,ruppert1988efficient}, our analysis
is based on representing the term $\Amatbar (\thetabar_T -
\thetastar)$ using a martingale to account for the noise at each
step.  Our setting involves additional noise terms, due to the
stochasticity in our observations of the matrix $\Amatbar$.  As a
consequence, the conditional covariance of the martingale difference terms
at each step are dependent on the current iterate $\theta_t$.
Handling this issue requires the ergodicity of $\{\theta_t \}_{t \geq
  0}$ as a Markov chain. Having established ergodicity, we can then
prove an asymptotic result by combining Lindeberg-type CLTs with
ergodic theorems.

In order to move from the asymptotic to the non-asymptotic setting, we
study the projection of the iterate $\theta_T$, for each time $T$,  in
some fixed but arbitrary direction. We can then apply the
Burkholder-Davis-Gundy inequality to the higher moments of the
supremum of a martingale, which separates the leading variance term
and other terms that vanish at faster rates in $T$. Similar to the
asymptotic case, the concentration results require a non-asymptotic
bound on the deviation of the empirical averages of a function along a Markov
chain, when compared to an expectation under the stationary distribution.  In order
to obtain such a bound, we exploit metric ergodic concentration
inequalities~\cite{joulin2010curvature} combined with a coupling
estimate.

In the case when the matrix $\Amatbar$ is not Hurwitz but has
non-negative real parts in its eigenvalues, the process
$\{\theta_t\}_{t \geq 0}$ does not generally approach $\thetastar$. In
the critical case, the dynamics is governed by a pure rotation with
stochastic terms diffusing in all directions. However, when averaging is applied, both the effect of rotation and the random noise can be controlled. The step size is chosen to decay at the faster rate
$1/\sqrt{T}$ in order to prevent an exponenential blowup.


\subsection{Related work}
\label{sec:related-works}

In the past decade, the growth of interest in stochastic
gradient descent (SGD) has revived both theoretical and applied interest in
stochastic approximation.  There is a long line of work on the
asymptotic regime of stochastic approximation
algorithms~\citep{ruppert1988efficient,polyak1992acceleration,KUSHNER-YIN,BORKAR,BENVENISTE-METIVIER-PRIOURET,li2018statistical}.
One core idea is that of averaging iterates along the path, which can
be shown to have favorable statistical properties in the asymptotic
setting~\cite{ruppert1988efficient,polyak1992acceleration}.\footnote{See, for instance, Theorem 1 in~\cite{ruppert1988efficient}.}
More recent
papers~\cite{chen2016statistical,su2018uncertainty,liang2019statistical,li2018statistical}
have developed iterative algorithms for constructing asymptotically
valid confidence intervals for statistical problems, as well as
non-asymptotic intervals obtained via Berry-Esseen-type corrections.

In addition to asymptotic results, there are also a wide range of
non-asymptotic results for stochastic approximation algorithms
(see, e.g.,~\citep{nemirovski2009robust,rakhlin2012making,wang2016stochastic,dieuleveut2017bridging,dieuleveut2017harder,jain2017parallelizing,jain2018accelerating,jain2019making,lakshminarayanan2018linear}).
Perhaps most closely related to our work is the analysis of
Lakshminarayanan and
{Sz}epesv{\'a}ri~\cite{lakshminarayanan2018linear}, who study linear
stochastic approximation with constant step sizes combined with
Polyak-Ruppert averaging. Relative to the analysis given here, their
bounds are looser, with sub-optimal dependence on problem-specific
constants, and do not characterize the effect of the choice of the step size.

Several bounds have been established on function values in stochastic
optimization. After processing $N$ samples, the averaged iterate
enjoys an $O(1/N)$ and $O(1/\sqrt{N})$ optimization error bounds for
strongly convex and convex objectives
\citep{nemirovski2009robust,rakhlin2012making,shamir2013stochastic}.
Such optimization error bounds are optimal in the sense that they
match the statistical lower bounds under a stochastic first-order
oracle~\citep{agarwal2012information,NEMIROVSKII-YUDIN}.  Dieuleveut
et al.~\cite{dieuleveut2017harder} studied a momentum accelerated
stochastic gradient scheme with appropriate regularization, proving
its optimality in the critical case.  Nevertheless when applied to
(often high-dimensional) statistical models with specific
distributional assumptions, the aforementioned sharp results often
lose essential statistical information due to their coarse-grained nature.

Stochastic approximation methods have also been widely applied in
reinforcement learning; in particular, TD
learning~\citep{sutton1988learning} and
Q-learning~\citep{watkins1992q} are based on linear and nonlinear
stochastic approximation updates for policy evaluation and
$Q$-function learning, respectively. It should be noted that the
various Bellman-type operators arising in RL do not correspond to
gradients of functions, so that the analysis requires different
techniques from stochastic optimization.  A recent line of work has
focused on the non-asymptotic analysis of TD learning and $Q$-learning
algorithms.  Bhandari et al.~\cite{bhandari2018finite} studied TD with
linear function approximation and established bounds on the
mean-squared error.
Wainwright~\cite{wainwright2019stochastic,wainwright2019variance}
analyzed $Q$-learning as a special case of a cone-contractive
operator, and established sharp $\ell_\infty$-norm bounds, both for
ordinary $Q$-learning and a variance-reduced version thereof. Karimi
et al.~\cite{karimi2019non} studied general biased stochastic
approximation procedures, in particular proving convergence of online
EM and policy gradient methods.

Additional perspectives and variations on stochastic approximation appear
in the literature, with improved non-asymptotic convergence properties in particular cases.
Recent work also studies tail averaging with
parallelization~\citep{jain2017parallelizing}, momentum-based
schemes~\citep{jain2018accelerating,dieuleveut2017harder}, Markov
chain perspectives \citep{dieuleveut2017bridging}, variational Bayesian
perspectives~\citep{mandt2017stochastic} and diffusion approximation
perspectives~\citep{fan2018statistical}.  There is also significant work on last-iterate SGD \citep{jain2019making} and variance-reduced
estimators (see, e.g.,~\cite{roux2012stochastic,johnson2013accelerating,defazio2014saga}).
Our discussion of these variants is limited in this paper; it will be
interesting to study whether these variants can be shown to have the desirable statistical properties that we uncover here under a similar set of assumptions.


\section{Background and problem formulation}
\label{SecBackground}

We begin by introducing the stochastic approximation
algorithm to be analyzed in this paper, along with discussion of some
of its applications.  In the final subsection, we collect some
notation to be used throughout the paper.


\subsection{Linear stochastic approximation}

In this paper, we study stochastic approximation procedures for
solving a linear system of the form $\Amatbar \theta = \bvecbar$,
where the deterministic quantities $\Amatbar \in \real^{d \times d}$
and $\bvecbar \in \real^d$ are parameters of the problem. Throughout
the paper, we assume that the matrix $\Amatbar$ is invertible, so that
the solution $\thetastar$ to the equation exists and is unique.
Suppose that we can observe a sequence of random variables of the form $\{(\Amat_t,
\bvec_t)\}_{t \geq 1}$, assumed to be independent and identically distributed ($\mathrm{i.i.d.}$), and exhibiting an unbiasedness property:
\begin{align}
  \label{EqnUnbiased}
  \Exs (\Amat_{t} \mid \filtration_{t-1}) = \Amatbar, \quad \mbox{and}
  \quad \Exs (\bvec_{t} \mid \filtration_{t-1}) = \bvecbar,
\end{align}
where $\filtration_{t-1}$ denotes the $\sigma$-field generated by
$\{(\Amat_k, \bvec_k \}_{k=1}^{t-1}$.  Given observations of this
form, our goal is to form an estimate $\thetahat$ of the solution
vector $\thetastar$.  For some given initial vector $\theta_0$, we consider the following linear stochastic approximation (LSA) procedure:
\begin{align}
\label{eq:lsa}
\theta_{t + 1} = \theta_t - \stepsize (\Amat_{t + 1} \theta_t - \bvec_{t + 1}),
\quad \mbox{for $t = 0, 1, 2, \ldots$},
\end{align}
where $\stepsize > 0$ is a pre-specified step size.  Our focus will be the Polyak-Ruppert averaged sequence $\{\thetabar_T\}_{T \geq
  1}$ given by
\begin{align}
  \label{EqnPolyakRuppert}
  \thetabar_T & \mydefn \frac{1}{T} \sum_{t = 0}^{T - 1} \theta_t.
\end{align}
In particular, our goals are to establish guarantees for the renormalized
error sequence $\sqrt{T} (\thetabar_T - \thetastar)$, both in an
asymptotic (i.e., $T \rightarrow \infty$) and non-asymptotic (i.e., finite $T$) setting.


\subsection{Some motivating examples}
\label{SecExamples}

Let us consider some applications that motivate the analysis of this
paper.  We begin with the simple example of stochastic gradient
methods for linear regression:

\begin{example}[Stochastic gradient methods for linear regression]
\label{ExaSGD}  
Let $X \in \real^d$ be a vector of features, and let $Y \in \real$ be
a scalar response.  A linear predictor of $Y$ based on $X$ takes the
form $\inprod{X}{\theta} = \sum_{j=1}^d X_j \theta_j$ for some weight
vector $\theta \in \real^d$.  If we view the pair $(X, Y)$ as random,
we can consider a vector $\thetastar$ that is optimal in the sense of
minimizing the mean-squared error of the prediction---that is,
  \begin{align}
    \thetastar \in \arg \min_{\theta \in \real^d} \Exs \Big[ \Big(Y -
      \inprod{X}{\theta} \Big)^2 \Big],
  \end{align}
  where $\Exs$ denote expectation over the joint distribution of
  $(X,Y)$.  A straightforward computation yields that $\thetastar$
  must be a solution of the linear system $\Amat \theta = \bvec$,
  where $\Amatbar \defn \Exs[X X^\top] \in \real^{d \times d}$ and
  $\bvecbar \defn \Exs[X Y] \in \real^d$.  Note that $\thetastar$
  exists and is unique whenever $\Amatbar$ is strictly positive
  definite.

  In practice, we do not know the joint distribution of $(X,Y)$, but
  might have access to a sequence of paired observations, say $\{(X_t,
  Y_t) \}_{t \geq 1}$, i.i.d. across different time instances $t$.
  The standard SGD algorithm computes an estimate of $\thetastar$ via
  the recursive update
  \begin{align}
    \theta_{t+1} & = \theta_t - \stepsize X_{t + 1} \big(
    \inprod{X_{t + 1}}{\theta_t} - Y_{t + 1} \big) \quad \mbox{for $t = 0, 1, 2
      \ldots$.}
  \end{align}
Note that this update is a special case of the general linear
update~\eqref{eq:lsa}, with the choices $\Amat_t = X_t X_t^T$ and
$\bvec_t = X_t Y_t$.  \hfill \goodendex
\end{example}

As a continuation of the previous example, let us consider a more
sophisticated algorithm for online linear regression, one based on the
introduction of an additional momentum component.

\begin{example}[Stochastic gradient with momentum]
  \label{ExaMomentumSGD}
For this particular example, let us adopt the shorthand $\Amat_t = X_t
X_t^T$ and $\bvec_t = X_t Y_t$.  Given a step size $\stepsize > 0$ and
a momentum term $\alpha > 0$, consider a recursion over a pair
$(\theta_t, v_t) \in \real^\usedim \times \real^\usedim$, of the
following form:
\begin{align*}
\begin{cases}
\theta_{t + 1} = \theta_t - \stepsize v_t \\ v_{t + 1} = v_t -
\stepsize \alpha v_t + \stepsize (\Amat_{t + 1} \theta_{t + 1} - \bvec_{t + 1}).
\end{cases}
\end{align*}
Let us reformulate these updates in the form~\eqref{eq:lsa}, where we lift the problem to
dimension $2 d$ and use a tilde to denote lifted quantities.  After some
algebra, we find that the algorithm can be formulated as an update of the $2d$-dimensional
vector $\thetatil_t \mydefn \begin{bmatrix} \theta_t & v_t
\end{bmatrix}^T \in \real^{2d}$ according to the recursion~\eqref{eq:lsa}, where
\begin{align*}
\AmatTil_t \mydefn
\begin{bmatrix}
    0 & \IdMat \\ - \Amat_t & \alpha \IdMat + \stepsize \Amat_t
\end{bmatrix}, \quad \mbox{and} \quad
 \tilde{b}_t \mydefn
\begin{bmatrix}
    0 \\ -\bvec_t
\end{bmatrix}.
\end{align*}
The underlying deterministic problem is to solve the $2d$-dimensional
linear system $\AmatTil \thetatil = \bvectil$, where $\AmatTil =
\Exs[\AmatTil_t]$ and $\bvectil = \Exs[\bvectil_t]$.  It can be seen
that $\thetastar \in \real^d$ is a solution to the original problem if
and only if the vector $\thetatil^* \mydefn \begin{bmatrix} \theta^* &
  0 \end{bmatrix}^T$ is a solution to the lifted problem.  In the
sequel, we will use our general theoretical results to show why the
addition of the momentum term can be beneficial.  \hfill \goodendex
\end{example}

\noindent The area of stochastic control and reinforcement learning is
another fertile source of stochastic approximation algorithms, and we
devote our next two examples to the problems of exact and approximate
policy evaluation.

\begin{example}[TD algorithms in reinforcement learning]
  \label{ExaTD}
We now describe how the \mbox{TD$(0)$-algorithm} in reinforcement
learning can be seen as an instance of the update~\eqref{eq:lsa}.  In
this example, we discuss the TD algorithm for exact policy evaluation;
in Example~\ref{ExaTDLinearFunction} to follow, we discuss the
extension to TD with linear function approximation.

We begin by reviewing the background on Markov reward
processes necessary to describe the problem; see the
books~\cite{Bertsekas_dyn1, Puterman05, SutBar18} for more details.
We focus on a discrete Markov reward process (MRP) with $D$ states;
any such MRP is specified by a pair $(\Pmat, \reward) \in \real^{D
  \times D} \times \real^D$.  The matrix $\Pmat \in \real^{D \times
  D}$ is row-stochastic, with entry $\Pmat_{ij} \in [0,1]$
representing the probability of transitioning to state $j$ from state
$i$.  The vector $r \in \real^D$ is the reward vector, with $r_i$
denoting the reward received when in state $i$.  

\paragraph{Discounted case:}
If future rewards are
discounted with a factor $\discount \in (0,1)$, then the value
function of the Markov reward process is a vector $\thetastar$ that
solves the Bellman equation $\thetastar = \reward + \discount \Pmat
\thetastar$.  This linear equation can be seen as a special case of
our general set-up with
\begin{align*}
\Amatbar \defn I_D - \discount \Pmat, \quad \mbox{and} \quad \bvecbar
\defn \reward,
\end{align*}
where $I_D$ denotes the $D$-dimensional identity matrix.

There are various observation models in reinforcement learning, with
one of the simpler ones being the \emph{generative model}.  In this
setting, at each time $t = 1, 2, \ldots$, we observe the following
quantities:
\begin{itemize}
\item for each state $i \in [D]$, a random reward $R_{t,i}$ that
  is an unbiased estimate of $r_i$ (i.e., $\Exs[R_{t,i}] = r_i$). For simplicity, from now on, we assume that $R_{t, i} \in [-1, 1]$ almost surely, for any $i \in [D]$ and $t \geq 0$.
\item for each state $i \in [D]$, a next state $J$ is drawn randomly
  according to the transition vector $\Pmat_{i, \cdot}$.
\end{itemize}
We place this model in our general LSA framework by
setting $b_t = R_t$ for each time $t$, and defining a random matrix
$A_t \in \{0,1\}^{D \times D}$ with a single one in each row; in
particular, row $i$ contains a $1$ in position $J$, where $J$ was the
randomly drawn next state for $i$.
\hfill \goodendex
\end{example}

\begin{example}[TD Algorithm with linear function approximation]
\label{ExaTDLinearFunction}
In practice, the state space $\mathcal{X}$ can be extremely large or
possibly infinite.  In such settings, the exact approach to policy
evaluation, as described in the previous example, becomes both
computationally infeasible and statistically inefficient.  In
practice, it is typical to combine TD algorithms with a linear
function approximation step.  Suppose that we are given a feature map
$\phi: \mathcal{X} \rightarrow \real^d$.  We consider the set of value
functions $V: \mathcal{X} \rightarrow \real$ that have a linear
parameterization of the form $V_\theta(x) = \inprod{\theta}{\phi(x)}
\; = \; \sum_{j=1}^d \theta_j \phi_j(x)$ for some vector of weights
$\theta \in \real^\usedim$.  We use $\LinSpace$ to denote the
collection of all such linearly parameterized value functions.

In this more general context, the TD$(0)$ algorithm seeks to compute a
particular approximation to the original value function, as we now
describe.  Suppose that the Markov process $(X_t)_{t \geq 0}$ has a
unique stationary distribution $\MRPstationary$, and let $\TDPROJ:
\mathcal{X} \rightarrow \LinSpace$ denote the $L^2(\mu)$-projection
onto the linear space $\LinSpace$---that is $\TDPROJ(V) \defn \arg
\min_{V_\theta \in \LinSpace} \|V - V_\theta\|_{L^2(\mu)}$.  We can
then define the \emph{projected Bellman equation} as
\begin{align}
\label{EqnProjBell}  
  V = \TDPROJ \Big( \reward + \discount \Pmat V \Big),
\end{align}
where $\reward: \mathcal{X} \rightarrow \real$ is the reward function
of the Markov reward process.  It can be shown that this equation has
a unique fixed point $\Vstar$, known as the TD approximation. Since
$\Vstar$ must belong to $\LinSpace$, we can write $\Vstar(x) =
\inprod{\thetastar}{\phi(x)}$ for some $\thetastar \in \real^\usedim$.

With this set-up, we can now describe the more general instantiation
of the TD$(0)$ algorithm, which uses linear stochastic approximation
to solve the projected Bellman equation~\eqref{EqnProjBell}.  Using
the optimality conditions for projection, it can be shown that the
vector $\thetastar$, which characterizes the projected Bellman fixed
point $\Vstar$, must satisfy the linear equation
\begin{align*}
  \Exs (\phi (X) \phi (X)^\top) \thetastar = \Exs (R (X) \phi (X)) +
  \discount \Exs (\phi (X) \phi(X^+)^\top) \thetastar.
\end{align*}
Here the expectations are taken over the joint distribution of a pair
$(X, X^+)$, where $X$ is distributed according to the stationary
distribution $\MRPstationary$, and $X^+$ is drawn from the transition
kernel $\Pmat$ (conditioned on the previous state being $X$).  Thus,
we see that the fixed point $\thetastar$ must satisfy an equation of
the form $\Amatbar \thetastar = \bvecbar$, where
\begin{align*}
\Amatbar \defn \Exs(\phi (X) \phi (X)^\top) - \discount \Exs (\phi (X)
\phi(X^+)^\top), \quad \mbox{and} \quad \bvecbar = \Exs(R (X) \phi
(X)).
\end{align*}

The TD$(0)$ algorithm corresponds to linear stochastic approximation
for solving this equation.  At time $t$, if we are given a triplet
$(X_t, X_t^+, R_t)$, where $X_t$ is distributed according to
$\MRPstationary$; the next state $X_t^+$ is drawn from $\Pmat$
conditioned on the previous state $X_t$, and $R_t$ is a random reward.
We can then run linear stochastic approximation using the quantities
\begin{align}
\Amat_t & = \phi(X_t) \phi(X_t)^T - \discount \phi(X_t) \phi(X^+_t)^T
\quad \mbox{and} \quad b_t = R_t \phi(X_t).
\end{align}
We return to analyze this algorithm in Section~\ref{SecTDFuncApprox}.
\hfill \goodendex
\end{example}

\noindent Finally, we turn to an example of a minimax saddle-point problem~\cite{rockafellar1970monotone}, which has broad application in computational game theory, machine learning and robust statistics (see~\cite{palaniappan2016stochastic} and references therein).
\begin{example}[Minimax games]
We consider a minimax saddle-point problem of the following form:
\begin{align}
\label{EqnMatrixGame}
    \min_{x \in \real^n} \max_{y \in \real^m} \frac{1}{2}
    \left[\begin{matrix} x\\ y \\ 1\end{matrix} \right]^\top \cdot
    \left[
    \begin{matrix}
    \Mgame_{xx} & \Mgame_{xy} & c_x\\ \Mgame_{xy}^\top & \Mgame_{yy} &
    c_y\\ c_x^\top & c_y^\top& 0\end{matrix} \right] \cdot
    \left[\begin{matrix} x\\ y \\ 1\end{matrix} \right].
\end{align}
In a computational game theory setting, for example, the vectors $x \in \real^n$ and $y \in \real^m$ represent the
actions of the two players.  The payoff matrix $\Mgame \in \real^{(n +
  m) \times (n + m)}$ satisfies the PSD conditions $\Mgame_{xx}
\succeq 0$ and $\Mgame_{yy} \preceq 0$, so that the game is of the
convex-concave type.  The matrix game~\eqref{EqnMatrixGame} is a type
of saddle-point problem, and its solution reduces to solving the
linear system
\begin{align}
\begin{bmatrix}
    \Mgame_{xx} & \Mgame_{xy} \\ - \Mgame_{xy}^\top & - \Mgame_{22}
\end{bmatrix}
  \cdot \begin{bmatrix} x \\ y 
  \end{bmatrix} = 
  \begin{bmatrix} - c_x \\ c_y 
  \end{bmatrix}.
\end{align}
Thus, this problem fits into our general set-up with $\Amatbar =
\Mgame$ and $\bvecbar = \begin{bmatrix} -c_x & c_y
\end{bmatrix}^T$, so that $d = n + m$.
Note that the conditions $\Mgame_{xx} \succ 0$ and $\Mgame_{yy} \prec
0$ imply that $\Amat = \Mgame$ is Hurwitz.  The setting of
$\Mgame_{xx} = 0$ and $\Mgame_{yy} = 0$ corresponds to the so-called
critical case.  \hfill \goodendex
\end{example}


\section{Main results and their consequences}
\label{SecMain}

We now turn to the statements of our main results.  We begin with the
easier case when the matrix $\Amatbar$ is Hurwitz (meaning that all
its eigenvalues have a positive real part), and provide both
asymptotic and non-asymptotic guarantees for the Polyak-Ruppert
sequence.  We then turn to the more challenging critical case, in
which the Hurwitz condition is violated (or the eigengap is too small
to be quantitatively useful), and prove bounds on the mean-squared
error.  For all our results, we impose an $\mathrm{i.i.d.}$ condition:
\begin{assumption}
  \label{assume-indp}
The sequences $\{\Amat_t \}_{t \geq 1}$ and $\{\bvec_t\}_{t \geq 1}$
have $\mathrm{i.i.d.}$ entries.
\end{assumption}


\subsection{Hurwitz Case}

This section is devoted to guarantees that hold for a Hurwitz
matrix.
\begin{assumption}
\label{assume-hurwitz}
The matrix $\Amatbar \in \real^{d \times d}$ is Hurwitz, meaning that
\begin{align}
\label{EqnSpectralGap}  
  \spectralgap & \defn \min_{i \in [d]} \RealPart \left( \lambda_i
  (\Amatbar) \right) > 0.
\end{align}
\end{assumption}

Our non-asymptotic statement involves various factors
that pertain to properties that are implied by the Hurwitz condition.  In particular, it is known~\citep{perko2013differential} that any Hurwitz
matrix is similar to a complex matrix $D$ such that $D + D^\hc$ is
positive definite.  Formally, we have:
\begin{lemma}
\label{LemHurwitz}
For any Hurwitz matrix $\Amatbar$, there exists a non-degenerate
matrix $U \in \complex^{d \times d}$ such that $\Amatbar = U D U^{-1}$
for some matrix $D \in \complex^{d \times d}$ that satisfies
\begin{align}
  D + D^\hc \succeq \min_{i \in [d]} \RealPart(\lambda_i(\Amatbar))
  \IdMat.
\end{align}
\end{lemma}
\noindent For completeness, we provide a proof of this known result in
Appendix~\ref{AppHurwitz}.

\subsubsection{An asymptotic guarantee}

We begin with the asymptotic guarantee. In addition to Hurwitz
condition on $\Amatbar$ and the i.i.d. assumption stated previously,
this result requires second-moment control on the noise sequences
$\NoiseAt = \Amat_t - \Amatbar$ and $\noisebt = \bvec_t - \bvecbar$.
\begin{assumption}
\label{assume-second-moment}
There exist finite scalars $\sigsqA$ and $\sigsqb$ such that
\begin{align*}
    \Exs \vecnorm{\NoiseAt u}{2}^{2} \leq \sigsqA, \quad \mbox{and}
    \quad \Exs |\noisebt^\top u|^{2} \leq \sigsqb,
\end{align*}
for any fixed vector $u$ in the Euclidean sphere $\sphere^{d - 1}$.
Moreover, the random elements $\NoiseAt$ and $\noisebt)$ are uncorrelated.
\end{assumption}

With these assumptions in place, we are now ready to state our first
result, which is an asymptotic guarantee.  We let $\NoiseA$ denote a
random matrix following the same distribution as each $\NoiseAt$
variable, and similarly, let $\noiseb$ denote a random vector
following the distribution of each $\noisebt$ vector.  Given these
quantities, we define the following covariance matrix:
\begin{align}
\label{EqnSigStar}  
  \SigStar & \mydefn \cov(\noiseb + \NoiseA \thetastar) \; = \;
  \cov(\noiseb) + \cov (\NoiseA \thetastar).
\end{align}
Note that $\SigStar$ is the sum of the covariances of the two kinds of
noise involved in the stochastic approximation scheme.  Given
$\SigStar$ and $\Amatbar$, we define a linear equation in a matrix
variable $\LamMat$:
\begin{align}
\label{eq:main-stationary-cov}      
        \Amatbar \LamMat + \LamMat \Amatbar^\top - \stepsize \Amatbar
        \LamMat \Amatbar^\top - \stepsize \Exs (\NoiseA \LamMat
        \NoiseA^\top) = \stepsize \SigStar.
    \end{align}
As shown in the sequel
(cf.\ Lemma~\ref{lemma:stationary-existence-uniqueness}), this
matrix equation always has a unique PSD solution, which we denote by
$\LamStar$.  In fact, the matrix $\LamStar$ corresponds to the
covariance matrix of the stationary distribution of the Markov process
$(\theta_t)_{t \geq 0}$.

\begin{theorem}
  \label{thm-asymptotic}
Suppose that the matrix $\Amatbar$ is Hurwitz, the i.i.d.\ condition
(Assumption~\ref{assume-indp}) and the second-moment condition
(Assumption~\ref{assume-second-moment}) hold, and
the random elements $\Amat_t$ and $\bvec_t$ both have finite $(2 +
\delta)$-order moments for some $\delta > 0$. Then there exists a
constant $\stepsize_0 > 0$ such that for any $\stepsize \in \big(0,
\stepsize_0 \big)$, we have
  \begin{align*}
    \sqrt{T} (\thetabar_T - \thetastar) \overset{d}{\rightarrow}
    \mathcal{N} \left(0, \Amatbar^{-1} \big(\Exs [\NoiseA \LamStar
      \NoiseA^\top] + \SigStar \big) (\Amatbar^{-1})^\top \right),
    \end{align*}
  where the $d$-dimensional matrix $\LamStar$ is the unique solution
  to equation~\eqref{eq:main-stationary-cov}.
\end{theorem}

Note that when $\stepsize \rightarrow 0$, then
equation~\eqref{eq:main-stationary-cov} becomes a rescaled version of
the classical Lyapunov equation $\Amatbar \Lambda + \Lambda \Amatbar^T
= \stepsize \Sigma$, the solution of which specifies the stationary
covariance matrix of a stochastic linear system.  For suitably
decaying step sizes, a minor extension\footnote{Such an extension is
  required to handle the randomness in $A_t$ in addition to that in
  $b_t)$.} of arguments due to Polyak and
Juditsky~\cite{polyak1992acceleration} give an asymptotic statement
involving the solution to the classic Lyapunov equation.  On the other
hand, for the constant step-size setting studied here, our result
includes an additional correction term corresponding to the lingering
effect of the non-zero step size.  Theorem~\ref{thm-asymptotic}
specifies the asymptotic covariance matrix in this more general
setting.

When $\stepsize$ is small,
the matrix $\LamStar$ scales linearly with $\stepsize$. The main term
$\Amatbar^{-1} \SigStar (\Amatbar^{-1})^\top$ corresponds to the
asymptotic limit of the classical Polyak-Ruppert averaging
procedure. However, the effect of step size is not fully captured by
the classical CLT. This additional term precisely characterizes the
effect of step size on the asymptotic behavior of the averaged iterates.


\subsubsection{Non-asymptotic concentration}

We now turn to a non-asymptotic concentration result, for which
additional tail conditions need to be imposed on the noise
distribution.  In particular, we replace the second-moment bounds in
Assumption~\ref{assume-second-moment} with the following stronger
conditions:
\setcounter{assump}{2}
\begin{assump}
  \label{assume-noise-subgaussian}
    For some $p \geq 2$, there exist positive scalars $\sigma_A,
    \sigma_b, \alpha, \beta > 0$ such that for any $u$ in the
    Euclidean sphere $\sphere^{d - 1}$, we have
    \begin{align}
\label{EqnPmomentCondition}
      \left( \Exs \vecnorm{ (\Amat_t - \Amatbar) u }{2}^p
      \right)^{\frac{1}{p}} \stackrel{(i)}{\leq} p^{\alpha}
      \sigma_A,\quad \left( \Exs \abss{u^\top (\bvec_t - \bvecbar)}^p
      \right)^{\frac{1}{p}} \stackrel{(ii)}{\leq} p^\beta \sigma_b.
    \end{align}
    Moreover, the noise components $(\NoiseAt$ and $\noisebt)$ are uncorrelated.
\end{assump}
The $p$-moment condition~\eqref{EqnPmomentCondition} with the
parameters $(\alpha, \beta)$ provides a natural generalization of the
notions of sub-Gaussian and sub-exponential tails
(cf.\ Chap. 2,~\cite{wainwright2019high}).  Focusing on the inequality
(ii) in the condition~\eqref{EqnPmomentCondition}, the setting $\beta
= \frac{1}{2}$ corresponds to a vector with sub-Gaussian tails,
whereas the case $\beta = 1$ corresponds to the sub-exponential
case. More generally, if we take the $p$-th power of a sub-Gaussian
random variable, then it satisfies the
condition~\eqref{EqnPmomentCondition} with exponent $2 p$.

Under these conditions, we can prove a result that gives a
concentration guarantee at a given (finite) iteration $T$.  The
guarantee depends on the matrix $U$ from
Assumption~\ref{assume-hurwitz} via its condition number, $\kappa(U) =
\max_{i \in [\usedim]} \sigma_i(U)/\min_{j \in [\usedim]}
\sigma_j(U)$, where $\{\sigma_i(U)\}_{i=1}^d$ are the singular values
of $U$.  For a given iteration $T$ and tolerance parameter $\delta \in
(0,1)$, we require a positive step size $\stepsize$ that satisfies the
bound
\begin{subequations}
\begin{align}
  \label{EqnStepBound}
\stepsize < \frac{\spectralgap}{\rho^2(\Amatbar) + \asymconditioning^2
  (\Umat) \sigma_A^2 \log^{2 \alpha + 1}(T/\delta)},
\end{align}
where $\spectralgap = \min_{i \in [\usedim]}
\RealPart(\lambda_i(\Amatbar)) > 0$ is the spectral gap of $\Amat$,
and $\rho(\Amatbar)$ is its spectral radius.

Our result also involves the asymptotic covariance matrix from
Theorem~\ref{thm-asymptotic}, namely the quantity
\begin{align}
\label{EqnGamStar}
\GamStar(\stepsize) & \mydefn \Amatbar^{-1} \left(\Sigma^* + \Exs
(\NoiseA \LamStar \NoiseA^\top) \right)(\Amatbar^{-1})^\top.
\end{align}
We bound the deviations of the rescaled process $\sqrt{T} (\thetabar_T
- \thetastar)$ in terms of the error term
\begin{equation}
  \begin{aligned}
\label{EqnNasty}
\Delta (T, \delta) \mydefn \; V(\thetastar) \left(\frac{\sigma_A +
  \sigma_B}{T^{1/4}} + \frac{1 +
  \sqrt{\sigma_A/\spectralgap}}{\stepsize \sqrt{T}} \right) \; \log^{2
  \max(\alpha, \beta) + 2} \left(\frac{T}{\delta}\right), \quad
\mbox{where} \\
V(\thetastar) \mydefn \frac{\asymconditioning^2(\Umat)}{ \min_{i \in
    [\usedim]}|\lambda_i (\Amatbar)|} \; \left \{ \vecnorm{\thetastar
  - \theta_0}{2} + \vecnorm{\thetastar}{2} + \sqrt{
  \tfrac{\stepsize}{\spectralgap}} \big (\sigma_A
\vecnorm{\thetastar}{2} + \sigma_b \sqrt{d} \big) \right \}.
  \end{aligned}
\end{equation}
\end{subequations}
With these definitions, we have the following non-asymptotic bound:
\begin{theorem}
\label{thm-non-asymptotic}    
Fix an iteration number $T$ and a tolerance $\delta \in (0,1)$, and
suppose that the i.i.d.\ condition (Assumption~\ref{assume-indp}),
higher-order moment condition
(Assumption~\ref{assume-noise-subgaussian}), and Hurwitz condition all
hold.  Then there exists a constant $c > 0$ such that for any step size
$\stepsize > 0$ satisfying the bound~\eqref{EqnStepBound} and for any
$v \in \sphere^{d - 1}$, we have
\begin{align}
  \Prob \left [ \sqrt{T} \big| v^\top (\thetabar_T - \thetastar) \big|
    \leq c \sqrt{\log(\tfrac{1}{\delta})} \Big \{ \sqrt{ v^\top
      \GamStar(\stepsize) v} + \Delta (T, \delta) \Big \} \right] &
  \geq 1- \delta,
\end{align}
where the asymptotic term $\GamStar(\stepsize)$ and deviation term
$\Delta(T, \delta)$ are defined in equations~\eqref{EqnGamStar}
and~\eqref{EqnNasty}, respectively.
  \end{theorem}

\paragraph{Remarks:} A few comments are in order:
first, we note that the leading term of $\sqrt{v^T \GamStar(\stepsize)
  v}$ of this non-asymptotic bound matches the term arising from the
asymptotic covariance in Theorem~\ref{thm-asymptotic}, up to universal
constants and the $\log(1/\delta)$ term.  Second, although the step
size is required to belong to an interval depending on $T$ and
$\delta$, the dependence is only logarithmic.  In fact, our step-size
condition~\eqref{EqnStepBound}  differs only by these logarithmic
factors from the stability threshold $\frac{\spectralgap}{\rho^2
  (\Amatbar) + \asymconditioning^2 (\Umat) v_A^2}$, assuming $\sigma_A$
and $v_A$ are of the same order.

Second, in the definition of $\Delta(T, \delta)$, observe that the
$\frac{1}{\sqrt{T}}$ term is accompanied by a $\frac{1}{\stepsize}$
dependence, while the $T^{-\frac{1}{4}}$ term does not diverge as
$\stepsize \rightarrow 0^+$.  This behavior is natural, because the
former comes from the ergodicity of the process $\{ \theta_t \}_{t =
  0}^{\infty}$, while the latter comes from the concentration.

Finally, let us consider the issue of how to set the step size
$\stepsize$ as a function of $T$ so as to achieve an optimal bound for
this pre-specfied $T$.  Note that the step-size-dependent term from the
matrix $\GamStar(\stepsize)$ scales linearly in $\stepsize$.
Collecting the terms from $V(\thetastar)$ and $\Delta(T, \delta)$ that
depend on the pair $(T, \stepsize)$, we arrive at a bound that scales
as
\begin{align*}
\underbrace{\stepsize}_{\small{\mbox{From $\GamStar(\stepsize)$}}} +
\underbrace{\sqrt{\stepsize} \left \{ \frac{1}{T^{1/4}} +
  \frac{1}{\stepsize \sqrt{T}} \right \}}_{{\small{\mbox{From
        $V(\thetastar) \Delta(T, \delta)$}}}}.
\end{align*}
In order to minimize this bound, the optimal choice is to set
$\stepsize = T^{-1/3}$, which leads to the overall error scaling as
$T^{-1/3}$.  Thus, with this scaling, we can conclude that
Theorem~\ref{thm-non-asymptotic} guarantees a high-probability bound
of the form
\begin{align*}
\sqrt{T} \big| v^\top (\thetabar_T - \thetastar) \big| \precsim \sqrt{
  v^\top\Amatbar^{-1} (\SigStar) (\Amatbar^{-1})^\top v} + \order{T^{-1/3}},
\end{align*}
where the notation $\precsim$ denotes inequality apart from constants
and logarithmic terms in $(T, \delta)$.

\paragraph{Constructing non-asymptotic confidence sets:}

The classical Polyak-Ruppert procedure gives a locally asymptotically-optimal covariance matrix, which can also be used for the construction
of asymptotic confidence sets.  Theorem~\ref{thm-non-asymptotic} has
analogous consequences for purposes of non-asymptotic inference.  When
going from asymptotically valid inference methods to the
non-asymptotic counterparts, Berry-Esseen-type estimates are often
used. But the sizes of confidence sets constructed in this way have
polynomial dependence on the confidence level $\delta$, even if the
data themselves are not heavy-tailed. When a large number of
confidence sets or tests are needed to be constructed, the size of
each confidence set can expand in a rapid way. In contrast to this
undesirable behavior, we now show how Theorem~\ref{thm-non-asymptotic}
yields a confidence set with better dependence on the confidence
level.

Using the notation of Theorem~\ref{thm-non-asymptotic}, we define the
positive definite matrix
  \begin{align}
\label{EqnEllipseMatrix}
B(T, \delta) & \defn \GamStar(\stepsize) \log (\tfrac{d}{\delta}) +
\Delta(T, \tfrac{\delta}{d}) \IdMat,
  \end{align}
  and the associated weighted Euclidean norm $\|v\|_{B(T, \delta)} =
  \sqrt{v^\top B(T, \delta) v}$.  Using this weighted norm, we then
  define an ellipse that provides us a confidence set that has
  coverage $1-\delta$.
    
\begin{corollary}
\label{prop:non-asymptotic-inference}
Under the conditions of Theorem~\ref{thm-non-asymptotic}, there is a
universal known constant $c > 0$ such that the ellipse
\begin{subequations}
\begin{align}
\ConfSet(T, \delta) & = \left \{ \theta \in \real^d \mid \|\theta -
\thetabar_T\|_{B(T, \delta)} \leq c \, \sqrt{\tfrac{\usedim}{T}}
\right \},
\end{align}
centered at the averaged iterate $\thetabar_T$, has the coverage
guarantee
\begin{align}
\label{EqnConfidenceEllipse}  
  \Prob \Big[ \ConfSet(T, \delta) \ni \thetastar \Big] & \geq 1 -
  \delta.
\end{align}
\end{subequations}
\end{corollary}

From the definition~\eqref{EqnEllipseMatrix} of the ellipse parameters
(recalling the definition of $\Delta(T, \delta)$ from
equation~\eqref{EqnNasty}, it can be seen that the size of our
confidence set depends only logarithmically (as opposed to
polynomially) on $1/\delta$.  In terms of computing the confidence
ellipse $\ConfSet(T, \delta)$, an obstacle is the fact that the the
matrix $\GamStar(\stepsize)$ is unknown (depending on both the unknown
$\Amatbar$, and other aspects of the noise distribution).  However, we
believe that it should be possible to estimate $\GamStar(\stepsize)$
based on the sample path of the algorithm itself.  Notably, in their
study of stochastic gradient methods, Chen et
al.~\cite{chen2016statistical} construct an online estimator for the
asymptotic covariance.  An interesting direction for future work is to
extend estimators of this type to the class of stochastic
approximation procedures considered here.

\subsection{Some extensions beyond the basic setting}

We now turn to some extensions that move beyond the basic setting of
$\ell_2$-bounds when the matrix $\Abar$ is Hurwitz.  We begin in
Section~\ref{SecEllInfty} by deriving some $\ell_\infty$-bounds that
are useful in our subsequent analysis of the TD algorithm.  In
Section~\ref{SecCritical} to follow, we develop a relaxation of the
Hurwitz condition.


\subsubsection{Bounds in the $\ell_\infty$-norm}
\label{SecEllInfty}

In this section, we extend the analysis framework of
Theorem~\ref{thm-non-asymptotic} to the $\ell_\infty$-setting.  Under
somewhat stronger assumption on the linear operator and the noise
distribution, we establish an $\ell_\infty$-bound in which leading term
matches the $\ell_\infty$-norm of the asymptotic distribution in
Theorem~\ref{thm-asymptotic}. Notably, the correction term has only
logarithmic dependence on the dimensionality of the problem, as
opposed to the polynomial dependence in
Theorem~\ref{thm-non-asymptotic}.  This much milder dimension
dependence is important in applications, such as TD algorithms in
reinforcement learning, where the dimension may be very large.

In order to obtain the tight dimension dependence, we impose the
following stronger condition on the noise:
\begin{assumption}
  \label{assume-noise-structure-linfty}
    The stochastic oracles satisfy $\vecnorm{\bvec_t}{\infty} \leq 1$
    and for any $u \in \real^d$, we have $\vecnorm{\Amat_t u}{\infty}
    \leq \vecnorm{u}{\infty}$ almost surely.
\end{assumption}
\noindent In addition, we replace the Hurwitz condition with the
following stronger contraction condition:
\begin{assumption}
  \label{assume-linfty-contraction}
There is a constant $\linftygap > 0$ such that the random matrix $I -
\Amat_t$ is a $(1 - \linftygap)$-contraction under the
$\ell_\infty$-norm, almost surely, meaning that
\begin{align*}
  \vecnorm{(I - \Amat_t) v }{\infty} \leq (1 - \linftygap)
  \vecnorm{v}{\infty} \qquad \mbox{for all $v \in \real^\usedim$.}
\end{align*}
\end{assumption}
\noindent Under Assumption~\ref{assume-noise-structure-linfty}, we are
able to establish an upper bound on each coordinate direction $e_j$,
leading to a high-probability upper bound on $\vecnorm{\thetabar_T -
  \thetastar}{\infty}$.  Naturally, this bound involves the maximal variance
\begin{align*}
  \sigmax^2 \defn \max_{j = 1, \ldots, \usedim} e_j^T
  \GamStar(\stepsize) e_j.
\end{align*}
  
\begin{theorem}\label{thm:non-asymptotic-linfty}
    Fix an iteration number $T$ and a tolerance $\delta \in (0,1)$,
    and suppose that the i.i.d.\ condition
    (Assumption~\ref{assume-indp}), the almost-sure $\ell_\infty$ bound
    condition (Assumption~\ref{assume-noise-structure-linfty}), and
    the almost-sure $\ell_\infty$ contraction condition
    (Assumption~\ref{assume-linfty-contraction}) all hold.  Then there
    exists a constant $c > 0$ such that for any step size $\stepsize >
    0$ satisfying the bound~\eqref{EqnStepBound}, we have
\begin{align*}
    \Prob \left[ \sqrt{T} \vecnorm{\thetabar_T - \thetastar}{\infty}
      \leq c \sqrt{\sigmax^2 \log (\usedim/\delta) } +
      c\frac{\linftygap^{-2}\stepsize + \linftygap^{-1}
      }{T^{\frac{1}{4}}} \sqrt{\log \frac{d}{\delta}} + c\frac{
        \linftygap^{-\frac{5}{2}}}{ \stepsize \sqrt{T}} \right] \geq
    \delta.
\end{align*}
\end{theorem}
We note that the theorem can actually be slightly refined 
by replacing the term $\sigmax^2 \log(\usedim/\delta)$ with the
quantity $Q \left( (e_j^\top \GamStar (\stepsize) \coordinate_j)_{j =
  1}^d; \delta \right)$, where for a vector $v \in \real^\usedim$, we
define
\begin{align}
  \label{EqnQfun}
  Q (v; \delta) \mydefn \inf \big \{ q \: \mid \: \sum_{j = 1}^d e^{ -
    q / v_j} \leq \delta \big \}.
\end{align}


\subsubsection{Critical case}
\label{SecCritical}

In many real-world situations, the Hurwitz assumption may be violated,
or the eigengap can be too small to be useful.  At the population
level, solving the deterministic equation $\Amatbar \theta = \bvec$ is
possible as long as the eigenvalues of $\Amatbar$ are bounded away
from zero.  Thus, it is natural to wonder whether the linear
stochastic approximation scheme~\eqref{eq:lsa} still behaves well
without this assumption. Furthermore, when the spectral gap
$\spectralgap$ is positive but extremely small, does one necessarily
obtain a slow convergence rate?  In this section, we show that the
non-asymptotic rates for LSA remain valid even in the critical case
with no contraction at all.

In this section, we prove a non-asymptotic convergence rate for LSA in
the critical case. We replace the Hurwitz condition on $\Amatbar$
(stated as Assumption~\ref{assume-hurwitz}) with the following
assumption:
\setcounter{assump}{1}
\begin{assump}
\label{assume-critical}
The matrix $\Amatbar$ is diagonalizable, and $\min_{i \in [d]}
\RealPart\left( \lambda_{i} (\Amatbar) \right) \geq 0$.
\end{assump}
The reader might wonder why Assumption~\ref{assume-critical} includes
a diagonalizability condition, which was not needed
before. Unfortunately, unlike the Hurwitz case, the diagonalizability
assumption is unavoidable in the critical case. In particular, the
Polyak-Ruppert procedure is not even consistent when $A$ has purely
imaginary eigenvalues and is non-diagonalizable at the same time, even
in the noiseless case.  We show this with an explicit construction in
Appendix~\ref{AppDiagNeeded}.

\begin{theorem}
\label{thm-critical}
Suppose that the $\mathrm{i.i.d.}$ condition
(Assumption~\ref{assume-indp}), the eigenvalue condition
(Assumption~\ref{assume-critical}), and the second-moment bounds
(Assumption~\ref{assume-second-moment}) all hold. Then, for the
step size \mbox{$\stepsize = \frac{1}{(\rho (\Amatbar) + 3
    \asymconditioning (\Umat) v_A)\sqrt{T}}$,} there is a universal
constant $c$ such that
\begin{align}
  \Exs \vecnorm{\Amatbar \thetabar_T - \bvecbar}{2}^2 & \leq c \;
  \frac{ \asymconditioning^2 (\Umat) (\rho^2 (\Amatbar) +
    \asymconditioning^2 (\Umat) v_A^2) \Exs \vecnorm{\theta_0 -
      \thetastar}{2}^2 + v_b^2 d + v_A^2
    \vecnorm{\thetastar}{2}^2}{T}.
\end{align}
\end{theorem}
Theorem~\ref{thm-critical} is particularly useful in the asymmetric
case, where the eigenvalues of $\Amatbar$ can be complex though the
matrix itself is real. Even if the matrix $\Amatbar$ has an eigenvalue
whose real part is exactly zero but with imaginary part being non-zero,
which is beyond the classical regime of stable dynamical systems, the
$1/T$ rate in mean-squared error is still guaranteed by averaging.
More precisely, we have
\begin{align*}
  \Exs \vecnorm{\thetabar_T - \thetastar}{2}^2 & \leq c \;
  \asymconditioning^2 (\Umat) \frac{ \asymconditioning^2 (\Umat)
    (\rho^2 (\Amatbar) + \asymconditioning^2 (\Umat) v_A^2) \Exs
    \vecnorm{\theta_0 - \thetastar}{2}^2 + v_b^2 d + v_A^2
    \vecnorm{\thetastar}{2}^2}{ \min_{i \in [d]} |\lambda_i
    (\Amatbar)|^2 T}.
\end{align*}
	
Although Theorem~\ref{thm-critical} achieves the correct $O (1 /T)$
rate for mean-squared error, the problem-dependent pre-factor is not
optimal in general. Indeed, a superior problem-dependent rate
$\frac{\sigma_A^2 \vecnorm{\thetastar}{2}^2 + \sigma_b^2d}{T}$ can be
achieved by a plug-in estimator solving $\Abar \hat{\theta} =
\bar{b}$.  In comparison, the initial distance $\Exs
\vecnorm{\thetastar - \theta_0}{2}^2$ appears in
Theorem~\ref{thm-critical}. Intuitively, one can view this term as the
counterpart of the correction term in Theorem~\ref{thm-asymptotic}
when mixing fails. It is also worth noticing that the step size choice
$O(1/\sqrt{T})$ is crucial in this case: larger step size makes the
dynamical system exponentially blow up, and smaller step size leads to
suboptimal rate. That being said, Theorem~\ref{thm-critical} does exhibit
the general effectiveness of LSA as it achieves the optimal $O(1/T)$ rate in
the critical case, with completely online update and $O(d)$ storage.

        
\section{Applications}
\label{sec:applications}
        
In this section, we illustrate the usefulness of our three main
theorems by applying them to some concrete problems, namely the
momentum SGD algorithm discussed in Example~\ref{ExaMomentumSGD} and
the temporal difference (TD) algorithm discussed in
Example~\ref{ExaTD}.


\subsection{Stochastic gradient method with momentum}
  
Recall the SGD with momentum algorithm for linear regression that was
previously introduced in Example~\ref{ExaMomentumSGD}.  In this
section, we use our general theory to analyze it.  As defined in
Example~\ref{ExaMomentumSGD} at the population level the algorithm
involves a matrix $\AmatTil \in \real^{d \times d}$ and vector
$\bvectil \in \real^{2 \usedim}$. At each time $t$, the algorithm
makes use of a pair $(\AmatTil_t, \bvectil_t)$ that are unbiased
estimates of these population quantities.  The momentum SGD update
rule takes the form
\begin{align}
\tilde \theta_{t + 1} = \tilde\theta_t - \stepsize ( \AmatTil_{t + 1} \tilde
\theta_t - \bvectil_{t + 1} ).
\end{align}
Consider the noise variables $\tilde{\NoiseAplain}_t = \AmatTil_t
- \AmatTil$ and $\tilde{\noisebplain}_t = \bvectil_t - \bvectil$.  It
can be seen that they satisfy the same assumptions as $\NoiseAt$ and
$\noisebplain_t$ do, with the constants $(\sqrt{1 + \stepsize^2}
\sigma_A, \sigma_b)$ or $(\sqrt{1 + \stepsize^2} v_A, v_b)$.

The addition of momentum to SGD has two effects: it changes the
mixing time of the process $(\theta_t)_{t \geq 0}$, and it alters the
structure of the asymptotic covariance matrix $\GamStar(\stepsize)$.
The spectrum of $\AmatTil$ plays a central role in these effects;
accordingly, let us investigate the structure of this spectrum.
Suppose that the matrix $\Amatbar$ is positive definite, and let
$\{\lambda_i\}_{i=1}^d$ denote its eigenvalues.

We claim that for any $\alpha \in \real_+ \setminus \{2
\sqrt{\lambda_i} - \stepsize \lambda_i\}_{i = 1}^d$, the matrix
$\AmatTil \in \real^{2d \times 2d}$ is diagonalizable, with paired
(possibly complex) eigenvalues
\begin{align}
\label{EqnSpectrumMomentum}  
\left( \frac{(\alpha + \stepsize \lambda_i) + \sqrt{(\alpha +
    \stepsize \lambda_i)^2 - 4 \lambda_i}}{2}, \; \frac{(\alpha +
  \stepsize \lambda_i) + \sqrt{(\alpha + \stepsize \lambda_i)^2 + 4
    \lambda_i}}{2} \right) \quad \mbox{for $i = 1, \ldots, d$}.
    \end{align}
See Appendix~\ref{AppSpectrumMomentum} for the proof of this claim.

Let us now consider the consequences of the
spectrum~\eqref{EqnSpectrumMomentum} for the mixing rate.  We claim
that when the parameter $\alpha$ is suitably chosen, the mixing
rate of the momentum-based method is faster by a factor of
$1/\sqrt{\lammin(\Amatbar)}$.  Introduce the shorthand
\begin{align*}
\nu_i \mydefn \frac{(\alpha + \stepsize \lambda_i) + \sqrt{(\alpha +
    \stepsize \lambda_i)^2 - 4 \lambda_i}}{2}, \quad \mbox{for $i = 1,
  \ldots, d$.}
\end{align*}
For an index $i$ such that $\alpha > 2 \sqrt{\lambda_i} - \stepsize
\lambda_i$, we have $\nu_i \in \real$, and for index $i$ such that
$\alpha < 2 \sqrt{\lambda_i} - \stepsize \lambda_i$, we have
$\RealPart(\nu_i) = \alpha + \stepsize \lambda_i$. Therefore, for
$\spectralgap = \lammin (\Amatbar)$, we have:
\begin{align*}
    \min_{i} \RealPart (\lambda_i (\AmatTil)) = \begin{cases} \alpha +
      \stepsize \spectralgap - \sqrt{(\alpha + \stepsize
        \spectralgap)^2 - 4 \spectralgap} \geq \frac{2
        \spectralgap}{\alpha + \stepsize \spectralgap}, & \alpha \geq
      2 \sqrt{\spectralgap} - \stepsize \spectralgap\\ \alpha +
      \stepsize \spectralgap, &\alpha < 2 \sqrt{\spectralgap} -
      \stepsize \spectralgap.
    \end{cases}
\end{align*}
When we take $\alpha \asymp \sqrt{\lammin (\Amatbar)}$, we have
$\min_{i} \RealPart (\lambda_i (\AmatTil)) \asymp \sqrt{\lammin
  (\Amatbar)}$.

Now Lemma~\ref{lemma:mixing-hurwitz} implies that for given step size
$\stepsize > 0$, the mixing time is upper bounded by
\begin{align*}
  \frac{1}{\stepsize \min \RealPart(\lambda_i(\AmatTil))} & \asymp
  \frac{1}{\stepsize \sqrt{\lammin(\Amatbar)} }.
\end{align*}
Consequently, the use of momentum speeds up the mixing rate by a
factor of $(1/\sqrt{\lammin(\Amatbar)})$, which is significant in the
regime $\lammin (\Amatbar) \ll 1$.


\subsection{Temporal difference learning}

We discuss the applications of our main theorems in TD learning, in
both exact (Example~\ref{ExaTD}) and linear function approximation
(Example~\ref{ExaTDLinearFunction}) settings. We consider both the
discounted case ($\discount < 1$) as well as the undiscounted case
($\discount =
1$). Theorem~\ref{thm-non-asymptotic},~\ref{thm:non-asymptotic-linfty}
and~\ref{thm-critical} turn out to have nontrivial implications to the
TD algorithm in these cases.


\subsubsection{Analysis of TD without function approximation}
\label{SecTDExact}

We start with the case of exact TD$(0)$.  We follow the model
definition and assumptions in Example~\ref{ExaTD}.

\paragraph{Non-asymptotic bounds in the Hurwitz case}

Recall that the Markov transition kernel matrix $P$ has eigenvalues
with norm at most $1$ and $\discount \in [0, 1)$.  Consequently, the matrix $\Amatbar = I - \discount \Pmat$ has eigenvalues with strictly positive real parts, and so is Hurwitz.  Consequently, we can apply Theorem~\ref{thm-non-asymptotic}, which allows us to obtain high-probability entry-wise bounds and $\ell_\infty$-bounds for policy evaluation.

In order to state the result, we require a few additional pieces of
notation.  Define the $D$-dimensional vector $\sigma^* \in \real^D$ of
standard deviations, with
\begin{align*}
 \sigma^*_j \mydefn \sqrt{\var (R(j)) + \var (Z (j, :) \thetastar)},
 \quad \mbox{for $j = 1, \ldots, D$.}
\end{align*}
Since the rows of $Z_t$ and entries of $R_t$ are independent, the
matrix $\Sigma_*$ in the main term is actually $\diag(\sigma^*
(j)^2)_{j \in [D]}$. It is easy to see that the structure of
stochastic oracles $(\Amat_t, \bvec_t)$ satisfies
Assumption~\ref{assume-noise-structure-linfty} and
Assumption~\ref{assume-linfty-contraction}.
Thus, we can apply
Theorem~\ref{thm:non-asymptotic-linfty}.  Doing so yields a
result that involves the matrix
\begin{align}
  \GamStar (\stepsize) \mydefn (I - \discount P)^{-1} (\mathrm{diag}
  (\sigma^*(j)^2)_{j\in [D]} + \LamStar) (I - \discount
  P^\top)^{-1},
\end{align}
where the matrix $\LamStar$ was defined in
equation~\eqref{eq:main-stationary-cov}.  It also involves the
function $Q$ defined in equation~\eqref{EqnQfun}.

\begin{corollary}
\label{cor:mrp-high-prob}
 Consider the $\mathrm{i.i.d.}$ observational model for Markov reward
 processes defined above. Given a discount factor $\discount \in (0,
 1)$ and a failure probability $\delta > 0$, the averaged TD$(0)$
 algorithm based on step size $\stepsize \in (0,1)$ satisfies the bound
 \begin{align*}
   \sqrt{T} \vecnorm{\hat{\theta}_T - \thetastar}{\infty} \lesssim
   \sqrt{Q (\mathrm{diag} (\GamStar (\stepsize)); \delta)} + T^{-
     \frac{1}{4}} \left( \frac{\stepsize}{(1 - \discount)^2} +
   \frac{1}{1 - \discount}\right) \sqrt{\log \frac{d}{ \delta} } +
   \frac{ T^{- \frac{1}{2}} }{\stepsize (1 - \discount)^{-
       \frac{5}{2}}},
 \end{align*}
 with probability at least $1 - \delta$.
 \end{corollary}
When the step size is chosen to be of order $\stepsize =
O(T^{-\frac{1}{3}})$, the leading term of
Corollary~\ref{cor:mrp-high-prob} is an instance-dependent term that slightly improves upon that of the offline plug-in estimator in~\cite{pananjady2019value}, which was shown to be minimax
optimal.


\paragraph{Critical case: Application of Theorem~\ref{thm-critical}.}

While most of existing results in policy evaluation require the discount
factor to be bounded away from one, our second result certifies that,
even if there is no discount at all (i.e., when $\discount = 1$,
corresponding to the average reward RL setting),
the linear stochastic approximation achieves a $O (1/\sqrt{T})$ error
decay, as long as the error is measured in terms of Bellman error
(i.e., the deficiency in the fixed point relation). Furthermore, for
discounted problems, the results show that the Bellman error can be
bounded independently of the $(1 - \discount)$ factor:
 \begin{corollary}
\label{cor:mrp-no-discount}
Suppose the transition matrix $P$ is diagonalizable with $P = \Umat
D_P \Umat^{-1}$, for $\stepsize = \frac{1}{(1 + 3 \asymconditioning
  (\Umat) v (P) ) \sqrt{T}}$, for any $\discount \in [0, 1]$, we have
\begin{align*}
  \Exs \vecnorm{\thetabar_T - (\discount P \thetabar_T +
    r)}{2}^2 \lesssim \frac{\asymconditioning^2 (\Umat) (1 +
    \asymconditioning^2 (\Umat) v (P)^2) \Exs \vecnorm{\theta_0
      - \thetastar}{2}^2 + v (r)^2 D + v (P)^2
    \vecnorm{\thetastar}{2}^2}{T}.
\end{align*}
 \end{corollary}
In the setting of average reward TD learning, athough the matrix
$\Amatbar = I - \Pmat$ is not invertible, with $\lambda_1 (\Pmat) =
1$, the algorithm is actually restricted to the quotient space
$\real^S / \mathrm{Ker} (\Amatbar)$ (assuming the graph is connected
and consequently no multiplicity of eigenvalue $1$, and
$\mathrm{dim}(\mathrm{Ker} (\Amatbar)) = 1$), by subtracting the
mean~\citep{tsitsiklis2002average}. Moreover, we can still translate the
bound in Bellman error to the parameter estimation
error. Corollary~\ref{cor:mrp-no-discount} implies that:
\begin{align*}
  \Exs \vecnorm{\thetabar_T - \thetastar}{2}^2 = O \left(
  \asymconditioning^2 (\Umat) \frac{ v (r)^2 D + v (P)^2
    \vecnorm{\thetastar}{2}^2 + \asymconditioning^2 (\Umat) (1 +
    \asymconditioning^2 (\Umat) v (P)^2) \Exs \vecnorm{\theta_0 -
      \thetastar}{2}^2}{T \cdot \min_{i \geq 2} |1 - \lambda_i
    (P)|^2} \right),
\end{align*}
where the problem-dependent complexity term is $\min_{i \geq 2} |1 -
\lambda_i (P)|$, as opposed to the real-part of eigengap $\min_{i \geq
  2} (1 - \RealPart(\lambda_i (P) ))$ in the Hurwitz case. In
particular, suppose that the transition matrix $P$ has a complex
eigenvalue of the form $e^{i \alpha}$ for some $\alpha \ll 1$.\footnote{This can happen, for example, in an $N$-state Markov chain
  where the transition from state $i$ is deterministically to the
  state $(i + 1) \mod N$. In such case the eigenvalues are
  $e^{\frac{2\pi k}{N}i}$.}  In this case, we have $\min_{i \geq 2} |1
- \lambda_i (P)| \asymp \alpha$ but $\min_{i \geq 2} (1 -
\RealPart(\lambda_i (P)) ) \asymp \alpha^2$. The dependency on
$\alpha$ in the critical case bound can even be better than the bound
we get by treating the matrix as Hurwitz. Specifically,
Corollary~\ref{cor:mrp-no-discount} yields a bound of order $O(1/
\alpha \sqrt{T})$; on the other hand, although the leading term in
Theorem~\ref{thm-non-asymptotic} is near-optimal, due to the presence
of a $\frac{1}{\stepsize \min_{i \geq 2} |1 - \lambda_i (P)| T }$ term
in the bound, it leads to a $O (1 / \alpha^3 T)$ term, as the step size
has to be chosen such that $\stepsize \lesssim
\alpha^2$. Corollary~\ref{cor:mrp-no-discount} leads to a better
$O(\frac{1}{\alpha^2 \varepsilon^2})$ sample complexity, compared with
the $O(\frac{1}{\alpha^2 \varepsilon^2} + \frac{1}{\alpha^3
  \varepsilon})$ complexity guaranteed by the theorem in Hurwitz case. This
is mainly because the step size choice $\stepsize \lesssim \alpha^2$
suggested by Theorem~\ref{thm-non-asymptotic} is too conservative,
compared to the gap-independent $O (1 / \sqrt{T})$ choice implied by
Theorem~\ref{thm-critical}.


\subsubsection{TD with Linear Function Approximation}
\label{SecTDFuncApprox}

We now consider an application of Theorem~\ref{thm-non-asymptotic} and
Theorem~\ref{thm-critical} to the use of the TD algorithm in
conjunction with linear function approximation; recall
Example~\ref{ExaTDLinearFunction}. Note that for any
vector $v \in \sphere^{d - 1}$, by the Cauchy-Schwartz inequality, we
have
\begin{align*}
    v^\top \Exs (\phi (X) \phi (X^+)) v \leq (v^\top \Exs (\phi (X)
    \phi (X)) v)^{\frac{1}{2}} (v^\top \Exs (\phi (X^+) \phi (X^+))
    v)^{\frac{1}{2}} = v^\top \Exs (\phi (X) \phi (X)) v.
\end{align*}
So we have $\min_i \mathrm{Re} (\lambda_i (A)) \geq (1 - \discount)
\min_i \lambda_i (\Exs \phi (X) \phi (X)^\top)> 0$ and
Theorem~\ref{thm-non-asymptotic} is applicable in this case.  In
stating the resulting corollary, we let $\mu$ denote the stationary
distribution of the Markov reward process; define the covariance
matrix $M = \Exs_\mu \phi (X) \phi (X)^\top$, and the quantity
\begin{align*}
  V(\thetastar) \mydefn \asymconditioning (\Umat) (\vecnorm{\thetastar
    - \theta_0}{2} + \vecnorm{\thetastar}{2} + \sqrt{\stepsize (1 -
    \discount)^{-1}} (\sqrt{d} \sigma_\phi \vecnorm{\thetastar}{2} +
  \sigma_r \sqrt{d})) \log^{4} \frac{T}{\delta}.
\end{align*}
 
\begin{corollary}\label{cor-td-linear-function}
  Suppose that the model assumptions in
  Example~\ref{ExaTDLinearFunction} hold, we are given a discount
  factor $\discount \in (0, 1)$ and a failure probability $\delta >
  0$, and we run the LSA algorithm using a step size  $\stepsize \in
  \Big(0, \frac{1 - \discount}{1 + \asymconditioning^2 (\Umat)
    \sigma_\phi^2 d \log^3 \frac{T}{\delta}} \Big)$. Then for any
  vector $v \in \sphere^{d - 1}$, the quantity
  $\sqrt{T} \abss{v^\top (\hat{\theta}_T - \thetastar)}$ is upper
  bounded, up to a universal pre-factor, by
  \begin{align}
    \label{EqnVbound}
     \sqrt{v^\top \GamStar (\stepsize) v\log \frac{1}{\delta}} +
     \frac{\asymconditioning (U) V(\thetastar)}{1 - \discount} \left(
     \frac{\sigma_\phi \sqrt{d} + \sigma_r}{T^{\frac{1}{4}}} + \frac{1
       + \sqrt{\sigma_r / (1 - \discount)}}{\stepsize T} \right).
  \end{align}
\end{corollary}

As a consequence of the bound~\eqref{EqnVbound}, we are guaranteed
that the rescaled error $\sqrt{T} \vecnorm{\thetahat_T -
  \thetastar}{L^2 (\mu)}$ is upper bounded as
 \begin{align*}
     \sqrt{\mathrm{Tr} \left( \GamStar (\stepsize) \cdot M \right)
       \log \frac{d}{\delta}} + \frac{\asymconditioning (U)
       V(\thetastar) \sqrt{\opnorm{M} d} \log^4 d }{1 - \discount}
     \left( \frac{\sigma_\phi \sqrt{d} + \sigma_r}{T^{\frac{1}{4}}} +
     \frac{1 + \sqrt{\sigma_r / (1 - \discount)}}{\stepsize T}
     \right),
 \end{align*}
with probability $1 - \delta$.


\section{Proofs}

We now turn the proofs of our three main theorems, along with the
various corollaries.  Before proceeding to the arguments themselves,
let us summarize some notation.

\paragraph{Summary of notation:} 
For an $L^2$-integrable quasi-martingale $\{ X_t \}_{t \geq 1}$
adapted to the filtration $ \{\filtration_{t \geq 0} \}$, we define
\begin{align*}
[X]_T \mydefn \sum_{t = 0}^{T - 1} \mathrm{var} \left(X_{t + 1} |
\filtration_{t} \right),\quad \mbox{and} \quad \langle X \rangle_T
\mydefn \sum_{t = 0}^{T - 1} \left( X_{t + 1} - \Exs (X_{t + 1} |
\filtration_t) \right)^2.
\end{align*}

For two matrices $A, B$, we use $A \otimes B$ to denote their
Kronecker product and $A \oplus B$ to denote their Kronecker sum. When
it is clear from the context, we slightly overload the notation to
let $A \otimes B$ denote the 4-th-order tensor produced by taking the
tensor product of $A$ and $B$. Note that Kronecker product is just a
flattened version of the tensor. For any matrix $\Amat$, we use
$\Vec (\Amat)$ to denote the vector obtained by flattening
$A$. For a $k$-th order tensor $T$, matrix $M$ and vector $v$, we use
$T[M]$ to denote the $(k - 2)$-th order tensor obtained by applying
$T$ to matrix $M$, and similarly, we use $T[v]$ to denote the $(k -
1)$-th order tensor obtained by applying $T$ to vector $v$.

For a matrix $W \in \complex^{d \times d}$, we use
$\{\lambda_i(W)\}_{i=1}^d$ to denote its eigevalues.  The spectral
radius is given by $\rho (W) \mydefn \max_{i \in [\usedim]} |\lambda_i
(W)|$. For an invertible matrix $W$, we define the condition number
$\asymconditioning(W) \; = \; \opnorm{W} \cdot \opnorm{W^{-1}}$, where
the operator norm is given by $\opnorm{W} \mydefn \sup_{\|x\|_2 = 1}
\vecnorm{W x}{2}$.


\subsection{Preliminaries}

We now state a few preliminary facts and auxiliary results that play
an important role in the proof.

\subsubsection{Telescope identity}
\label{SecTelescope}

The proofs of all theorems make use of a basic telescope identity.  In
particular, we define the noise term
\begin{align}
  e_t (\theta) \mydefn \underbrace{(\Amat_{t} -
    \Amatbar)}_{\NoiseAplain_t} \theta - \underbrace{(\bvec_{t} -
    \bvec)}_{\noisebplain_t}.
\end{align}
With this shorthand, some straightforward algebra shows that the
Polyak-Ruppert averaged iterate $\thetabar_T$ satisfies the
\emph{telescope relation}
\begin{align}
 \label{eq:telescope}  
\Amatbar (\thetabar_T - \thetastar) = \frac{\theta_0 -
  \theta_T}{\stepsize T} - \frac{1}{T} \sum_{t = 0}^{T - 1} e_{t + 1} (\theta_{t}),
\end{align}
involving the non-averaged sequence $\{\theta_t \}_{t \geq 1}$.

\subsubsection{Properties of the process $\{\theta_t\}_{t \geq 0}$}
\label{SecBasicProperties}

We make repeated use of a number of basic properties of the Markov
process $\{\theta_t \}_{t \geq 0}$, which we state here for future
reference.  All of these claims are proved in
Appendix~\ref{AppProcessProperties}.

\begin{lemma}
\label{lemma:simple-l2-estimate}
Under Assumptions~\ref{assume-indp}, ~\ref{assume-second-moment},
and~\ref{assume-hurwitz}, for any step size $\stepsize \in \Big(0,
\frac{\spectralgap}{\rho^2(\Amatbar) + \asymconditioning^2 (\Umat)
  v_A^2}\Big)$ and any $t \geq 1$, we have the moment bounds
\begin{subequations}
  \begin{align}
\label{EqnBasicMoment}    
\Exs \vecnorm{\theta_t - \thetastar}{2}^2 & \leq
\asymconditioning^2(\Umat) \left( \Exs \vecnorm{\theta_0 -
  \thetastar}{2}^2 + \frac{\stepsize}{\spectralgap} (v_A^2
\vecnorm{\thetastar}{2}^2 + v_b^2 d) \right).
\end{align}
If we assume furthermore that $(2 + \alpha)$-moments of the noises $\NoiseA$ and $\noiseb$ are finite, there exists a constant $\stepsize_0$, such that for $\stepsize < \stepsize_0$ we have:
\begin{align}
\label{Eqn2plusMoment}
\Exs \vecnorm{\theta_t - \thetastar}{2}^{2 + \alpha} & \leq M \quad
\mbox{for some $M < \infty$.}
\end{align}
\end{subequations}
\end{lemma}
\noindent See Appendix~\ref{sec:proof:lemma:simple-l2-estimate} for
the proof of this claim. \\

\noindent For future use, we also state a foundational lemma on the
stationary distribution of the Markov chain.
\begin{lemma}
\label{lemma:stationary-existence-uniqueness}
Under Assumptions~\ref{assume-indp}, ~\ref{assume-second-moment},
and~\ref{assume-hurwitz}, for any choice of step size \mbox{$\stepsize
  \in \Big(0, \frac{\spectralgap}{\rho^2(\Amatbar) +
    \asymconditioning^2(\Umat) v_A^2} \Big)$,} the Markov process
$(\theta_t)_{t = 0}^{+\infty}$ satisfies the following properties: (i)
it has a unique stationary distribution $\pi_\stepsize$; and (ii) the
stationary distribution has finite second moments, and concretely we
have
\begin{subequations}
\begin{align}
    \label{EqnStatMoments}
\Exs_{\pi_\stepsize} (\theta) = \thetastar, \quad \mbox{and} \quad
\mathrm{cov}_{\pi_\stepsize} (\theta) = \LamStar,
  \end{align}
  where $\LamStar$ is the unique solution to
  equation~\eqref{eq:main-stationary-cov}.  Finally, we have the
  moment bound
  \begin{align}
   \label{EqnStatSecondMomentBound}
   \Exs_{\pi_\stepsize} \vecnorm{\theta - \thetastar}{2}^2 \leq
   \asymconditioning^2 (\Umat) \frac{\stepsize}{\spectralgap} (v_A^2
   \vecnorm{\thetastar}{2}^2 + v_b^2 d).
  \end{align}
\end{subequations}
\end{lemma}
\noindent See Appendix~\ref{sec:stationary-distribution} for the proof
of this claim. \\


In the following, we state a coupling result that allows us to prove
existence of the stationary distribution, and to control the rate of convergence to stationarity.  We
first observe that using standard properties of the Kronecker product,
the matrix equation~\eqref{eq:main-stationary-cov} can be re-written
in the following equivalent but vectorized form:
\begin{align}
    \left( \Amat \oplus \Amat - \stepsize \Amat \otimes \Amat -
    \stepsize \Exs (\NoiseA \otimes \NoiseA) \right) \Vec(\LamMat) =
    \stepsize \Vec(\SigStar).
\end{align}
Moreover, since we have $\Amat \oplus \Amat \succeq 2 \spectralgap$
under Assumption~\ref{assume-hurwitz}, the minimal requirement (up to
constant factors) on the step size $\stepsize$ for
equation~\eqref{eq:main-stationary-cov} to have a PSD solution is:
\begin{align}
  \label{eq:stepsize-requirement}
A \oplus A - \stepsize A \otimes A - \stepsize \Exs (\NoiseA \otimes
\NoiseA) \succeq \spectralgap I_{d \times d}.
\end{align}
With this definition, we have
\begin{lemma}
\label{lemma:mixing-hurwitz}
Suppose that Assumptions~\ref{assume-indp}, \ref{assume-second-moment}
and~\ref{assume-hurwitz} all hold, and consider the Markov chain
$(\theta_t)_{t \geq 0}$ with any step size $\stepsize > 0$ satisfying
equation~\eqref{eq:stepsize-requirement}.  Then for any two starting
points $\theta_0^{(1)}$ and $\theta_0^{(2)}$, we have:
\begin{align}
  \Wass_2 (\law(\theta_T^{(1)}), \law(\theta_T^{(2)})) \leq e^{-
    \spectralgap \stepsize T/2} \asymconditioning (\Umat)
  \vecnorm{\theta_0^{(1)} - \theta_0^{(2)}}{2}.
\end{align}
In particular, any $\stepsize \leq \frac{\spectralgap}{\rho (\Amat)^2
  + \asymconditioning^2 (U) \sigsqA}$ satisfies
equation~\eqref{eq:stepsize-requirement} and makes the above claim
true.
\end{lemma}
\noindent See Appendix~\ref{AppLemMixingHurwitz} for the proof of claim. \\

An elementary consequence of Lemma~\ref{lemma:mixing-hurwitz} is the
following bound on the Wasserstein-$2$ distance:
\begin{align}
\label{EqnWassersteinBound}  
  \Wass_2 \left( \law (\theta_T), \pi_\stepsize \right) \leq e^{-
    \frac{\stepsize \spectralgap T}{2}} \asymconditioning (\Umat)
  \Wass_2 (\mu, \pi_\stepsize).
\end{align}
The proof of this claim is straightforward: we simply take the optimal
coupling between the initial laws $\mu_0$ and $\pi_\stepsize$, apply
Lemma~\ref{lemma:mixing-hurwitz} conditionally on the starting points,
and then take expectations.

Finally, we give control on the support size and coupling estimates on the process in the $\ell_\infty$ setting, which is used in the proof of Theorem~\ref{thm:non-asymptotic-linfty}.

\begin{lemma}\label{lemma:linfty-estimate-and-contraction}
  Under Assumption~\ref{assume-indp},~\ref{assume-noise-structure-linfty} and~\ref{assume-linfty-contraction}, for $\stepsize \leq 1$, given $\theta_0 \in [- \linftygap^{-1}, \linftygap^{-1}]^{d}$, we have $\vecnorm{\theta_t}{\infty} \leq \linftygap^{-1}$ for any $t \geq 0$. Furthermore, for any two starting points $\theta_0^{(1)}, \theta_0^{(2)} \in [- \linftygap^{-1}, \linftygap^{-1}]^{d}$, we have:
  \begin{align*}
      \Wass_{\vecnorm{\cdot}{\infty}, \infty} (\law (\theta_1^{(1)}), \law (\theta_1^{(2)})) \leq (1 - \stepsize \linftygap) \vecnorm{\theta_0^{(1)} - \theta_0^{(2)}}{\infty}.
  \end{align*}
\end{lemma}
See Appendix for the proof of this lemma.


\subsection{Proof of Theorem~\ref{thm-asymptotic}}

We are now equipped to prove Theorem~\ref{thm-asymptotic}.  First, by
the telescope identity~\eqref{eq:telescope}, we have
\begin{align*}
\frac{\theta_T - \theta_0}{\eta \sqrt{T}}= - \Amatbar\left[
  \frac{1}{\sqrt{T}} \sum_{t=0}^{T - 1} (\theta_{t} - \thetastar) \right]
- \frac{1}{\sqrt{T}} \sum_{t = 0}^{T - 1} e_{t + 1} (\theta_{t}).
\end{align*}
From its definition, it can be seen that the sequence
$\{e_t (\theta_{t}) \}_{t \geq 0}$ is a vector martingale difference
sequence with respect to the filtration $\{
\mathcal{F}_{t - 1} \}_{t \geq 0}$ (for notational consistency, we let $\filtration_{-1}$ denote the trivial $\sigma$-field). Accordingly, we can apply a
martingale CLT en route to establishing the claim.  In order to do so,
we begin by computing the relevant conditional second moments.

We let $r_t \mydefn \theta_t - \thetastar$ denote the error in the
non-averaged sequence at time $t$.  Observe that we have the relation
$e_{t + 1} (\theta_{t}) = e^{(1)}_{t + 1} + e^{(2)}_{t + 1}$, where
\begin{align*}
  e^{(1)}_{t + 1} \mydefn \NoiseAplain_{t + 1} r_{t}, \quad \mbox{and} \quad
  e^{(2)}_{t + 1} \mydefn - \noisebplain_{t + 1} + \NoiseAplain_{t + 1} \thetastar.
\end{align*}
Based on this decomposition, we can expand the conditional covariance
of $e_{t + 1}(\theta_{t})$ as a sum of four terms:
\begin{align*}
\Exs\left[e_{t + 1} (\theta_{t}) e_{t + 1} (\theta_{t})^\top \mid
  \filtration_{t} \right] &= \Exs \left[ e_{t + 1}^{(1)} (e_{t + 1}^{(1)})^\top +
  e_{t + 1}^{(2)} (e_{t + 1}^{(2)})^\top + e_{t + 1}^{(1)} (e_{t + 1}^{(2)})^\top + e_{t + 1}^{(2)}
  (e_{t + 1}^{(1)})^\top \mid \filtration_{t} \right].
\end{align*}

We treat each of these four terms in turn.  For the first term, we
note that:
\begin{subequations}
\begin{align}
    \frac{1}{T} \sum_{t = 0}^{T - 1} \Exs \left[ e_{t + 1}^{(1)} (e_{t + 1}^{(1)})^\top
      \mid \filtration_t \right] & = \frac{1}{T}\sum_{t =
      0}^{T - 1}\Exs\left[\NoiseAplain_{t + 1} r_t r_t^\top \NoiseAplain_{t + 1}^\top \mid
      \filtration_{t}\right] \nonumber \\
& = \Exs (\NoiseA \otimes \NoiseA) \left[\frac{1}{T} \sum_{t = 0}^{T -
        1} r_t r_t^\top \right].
\end{align}
For the second term, by Assumption~\ref{assume-indp}, the noises
$\NoiseAt$ and $\noisebt$ are uncorrelated, so we have:
\begin{align}
 \Exs \left[ e_{t + 1}^{(2)} (e_{t + 1}^{(2)})^\top \mid \filtration_{t}
   \right] & = \Exs\left[\left( - \noisebplain_{t + 1} + \NoiseAplain_{t + 1} \thetastar
   \right) \left( - \noisebplain_{t + 1} + \NoiseAplain_{t + 1} \thetastar \right)^\top
   \mid \filtration_{t}\right] \nonumber \\
& = \Exs (\xi \xi^\top) + \Exs \left( (\NoiseA \thetastar) (\NoiseA
 \thetastar)^\top \right).
\end{align}
For the third term, we note that:
\begin{align}
   \frac{1}{T} \sum_{t = 0}^{T - 1} \Exs \left[ e_{t + 1}^{(1)} (e_{t + 1}^{(2)})^\top
     \mid \filtration_{t}\right] & = \frac{1}{T} \sum_{t = 0}^{T - 1}
   \Exs \left[ \NoiseAplain_{t + 1} r_{t} (\NoiseAplain_{t + 1} \thetastar)^\top \mid
     \filtration_{t}\right] \nonumber \\
 &  = \Exs (\NoiseA \otimes \NoiseA)
   \left[ \frac{1}{T} \sum_{t = 0}^{T - 1} r_{t} \thetastar^\top
     \right].
\end{align}
Similarly, for the fourth term, we have
\begin{align}
    \frac{1}{T} \sum_{t = 0}^{T - 1} \Exs \left[ e_{t + 1}^{(2)} (e_{t + 1}^{(2)})^\top
      \mid \filtration_{t} \right] & = \frac{1}{T} \sum_{t = 1}^T
    \Exs \left[ e_{t + 1}^{(2)} (e_{t + 1}^{(1)})^\top \mid \filtration_{t}\right] \nonumber \\
& = \Exs (\NoiseA \otimes \NoiseA) \left[ \frac{1}{T} \sum_{t = 0}^{T
        - 1} \thetastar r_t^\top \right ].
\end{align}
\end{subequations}
The second conditional expectation term is a deterministic quantity,
while other three terms depend on the random variable $r_{t}$. When taking the quadratic variation of the martingale $M_t$, we
get the partial sum of functions of a Markov chain $(\theta_t)_{t \geq
  0}$.  Accrdingly, we now use
Lemma~\ref{lemma:stationary-existence-uniqueness}, which guarantees
the existence of a unique stationary measure $\pi_\stepsize$, in order
to study the limits of the first three terms.

Note that for any vectors $u, v \in \sphere^{d - 1}$, the functions
$(u,v) \mapsto (u^\top \theta) (v^\top \theta)$ and $v \mapsto (v^\top
\theta) (v^\top \thetastar)$ are $L^1$ integrable under the stationary
measure $\pi_\stepsize$.  Consequently, by Birkhoff's ergodic theorem
(cf.~\cite{kallenberg2006foundations}, Theorem 9.6), we have:
\begin{align*}
\frac{1}{T} \sum_{t = 0}^{T - 1} u^\top r_t r_t^\top v &\rightarrow u^\top
\Exs_{\pi_\stepsize} (\theta - \thetastar) (\theta - \thetastar)^\top
v = u^\top \LamStar v, \quad \mathrm{a.s.}\\ \frac{1}{T} \sum_{t =
  0}^{T - 1} u^\top r_t \thetastar^\top v &\rightarrow u^\top
(\Exs_{\pi_\stepsize} \theta - \thetastar) \thetastar^\top v = 0,
\quad \mathrm{a.s.}
\end{align*}
Thus, the ergodic averages converge to the corresponding limits,
which implies that
\begin{align*}
  \frac{1}{T} \sum_{t = 0}^{T - 1} \Exs \left[ e_{t + 1}^{(1)} (e_{t + 1}^{(1)})^\top
    \mid \filtration_{t} \right] = \Exs (\NoiseA \otimes \NoiseA)
  \left[\frac{1}{T} \sum_{t = 0}^{T - 1} r_t r_t^\top \right] &\rightarrow
  \Exs\left( \NoiseA \LamStar \NoiseA^\top \right), \quad
  \mathrm{a.s.}, \quad \mbox{and} \\
  \frac{1}{T} \sum_{t = 0}^{T - 1} \Exs \left[ e_{t + 1}^{(1)} (e_{t + 1}^{(2)})^\top
    \mid \filtration_{t} \right] = \Exs (\NoiseA \otimes \NoiseA)
  \left[\frac{1}{T} \sum_{t = 0}^{T - 1} r_t \thetastar^\top \right]
  &\rightarrow 0, \quad \mathrm{a.s.}
 \end{align*}
Combining the pieces yields
 \begin{align*}
   \frac{1}{T} \sum_{t = 0}^{T - 1} \Exs \left[ e_{t + 1} (\theta_{t}) (e_{t + 1}
     (\theta_{t}))^\top \mid \filtration_{t} \right]
   \rightarrow \Exs (\noiseb \noiseb^\top) + \Exs \left( (\NoiseA
   \thetastar) (\NoiseA \thetastar)^\top \right) + \Exs \left(\NoiseA
   \LamStar \NoiseA^\top \right), \quad \mathrm{a.s.}
 \end{align*}

In order to prove the martingale CLT, it remains to verify that the
process $e_t(\theta_{t-1})$ satisfies a Lindeberg-type condition when
projected in an arbitrary direction $u \in \sphere^{d-1}$.  (Doing so
is sufficient since Markov's inequality allows us to translate it to a
Lyapunov-type condition.)  Accordingly, we seek to bound a $(2 +
\alpha)$-moment of the martingale differences, which furthermore
requires a uniform bound on the $(2 + \alpha)$-moment for the process
$(\theta_t)_{t \geq 0}$.

Using the $(2 + \alpha)$-moment bound~\eqref{Eqn2plusMoment} from
Lemma~\ref{lemma:stationary-existence-uniqueness}, we have
\begin{align*}
  \Exs |u^\top e_{t + 1} (\theta_{t})|^{2 + \alpha} &\leq \Exs \abss{
    2 u^\top e_{t + 1}^{(1)} }^{2 + \alpha} + \Exs \abss{ 2 u^\top
    e_{t + 1}^{(2)} }^{2 + \alpha} \\ &\leq 2^{2 + \alpha}\Exs
  \left(\opnorm{\NoiseAplain_{t + 1}} \vecnorm{r_{t}}{2} \right)^{2 +
    \alpha} + 2^{2 + \alpha} \Exs \vecnorm{\NoiseA \thetastar -
    \noiseb}{2}^{2 + \alpha}\\ & \leq 2^{2 + \alpha} \Exs
  \opnorm{\NoiseA}^{2 + \alpha} \cdot M + 4^{2 + \alpha} \left( \Exs
  \vecnorm{\NoiseA \thetastar}{2}^{2 + \alpha} + \Exs
  \vecnorm{\noiseb}{2}^{2 + \alpha}\right) \mydefn Q < +\infty.
\end{align*}
Notably, the quantity $Q$ is independent of $t$.

Therefore, for a fixed $\epsilon > 0$, the quantity $E
\defn\frac{1}{T} \sum_{t = 0 }^{T - 1} \Exs\left[ \left(u^\top e_{t + 1}
  (\theta_{t})\right)^2 1\left( \left|u^\top e_{t + 1} (\theta_{t})
  \right| > \epsilon\sqrt{T}\right) \right]$ is upper bounded
as
\begin{align*}
E & \leq \frac{1}{T} \sum_{t = 0}^{T - 1} \frac{1}{\epsilon^{\alpha}
  T^{\alpha/2}} \Exs\left[ \left(u^\top e_t
  (\theta_{t})\right)^{2+\alpha} 1\left( \left|u^\top e_t
  (\theta_{t}) \right| > \epsilon\sqrt{T}\right) \right] \\&\le
\frac{1}{\epsilon^{\alpha} T^{\alpha/2}}\cdot \frac{1}{T} \sum_{t = 0}^{T - 1}
\Exs \left(u^\top e_t (\theta_{t})\right)^{2+\alpha} \; \leq \;
\frac{1}{\epsilon^{\alpha} T^{\alpha/2}}\cdot Q.
\end{align*}
Note that this bound converges to zero as $T \rightarrow \infty$.
 
Applying the one-dimensional martingale central limit theorem
(cf. Corollary 3.1 in the book~\cite{HALL-HEYDE}), we have the
convergence of $\frac{1}{\sqrt{T}} \sum_{t = 0}^{T - 1} u^\top e_t
(\theta_{t})$. Combined with the Cram\'er-Wold device, we conclude
that $\frac{1}{\sqrt{T}} \sum_{t=0}^T e_t (\theta_{t})$ converges in
distribution to a zero-mean Gaussian with covariance $\Exs (\NoiseA
\LamStar \NoiseA^\top) + \SigStar$.  By
Lemma~\ref{lemma:simple-l2-estimate}, we have $\sqrt{T} \cdot
\frac{1}{\stepsize T} (\theta_T - \thetastar) \rightarrow 0$ almost
surely. Therefore, by the telescoping equation~\eqref{eq:telescope},
we have:
\begin{align*}
  A \left[ \frac{1}{\sqrt{T}} \sum_{t=0}^{T - 1} (\theta_{t}-\thetastar)
    \right] \overset{d}{\rightarrow} \mathcal{N} \left(0, \Exs
  (\NoiseA \Lambda \NoiseA^\top) + \SigStar \right).
\end{align*}
Taking the inverse of $A$ completes the proof.




\subsection{Proof of Theorem~\ref{thm-non-asymptotic}}

In order to prove this theorem, we require an auxiliary result that
provides bounds on higher-order moments of the process.
\begin{lemma}
\label{lemma:theta-norm-hurwitz}
Suppose that Assumptions~\ref{assume-indp},
\ref{assume-noise-subgaussian} and~\ref{assume-hurwitz} all hold.
Given some $p \geq 2 \log T$, consider any step size \mbox{$\stepsize
  \in \Big(0, \frac{\spectralgap}{\rho^2 (\Amatbar) + C p^{2 \alpha +
      1} \asymconditioning^2 (\Umat) \sigma_A^2}\Big)$.} Then there is
a universal constant $c$ such that
\begin{align}
(\Exs \vecnorm{\theta_t - \thetastar}{2}^p)^\frac{2}{p} & \leq c \;
  \asymconditioning^2 (\Umat) \left( (\Exs \vecnorm{\theta_0 -
    \thetastar}{2}^p)^{\frac{2}{p}} + \frac{\stepsize}{\spectralgap}
  (p^{2 \beta + 1} \sigma_b^2 d + p^{2 \alpha + 1} \sigma_A^2
  \vecnorm{\thetastar}{2}^2 ) \right).
\end{align}
\end{lemma}
\noindent See Appendix~\ref{AppThetaNormHurwitz} for the proof of this
claim. \\

Equipped with this lemma, we now turn to the proof of the theorem.  We
consider the martingale term $M_t \mydefn \sum_{s = 1}^{t} e_{s + 1}
(\theta_{s})$.  By the telescope equation~\eqref{eq:telescope}, we
need to bound in any direction the variation of $\frac{1}{T \stepsize}
\Amatbar^{-1} (\theta_0 - \theta_T)$ and $\frac{1}{T} \sum_{t = 0}^{T
  - 1} \Amatbar^{-1} e_{t + 1} (\theta_t)$, respectively. For any vector $v
\in \sphere^{d - 1}$, define $M^{(v)}_t \mydefn \sum_{t = 0}^{T - 1}
\Amatbar^{-1} v^\top e_{t + 1} (\theta_t)$. Since $M^{(v)}_t$ is a
martingale, we can apply the discrete-time Burkholder-Davis-Gundy
(BDG) inequality~\citep{burkholder1972integral}: it guarantees the
existence of a finite constant $C$ such that for any $p \geq 4$, we
have
\begin{align*}
\Exs \sup_{0 \leq t \leq T} |M^{(v)}_t|^p & \leq (C p)^{\frac{p}{2}}
\Exs \langle M^{(v)} \rangle_T^\frac{p}{2} \; = \;
(C p)^{\frac{p}{2}} \Exs \left( \sum_{t = 0}^{T - 1} (v^\top e_{t} (\theta_t))^2 \right)^\frac{p}{2}.
\end{align*}
Moreover, we have
\begin{align*}
\Exs \left( \sum_{t = 0}^{T - 1} (v^\top e_{t + 1} (\theta_t))^2
\right)^\frac{p}{2} & = \Exs \left( \sum_{t = 0}^{T - 1} \left(
(v^\top \NoiseAplain_{t + 1} \theta_t)^2 + (\noisebplain_{t + 1}^\top
v)^2 - 2 (v^\top \NoiseAplain_{t + 1} \theta_t^\top)(v^\top
\noisebplain_{t + 1}) \right) \right)^\frac{p}{2} \\
& \leq 3^{p/2} \; \sum_{j=1}^3 I_j,
\end{align*}
where $I_1 \defn \Exs \left( \sum_{t = 0}^{T - 1} (\theta_t^\top
\NoiseAplain_{t + 1} v)^2 \right)^\frac{p}{2}$, along with
\begin{align*}
I_2 \defn \Exs \left( \sum_{t = 0}^{T - 1} (v^\top \noisebplain_{t + 1})^2 \right)^\frac{p}{2}, \quad \mbox{and}
\quad I_3 \defn \Exs \abss{\sum_{t = 0}^{T - 1}
  2 (v^\top \NoiseAplain_{t + 1} \theta_t^\top)(v^\top
  \noisebplain_{t + 1})}^{\frac{p}{2}}.
\end{align*}
By the Cauchy-Schwartz inequality, we have $I_3 \leq \sqrt{I_1 I_2}$. So we only need to bound the terms $I_1$ and $I_2$.

We now state an auxiliary result that bounds each of these terms:
\begin{lemma}
\label{LemIbounds}
\begin{subequations}
We have the bounds
    \begin{align}
  I_2 & \leq (2 v^\top \Sigma_\xi v T)^{\frac{p}{2}} + C_\beta^p
  \sigma_b^p \left( (pT)^{\frac{p}{4}} + (p \log T)^{\frac{p}{2} (1 +
    2 \beta)} \right),
    \end{align}
and
\begin{multline}
  (I_1)^{\frac{2}{p}} \leq 3 T v^\top \Exs (\NoiseA (\LamStar +
  \thetastar \thetastar^\top) \NoiseA^\top) v + \frac{12 v_A^2
    \asymconditioning^2 (\Umat)}{\spectralgap \stepsize } \left(
  \trace (\LamStar) + \vecnorm{\thetastar}{2}^2 + \vecnorm{\theta_0 -
    \thetastar}{2}^2 \right) \\
+ C \opnorm{\Sigma_\NoiseAplain [v, v]} \frac{\asymconditioning^2
  (\Umat)}{\spectralgap} B_p \left( \sigma_A (B_p +
\vecnorm{\thetastar}{2}) (p \log T)^{\alpha}+ \sigma_b \sqrt{d} (p
\log T)^\beta \right) \sqrt{p T \log T }\\ + \sqrt{C p T} \sigma_A^2
p^{2 \alpha} \asymconditioning^2 (\Umat) B_p^2.
\end{multline}
\end{subequations}
\end{lemma}
\noindent See Section~\ref{AppLemIbounds} for the proof of this claim. \\

Combining the results for $I_1$, $I_2$, $I_3$, we obtain the main
moment bound on the supremum of martingale $M_t^{(v)}$. Denote the
matrix $\tilde{\Sigma} \mydefn \Exs (\NoiseA \otimes \NoiseA ) \otimes
\LamStar$, and denote $Z_p \mydefn \sigma_A
\vecnorm{\thetastar}{2} (p \log T)^\alpha + \sigma_b \sqrt{d} (p \log
T)^{\beta}$. We obtain:
\begin{multline*}
  \frac{1}{\sqrt{T}} \left( \Exs \sup_{0 \leq t \leq T} | M_t^{(v)}|^p
  \right)^{\frac{1}{p}} \lesssim \sqrt{p v^\top (\SigStar +
    \tilde{\Sigma}) v} + \sqrt{p} \sigma_b \left(
  (\frac{p}{T})^{\frac{1}{4}} + \frac{(p \log T)^{\beta +
      1/2}}{\sqrt{T}} \right) \\
  + p \log T \cdot T^{- \frac{1}{4}} \frac{\asymconditioning (\Umat)
    \sqrt{\opnorm{\tilde{\Sigma}}} }{\sqrt{\spectralgap \stepsize}}
  (\sqrt{\frac{\stepsize}{\spectralgap}} Z_p + \vecnorm{\thetastar -
    \theta_0}{2}) + \frac{v_A \asymconditioning (\Umat)}{\sqrt{T
      \spectralgap \stepsize}} (\vecnorm{\thetastar}{2} +
  \vecnorm{\theta_0}{2} + \sqrt{\trace(\LamStar)} )\\ + \sqrt{p} T^{-
    \frac{1}{4}} p^{\alpha + \beta} \sqrt{\sigma_A \sigma_b}
  \asymconditioning (\Umat) (\vecnorm{\theta_0 - \thetastar}{2} +
  \sqrt{\frac{\stepsize}{\spectralgap}} Z_p ),
\end{multline*}
for $p > 2\log T$ and $\stepsize$ satisfying the assumption in the
theorem.

For the bias term, we note that:
\begin{align*}
  \left( \Exs \vecnorm{\theta_T - \thetastar}{2}^{p}
  \right)^{\frac{2}{p}} \leq \asymconditioning^2 (\Umat)
  \left(\vecnorm{\theta_0 - \thetastar}{2} +
  \frac{\stepsize}{\spectralgap} Z_p \right).
\end{align*}
Finally, putting together the previous results and merging the terms, we obtain
the upper bound
\begin{multline*}
\sqrt{T} \left( \Exs |v^\top A(\thetabar_T -
\thetastar)|^p\right)^{\frac{1}{p}} \lesssim \sqrt{p v^\top (\SigStar
  + \tilde{\Sigma}) v} \\
+ \asymconditioning (\Umat) (p \log T)^{2 \max(\alpha, \beta) + 2}
\left( \frac{\sigma_A + \sigma_b}{T^{\frac{1}{4}}} + \frac{1 +
  \sqrt{\sigma_A/\spectralgap}}{\stepsize \sqrt{T}}\right) \left(
\vecnorm{\thetastar}{2} + \vecnorm{\theta_0}{2} +
\sqrt{\frac{\stepsize}{\spectralgap}} (\sigma_A
\vecnorm{\thetastar}{2} + \sigma_b \sqrt{d}) \right).
\end{multline*}
Applying Markov's inequality yields the claimed high-probability
bound.

\subsection{Proof of Lemma~\ref{LemIbounds}}
\label{AppLemIbounds}

The remainder of our effort is devoted to proving the bounds on the
terms $\{I_1, I_2, I_3 \}$ claimed in Lemma~\ref{LemIbounds}.

\subsubsection{Upper bounds on $I_1$}

We begin by observing that
\begin{align*}
  \Exs \sum_{t = 0}^{T - 1} (\theta_t^\top \NoiseAplain_{t + 1} v)^2 =
  \Exs \sum_{t = 0}^{T - 1} v^\top \Exs ( \NoiseAplain_{t + 1} \otimes
  \NoiseAplain_{t + 1}^\top | \filtration_{t} ) [ \theta_t
    \theta_t^\top, v] = \inprod{ \Sigma_\NoiseAplain [vv^\top]}{ \Exs
    \left( \sum_{t = 0}^{T - 1} \theta_t \theta_t^\top \right)}
\end{align*}
In order to deal with the concentration behavior of this term, let
$\Psi_T \mydefn \sum_{t = 0}^{T - 1} \Exs\left( (\theta_t^\top \NoiseAplain_{t + 1} v)^2
| \filtration_t \right)$, and let $\Upsilon_T \mydefn \sum_{t = 0}^{T - 1}
(\theta_t^\top \NoiseAplain_{t + 1} v)^2 - \Psi_T$. By definition, it is easy to
see that $\Upsilon$ is a martingale. Applying the BDG inequality and
H\"{o}lder's inequality, we have:
\begin{align*}
  \Exs \sup_{0 \leq t \leq T - 1} |\Upsilon_t|^{\frac{p}{2}} &\leq (C
  p)^{\frac{p}{4}} \Exs \langle \Upsilon \rangle_{T}^{\frac{p}{4}}\\
  & = (C p)^{\frac{p}{4}} \Exs \left( \sum_{t = 0}^{T - 1} \left(
  (\theta_t^\top \NoiseAplain_{t + 1} v)^2 - \Exs ((\theta_t^\top \NoiseAplain_{t + 1} v)^2 \mid \filtration_{t})^2 \right)
  \right)^{\frac{p}{4}}\\ &\leq (C p)^{\frac{p}{4}} \Exs \left(
  \sum_{t = 0}^{T - 1} (\theta_t^\top \NoiseAplain_{t + 1} v)^4
  \right)^{\frac{p}{4}} \\
  & \leq (C p)^{\frac{p}{4}} T^{\frac{p}{4} - 1}\sum_{t = 0}^{T - 1}
  \Exs |\theta_t^\top \NoiseAplain_{t + 1} v|^p\\
  & \leq (C p)^{\frac{p}{4}} T^{\frac{p}{4}} \sigma_A^p p^{\alpha p}
  \max_{0 \leq t \leq T - 1}\Exs \vecnorm{\theta_t}{2}^p.
\end{align*}
As for the process $\{\Psi_T \}_{T \geq 1}$, a straightforward
calculation yields:
\begin{align*}
    \Psi_T = \sum_{t = 0}^{T - 1} \Exs\left( (\theta_t^\top
    \NoiseAplain_{t} v)^2 \mid \filtration_t \right) = \inprod{
      \Sigma_\NoiseAplain [vv^\top]}{\sum_{t = 0}^{T - 1} \theta_t
      \theta_t^\top}.
\end{align*}
The summation $\sum_{t = 0}^{T - 1} \theta_t \theta_t^\top$ involves terms
that are functions of an ergodic Markov chain. Thus, metric ergodicity
concentration inequalities based on Ricci curvature techniques can
show its concentration around its expectation. We first study the
expectation of this process. Let $(\thetatil_t)_{t \geq 0}$ be a
stationary chain which starts from $\pi_\stepsize$, couple the
processes $(\theta_t)_{t \geq 0}$ and $(\thetatil_t)_{t \geq 0}$
in the manner defined by Lemma~\ref{lemma:mixing-hurwitz}. By definition,
there is $\Exs \thetatil_t \thetatil_t^\top =
\Exs_{\pi_\stepsize} \theta \theta^\top$. For any matrix $L$, we have
\begin{align*}
  &\abss{\frac{1}{T} \Exs \left(\sum_{t = 0}^{T - 1} \inprod{\theta_t
      \theta_t^\top} {L} \right) - \Exs_{\pi_\stepsize}
    \inprod{\theta\theta^\top}{L}} \leq \frac{1}{T} \sum_{t = 0}^{T -
    1} \Exs \abss{\theta_t^\top L \theta_t - \thetatil_t^\top L
    \thetatil_t} \\
& \leq \frac{1}{T} \sum_{t = 0}^{T - 1} \left( \Exs \abss{(\theta_t -
    \thetatil)^\top L (\theta_t -\thetatil) } + 2 \Exs
  \abss{ (\theta_t - \thetatil_t)^\top L \thetatil_t}
  \right)\\ &\leq \frac{1}{T} \sum_{t = 0}^{T - 1} \left(\opnorm{L}
  \Exs \vecnorm{\theta_t - \thetatil_t}{2}^2 + 2 \opnorm{L}
  \sqrt{\Exs \vecnorm{\theta_t - \thetatil_t}{2}^2} \cdot
  \sqrt{\Exs \vecnorm{\thetatil_t}{2}^2} \right).
\end{align*}
By Lemma~\ref{lemma:mixing-hurwitz}, for this coupling, we have:
\begin{align*}
  \Exs \vecnorm{\theta_t - \thetatil_t}{2}^2 \leq
  \asymconditioning^2 (\Umat) e^{- \spectralgap \stepsize t}
  \Wass_2^2 (\law(\theta_0), \pi_\stepsize).
\end{align*}
By definition, we have $\Exs\vecnorm{\thetatil_t}{2}^2 = \trace
(\LamStar) + \vecnorm{\thetastar}{2}^2$, and it is easy to
see that $\Wass_2^2 (\law(\theta_0), \pi_\stepsize) \leq \Exs
\vecnorm{\theta_0 - \thetastar}{2} + \Exs_{\pi_\stepsize}
\vecnorm{\theta - \thetastar}{2}^2 \leq \vecnorm{\theta_0 -
  \thetastar}{2} + \trace (\LamStar)$. Plugging into the
above inequality, we obtain:
\begin{multline*}
    \abss{\frac{1}{T} \Exs \left(\sum_{t = 0}^{T - 1} \inprod{\theta_t
        \theta_t^\top} {L} \right) - \Exs_{\pi_\stepsize}
      \inprod{\theta\theta^\top}{L}} \leq \frac{2 \opnorm{L}
      \asymconditioning^2 (\Umat)}{T} \left( \trace (\LamStar) +
    \vecnorm{\thetastar}{2}^2 + \vecnorm{\theta_0 - \thetastar}{2}^2
    \right) \sum_{t = 0}^{T - 1} e^{- \frac{\spectralgap \stepsize
        t}{2}} \\
    \leq \frac{4 \opnorm{L} \asymconditioning^2 (\Umat)}{\spectralgap
      \stepsize T} \left( \trace (\LamStar) +
    \vecnorm{\thetastar}{2}^2 + \vecnorm{\theta_0 - \thetastar}{2}^2
    \right).
\end{multline*}
In particular, for $L = \Sigma_\NoiseAplain [v, v]$, we have:
\begin{align*}
    \abss{\frac{1}{T} \Exs \Psi_T - \Exs (\NoiseA \otimes \NoiseA)
      [\LamStar + \thetastar \thetastar^\top] } \leq
    \frac{4 v_A^2 \asymconditioning^2 (\Umat)}{\spectralgap \stepsize T}
    \left( \trace (\LamStar) + \vecnorm{\thetastar}{2}^2 +
    \vecnorm{\theta_0 - \thetastar}{2}^2 \right).
\end{align*}
By Lemma~\ref{LemErgconcTwo}, for any $\delta > 0$, for $B =
\vecnorm{\theta_0 - \thetastar}{2} + \frac{\stepsize}{\spectralgap}
(\sigma_b \sqrt{d} \log^{\beta + 1/2} \frac{T}{\delta} + \sigma_A
\vecnorm{\thetastar}{2} \log^{\alpha + 1/2} \frac{T}{\delta} )$, for
$\stepsize < \frac{\spectralgap}{\rho^2 (\Amatbar) + C
  \asymconditioning^2 (\Umat) \sigma_A^2 \log^{2 \alpha + 1} T/\delta
}$, with probability $1 - \delta$, we have:
\begin{align*}
  \abss{ \Psi_T - \Exs \Psi_T } \leq C \opnorm{ \Sigma_\NoiseAplain
    [vv^\top]} \frac{\asymconditioning^2 (\Umat)}{\spectralgap} B \left(
  \sigma_A (B + \vecnorm{\thetastar}{2})\log^{\alpha}\frac{T}{\delta}
  + \sigma_b \sqrt{d} \log^{\beta}\frac{T}{\delta} \right) \sqrt{T
    \log \delta^{-1}} \mydefn Q_\delta.
\end{align*}
Note that this bound holds true only for a fixed failure probability
$\delta$. In order to obtain the moment bounds on $\Psi$, we also use
a coarse estimate: $|\Psi_T - \Exs \Psi_T| \leq T \opnorm{vv^\top
  \Sigma_\NoiseAplain} \max_{0 \leq t \leq T -
  1}. \vecnorm{\theta_t}{2}^2$. Putting them together, we have:
    \begin{align*}
        \Exs |\Psi_T - \Exs \Psi_T|^\frac{p}{2} \leq
        Q_\delta^\frac{p}{2} + \Exs \left( |\Psi_T - \Exs
        \Psi_T|^\frac{p}{2} \bm{1}_{ |\Psi_T - \Exs \Psi_T| >
          Q_\delta} \right) \leq Q_\delta^\frac{p}{2} + \sqrt{\delta
          T^{p} \opnorm{ \Sigma_\NoiseAplain[v, v] }^{p} \Exs \max_{0
            \leq t \leq T - 1}. \vecnorm{\theta_t}{2}^{2p}}.
    \end{align*}
By Lemma~\ref{lemma:theta-norm-hurwitz}, we have $(\Exs \max_{0 \leq t
  \leq T - 1}. \vecnorm{\theta_t}{2}^{p})^{\frac{1}{p}} \leq C
\asymconditioning^2 (\Umat) (\vecnorm{\theta_0 - \thetastar}{2} +
\vecnorm{\thetastar}{2} + \frac{T\stepsize}{\spectralgap} (\sigma_b
\sqrt{d} p^{\beta + 1/2} + \sigma_A \vecnorm{\thetastar}{2} p^{\alpha
  + 1/2} )$. Choosing some $\delta \in \big(0, (CT)^{- p} \big)$, we
obtain that:
\begin{align*}
       \left( \Exs \abss{ \Psi_T - \Exs \Psi_T }^\frac{p}{2}
       \right)^{\frac{2}{p}} \leq C \opnorm{ \Sigma_\NoiseAplain [v,
           v]} \frac{\asymconditioning^2 (\Umat)}{\spectralgap} B_p
       \left( \sigma_A (B_p + \vecnorm{\thetastar}{2}) (p \log
       T)^{\alpha}+ \sigma_b \sqrt{d} (p \log T)^\beta \right) \sqrt{p
         T \log T },
    \end{align*}
where $B_p \mydefn \vecnorm{\theta_0 - \thetastar}{2} +
\frac{\stepsize}{\spectralgap} (\sigma_b \sqrt{d} (p\log T)^{\beta +
  1/2} + \sigma_A \vecnorm{\thetastar}{2} (p \log T)^{\alpha + 1/2}
)$.
    
Recall that we can decompose $I_1$ into three parts:
\begin{align*}
  I_1 \leq \Exs \left( \Upsilon_T + \Psi_T \right)^\frac{p}{2} \leq
  3^{\frac{p}{2}} \left( \Exs |\Upsilon_T|^{\frac{p}{2}} + (\Exs
  \Psi_T)^{\frac{p}{2}} + \Exs |\Psi_T - \Exs \Psi_T|^{\frac{p}{2}}
  \right).
\end{align*}   
Using the bounds for three terms derived above, we obtain:
\begin{multline*}
  (I_1)^{\frac{2}{p}} \leq 3 T v^\top \Exs (\NoiseA (\LamStar +
  \thetastar \thetastar^\top) \NoiseA^\top) v + \frac{12 v_A^2
    \asymconditioning^2 (\Umat)}{\spectralgap \stepsize } \left(
  \trace (\LamStar) + \vecnorm{\thetastar}{2}^2 + \vecnorm{\theta_0 -
    \thetastar}{2}^2 \right) \\
+ C \opnorm{\Sigma_\NoiseAplain [v, v]} \frac{\asymconditioning^2
  (\Umat)}{\spectralgap} B_p \left( \sigma_A (B_p +
\vecnorm{\thetastar}{2}) (p \log T)^{\alpha}+ \sigma_b \sqrt{d} (p
\log T)^\beta \right) \sqrt{p T \log T }\\ + \sqrt{C p T} \sigma_A^2
p^{2 \alpha} \asymconditioning^2 (\Umat) B_p^2.
\end{multline*}


\subsubsection{Upper bounds on $I_2$:}

Define $\noisebplain_T \mydefn \sum_{t = 0}^{T - 1} (v^\top
\noisebplain_{t + 1})^2$, we have $\Exs \noisebplain_T = v^\top
\Sigma_\xi v T$. It is easy to see that $\noisebplain_t -\Exs
\noisebplain_t$ is a martingale difference sequence, and tnus by standard
sub-exponential martingale concentration inequalities and
Assumption~\ref{assume-noise-subgaussian}, for $p \geq 2$, we have:
\begin{align*}
  \Exs \left( (v^\top \noisebplain_t)^2 - \Exs (v^\top \noisebplain_t)^2 \right)^p \leq
  \Exs (v^\top \noisebplain_t)^{2p} \leq p^{2 \beta p} \sigma_b^{2 p}.
\end{align*}

By Lemma~\ref{lemma:heavy-tail-concentration}, for any $\delta > 0$,
we have:
\begin{align*}
  \Prob \left( \frac{1}{T} |\noisebplain_T - \Exs \noisebplain_T| > C_\beta \sigma_b^2
  \left( \sqrt{\frac{\log \delta^{-1}}{T}} + \frac{\log^{1 + 2 \beta}
    T / \delta}{T} \right) \right) < \delta.
\end{align*}
Integrating the expression, we obtain the upper bound:
\begin{align*}
  I_2 & \leq (2 v^\top \Sigma_\xi v T)^{\frac{p}{2}} + 2
  \int_0^{+\infty} \Prob \left( |\noisebplain_T - \Exs \noisebplain_T|
  \geq \varepsilon \right) \varepsilon^{\frac{p}{2} - 1} d
  \varepsilon\\ & \leq (2 v^\top \Sigma_\xi v T)^{\frac{p}{2}} +
  C_\beta^p \sigma_b^p \left( (pT)^{\frac{p}{4}} + (p \log
  T)^{\frac{p}{2} (1 + 2 \beta)} \right).
\end{align*}


\subsection{Proof of Theorem~\ref{thm:non-asymptotic-linfty}}
In order to prove the theorem, we require an auxiliary lemma that provides an almost-sure bound for the $\ell_\infty$ norm of the process.

Let $(\coordinate_1, \coordinate_2, \cdots, \coordinate_d)$ denote the standard orthonormal basis of $\real^d$. We consider the projection of error terms onto the set of vectors $\testvector_i \mydefn (\Amat^{-1})^\top \coordinate_i$ for $i = 1, 2, \cdots, d$. We first note that by Assumption~\ref{assume-linfty-contraction}, we have:
\begin{align*}
   \vecnorm{\testvector_i}{1} - 1 \leq \vecnorm{\testvector_i - \coordinate_i}{1} \leq \vecnorm{\testvector_i - \Amat^\top \testvector_i}{1} = \sup_{\vecnorm{u}{\infty} \leq 1} \testvector_i^\top (\IdMat - \Amat) u \leq (1 - \linftygap) \vecnorm{\testvector_i}{1},
\end{align*}
and consequently, $\vecnorm{\testvector_i}{1} \leq \linftygap^{-1}$.

We consider the martingales $M_t^{(\testvector_i)}$ for each $i = 1,2, \cdots, d$. Similar to the proof of Theorem~\ref{thm-non-asymptotic}, we use the BDG inequality and decompose the deviation into three terms:
\begin{align*}
    \Exs \sup_{0 \leq t \leq T} \abss{M_t^{(\testvector_i)}}^p \leq (Cp)^{\frac{p}{2}} \Exs \langle M_t^{(\testvector_i)} \rangle_T^{\frac{p}{2}} \leq (3 Cp)^{\frac{p}{2}} \left(I_1 + I_2 + I_3 \right),
\end{align*}
where $I_1 \defn \Exs \left( \sum_{t = 0}^{T - 1} ( \testvector_i^\top
\NoiseAplain_{t + 1} \theta_t)^2 \right)^\frac{p}{2}$, along with
\begin{align*}
I_2 \defn \Exs \left(  \sum_{t = 0}^{T - 1} 
(\noisebplain_{t + 1}^\top \testvector_i)^2 \right)^\frac{p}{2}, \quad \mbox{and}
\quad I_3 \defn \Exs  \abss{\sum_{s = 0}^{T - 1}
  2 (\testvector_i^\top \NoiseAplain_{s + 1} \theta_s^\top)(\testvector_i^\top
  \noisebplain_s)}^{\frac{p}{2}}.
\end{align*}
By Cauchy-Schwartz, we know that $I_3 \leq \sqrt{I_1 I_2} \leq \frac{1}{2} (I_1 + I_2)$. We now give upper bounds on the terms $I_1$ and $I_2$, respectively.

\paragraph{Upper bound for $I_2$:}
For the term $I_2$, note that the terms $(\noisebt^\top \testvector_i)$ are $\mathrm{i.i.d.}$ random variables. And by Assumption~\ref{assume-noise-structure-linfty},
$\abss{\noisebt^\top \testvector} \leq \vecnorm{\noisebt}{\infty} \cdot \vecnorm{\testvector_i}{1} \leq \linftygap^{-1}$. A simple application of Hoeffding's inequality leads to:
\begin{align*}
    \forall \varepsilon > 0, \quad \Prob \left( \abss{\frac{1}{T} \sum_{t = 0}^{T - 1} (\noisebt^\top \testvector_i)^2 - \Exs (\noiseb^\top \testvector_i)^2} > \varepsilon \right) \leq 2 \exp \left( - T \varepsilon^2 \linftygap^4 \right),
\end{align*}
which can be easily converted into a moment bound:
\begin{align*}
    I_1^{\frac{2}{p}} \leq C \left( T \cdot \Exs (\noiseb^\top \testvector_i)^2 + p \sqrt{T} \linftygap^{-2} \right).
\end{align*}

\paragraph{Upper bound for $I_1$:}
As in the proof of Lemma~\ref{LemIbounds}, we decompose the sequence into a martingale term and a predictable sequence. Let
$\Psi_T \mydefn \sum_{t = 1}^T \Exs\left( (\testvector_i^\top \NoiseAt \theta_t)^2
| \filtration_t \right)$, and let $\Upsilon_T \mydefn \sum_{t = 1}^T
(\testvector_i^\top \NoiseAt \theta_t)^2 - \Psi_T$. By definition, it is easy to
see that $\Upsilon$ is a martingale. Note that for each term in $\Upsilon$, by Lemma~\ref{lemma:linfty-estimate-and-contraction} and Assumption~\ref{assume-noise-structure-linfty}, we have:
\begin{align*}
    \abss{(\testvector_i^\top \NoiseAt \theta_t)^2 - \Exs ((\testvector_i^\top \NoiseAt \theta_t)^2 | \filtration_{t})} \leq 2 \abss{(\testvector_i^\top \NoiseAt \theta_t)^2} \leq 2 \vecnorm{\testvector_i}{1}^2 \cdot \vecnorm{\NoiseAt \theta_t}{\infty}^2 \leq 2 \vecnorm{\testvector_i}{1}^2 \cdot \vecnorm{\theta_t}{\infty}^2 \leq 2 \linftygap^{-4}.
\end{align*}
By the Azuma-Hoeffding inequality, we obtain:
\begin{align*}
   \forall \varepsilon > 0, \quad \Prob \left( \frac{1}{T}\abss{\Upsilon_T} \geq \varepsilon \right) \leq 2 \exp (- T \varepsilon^2 \linftygap^{-8} / 4),
\end{align*}
which can easily be converted to a moment bound:
\begin{align*}
    \left(\Exs |\Upsilon_T|^{\frac{p}{2}} \right)^{\frac{2}{p}} \leq C p \sqrt{T} \linftygap^{-4}.
\end{align*}
Now we turn to an upper bound for the term $\Psi_T$. Define $\psi (\theta) \mydefn \Exs (v_i^\top \NoiseA \theta)^2$. Note that $\Psi_T$ is the partial sum of function $\psi$ applied to the Markov process $(\theta_t)_{t \geq 0}$. We seek to use the ergodic concentration inequalities based on Ricci curvature techniques~\cite{joulin2010curvature}.

First, we note that for $\theta_1, \theta_2 \in [- \linftygap^{-1}, \linftygap^{-1}]^d$, we have:
\begin{align*}
    \psi (\theta_1) - \psi (\theta_2) &= \Exs (v_i^\top \NoiseA \theta_1)^2 - \Exs (v_i^\top \NoiseA \theta_2)^2\\
    &= \Exs \left((v_i^\top \NoiseA \theta_1) (v_i^\top \NoiseA (\theta_1 - \theta_2)) \right) + \Exs \left((v_i^\top \NoiseA \theta_2) (v_i^\top \NoiseA (\theta_1 - \theta_2)) \right)\\
    &\leq \vecnorm{v_i}{1}^2 \Exs (\vecnorm{\NoiseA \theta_1}{\infty} \cdot \vecnorm{\NoiseA (\theta_1 - \theta_2)}{\infty} ) + \vecnorm{v_i}{1}^2 \Exs (\vecnorm{\NoiseA \theta_2}{\infty} \cdot \vecnorm{\NoiseA (\theta_1 - \theta_2)}{\infty} ) \\
    &\leq \linftygap^{-3} \vecnorm{\theta_1 - \theta_2}{\infty}.
\end{align*}
So $\psi$ is $\linftygap^{-3}$-Lipschitz under the $\vecnorm{\cdot}{\infty}$ norm, within the region $[- \linftygap^{-1}, \linftygap^{-1}]^d$.

Denote by $\transition$ the transition kernel of the Markov chain $(\theta_t)_{t \geq 0}$.
By Assumption~\ref{assume-linfty-contraction}, when we take the synchronous coupling by using the same oracle for the process starting at two different points, there is:
\begin{align*}
    \Wass_{\vecnorm{\cdot}{\infty}, 1} (\markovtransition \delta_{\theta_1}, \markovtransition \delta_{\theta_2} ) \leq \Exs \vecnorm{(I - \stepsize \Amat_t) (\theta_1 - \theta_2)}{\infty} \leq (1 - \stepsize \linftygap) \vecnorm{\theta_1 - \theta_2}{\infty}.
\end{align*}
So the Markov chain $(\theta_t)_{t \geq 0}$ is a $\Wass_1$ contraction with parameter $(1 - \stepsize \linftygap)$ under $\ell_\infty$ norm. Finally, by Assumption~\ref{assume-noise-structure-linfty}, we note that:
\begin{align*}
    \mathrm{diam}_{\vecnorm{\cdot}{\infty}} \left( \mathrm{supp} (\markovtransition \delta_{\theta}) \right) \leq \stepsize \left(1 + \vecnorm{\theta}{\infty} \right).
\end{align*}
So the support size of the one-step transition kernel within the region $[- \linftygap^{-1}, \linftygap^{-1}]^d$ is uniformly bounded by $2 \stepsize \linftygap^{-1}$.

We apply Proposition~\ref{prop:joulin-ollivier}, and obtain the following concentration inequality:
\begin{align*}
    \forall \varepsilon > 0, \quad \Prob \left( \abss{\frac{1}{T} \sum_{t = 0}^{T} (\psi (\theta_{t}) - \Exs \psi (\theta_{t}))} > \linftygap^{-3} \varepsilon \right) \leq \begin{cases}
    2 \exp \left( - \frac{\varepsilon^2 T \linftygap^2}{128 \stepsize^2} \right)& \varepsilon < \frac{8}{3} \linftygap^{-1},\\
    2 \exp \left( - \frac{\varepsilon T \linftygap}{24 \stepsize} \right), & \varepsilon > \frac{8}{3} \linftygap^{-1}.
    \end{cases}
\end{align*}
This tail probability bound can be easily translated into a moment bound:
\begin{align*}
   \left( \Exs \abss{\Psi_T}^{\frac{p}{2}} \right)^{\frac{2}{p}} \leq 2 \Exs \Psi_T + C \linftygap^{-4} \stepsize \left( \sqrt{T p} +  p \right),
\end{align*}
for a universal constant $C > 0$.

For the term $\Exs \Psi_T$, the $\Wass_1$ contraction implies that:
\begin{align*}
    \abss{\Exs \psi (\theta_t) - \Exs_{\pi_\stepsize} \psi (\theta) }\leq \linftygap^{-3} (1 - \stepsize \linftygap)^{t} \Exs \vecnorm{\theta_0 -\theta}{\infty} \leq \linftygap^{-5} (1 - \stepsize \linftygap)^{t}.
\end{align*}
So we obtain $\Exs \Psi_T \leq T \Exs_{\pi_\stepsize} \psi (\theta) + \sum_{t = 0}^{T} \linftygap^{-4} (1 - \stepsize \linftygap)^{t} \leq T ( (a_i \thetastar)^2 + a_i^\top \LamStar a_i) + \frac{1}{\stepsize \linftygap^5}$.

Putting these results together, we have:
\begin{align*}
    I_3^{\frac{2}{p}} \leq C T ( (\coordinate_i \thetastar)^2 + \coordinate_i^\top \LamStar a_i) + C \linftygap^{-4} p \stepsize \sqrt{T} + C \linftygap^{-5} \stepsize^{-1},
\end{align*}
and combining the upper bounds for $I_1$ and $I_2$, we obtain:
\begin{align*}
    \left( \Exs \sup_{0 \leq t \leq T} \abss{M_t^{(\testvector_i)}}^p \right)^{\frac{2}{p}} \leq C p T \coordinate_i^\top \GamStar \coordinate_i  +  C (\linftygap^{-4} \stepsize + \lambda^{-2} ) p \sqrt{T} + C \linftygap^{-5} \stepsize^{-1}.
\end{align*}
For the term $\frac{A^{-1} (\theta_0 - \thetastar)}{\stepsize T}$, we note that by Lemma~\ref{lemma:linfty-estimate-and-contraction}, we have $\vecnorm{\theta_0 - \theta_T}{\infty} \leq 2 \linftygap^{-1}$, and furthermore, we note that for any $v \in \real^d$, we have:
\begin{align*}
    \vecnorm{A^{-1} v}{\infty} = \vecnorm{(I - A) A^{-1} v}{\infty} + \vecnorm{v}{\infty} \leq \vecnorm{v}{\infty} + (1 - \linftygap) \vecnorm{A^{-1} v}{\infty},
\end{align*}
which leads to $\vecnorm{A^{-1} v}{\infty} \leq \linftygap^{-1}$, and consequently, we have $\vecnorm{A^{-1} (\theta_0 - \theta_T)}{\infty} \leq \frac{2}{\stepsize \linftygap^2}$ almost surely.

Putting these results together, we obtain:
\begin{align*}
    \left( \Exs \abss{\sqrt{T} \coordinate_i^\top (\thetabar_T - \thetastar)}^p \right)^{\frac{1}{p}} \leq C \sqrt{p \coordinate_i^\top \GamStar (\stepsize) \coordinate_i} + C (\linftygap^{-2}\stepsize + \linftygap^{-1} ) \sqrt{p} T^{-\frac{1}{4}} + C \linftygap^{-\frac{5}{2}} \stepsize^{-1}.
\end{align*}
Converting this bound into a high-probability bound and taking a union bound over the $d$ coordinates, for any $Q > 0$, we obtain:
\begin{align*}
    \Prob \left( \sqrt{T} \vecnorm{\thetabar_T - \thetastar}{\infty} \geq C \sqrt{Q} + C\frac{\linftygap^{-2}\stepsize + \linftygap^{-1} }{T^{\frac{1}{4}}} \sqrt{\log \frac{d}{\delta}}  + \frac{C \linftygap^{-\frac{5}{2}}}{ \stepsize \sqrt{T}} \right)
    \leq \sum_{i = 1}^d \exp \left( - \frac{Q}{\coordinate_i^\top \GamStar (\stepsize) \coordinate_i} \right).
\end{align*}
Take $Q = Q \left( (\coordinate_i^\top \GamStar (\stepsize) \coordinate_i)_{i = 1}^d; \delta \right) $ to obtain the result.


\subsection{Proof of Theorem~\ref{thm-critical}}

The proof is also based on the telescope
identity~\eqref{eq:telescope}. The key ingredient in the proof is an upper bound on the
second moment of $\vecnorm{\theta_t - \thetastar}{2}$, as stated in
the following:
\begin{lemma}
\label{lemma-bound-on-theta-critical}
Under Assumptions~\ref{assume-critical}, ~\ref{assume-second-moment}
and~\ref{assume-indp}, given a step size $\stepsize
\leq \frac{1}{( \rho (\Amatbar) + 3 \asymconditioning (\Umat) v_A) \sqrt{T}}$,
for any integer $t \geq 0$, we have
\begin{align*}
  \Exs \vecnorm{\theta_t - \thetastar}{2}^2 \leq e \asymconditioning^2
  (\Umat) \left( \Exs \vecnorm{\theta_0 - \thetastar}{2}^2 + \stepsize^2 t
  (v_b^2 d + v_A^2 \vecnorm{\thetastar}{2}^2) \right),
\end{align*}
where the matrix $U$ has columns composed of the eigenvectors of
$\Amatbar$.
\end{lemma}
\noindent See
Appendix~\ref{subsubsec:proof_of_lemma_lemma-bound-on-theta-critical}
for the proof of this claim.

Taking Lemma~\ref{lemma-bound-on-theta-critical} as given, we now
prove Theorem~\ref{thm-critical}.  By equation~\eqref{eq:telescope},
we have:
\begin{align*}
   \Exs \vecnorm{\Amatbar (\thetabar_T - \thetastar)}{2}^2 \leq
   \frac{4}{\stepsize^2 T^2} \left( \Exs \vecnorm{\theta_0 -
     \thetastar}{2}^2 + \Exs \vecnorm{\theta_T - \thetastar}{2}^2
   \right) + \frac{2}{T^2} \Exs \vecnorm{M_T}{2}^2.
\end{align*}
By Lemma~\ref{lemma-bound-on-theta-critical}, we have:
\begin{align*}
    \Exs \vecnorm{\theta_T - \thetastar}{2}^2 \leq e \kappa^2 (\Umat)
    \left( \Exs \vecnorm{\theta_0 - \thetastar}{2}^2 + 3 \stepsize^2 T
    (v_b^2 d + v_A^2 \vecnorm{\thetastar}{2}^2) \right).
\end{align*}
For the martingale term, note that:
\begin{align*}
    \Exs \vecnorm{M_T}{2}^2 &= \Exs \sum_{t = 0}^{T - 1} \vecnorm{e_{t
        + 1} (\theta_t)}{2}^2 \\
    & \leq 3 \Exs \sum_{t = 0}^{T - 1} \left( \Exs (\vecnorm{b_{t + 1}
      - b}{2}^2 \mid \filtration_t) + \Exs (\vecnorm{(\Amat_{t + 1} -
      \Amatbar) (\theta_t - \thetastar)}{2}^2 \mid \filtration_t) + \Exs
    (\vecnorm{(\Amat_{t + 1} - \Amatbar) \thetastar}{2}^2 \mid
    \filtration_t) \right) \\
    & \leq 3 \Exs \sum_{t = 0}^{T - 1} \left( v_b^2d + v_A^2
    \vecnorm{\theta_t - \thetastar}{2}^2 + v_A^2
    \vecnorm{\thetastar}{2}^2 \right)\\
    & \leq 3 T v_b^2 d + 3 T v_A^2 \vecnorm{\thetastar}{2}^2 + 3 T
    v_A^2 e \asymconditioning^2 (\Umat) \left( \Exs \vecnorm{\theta_0 -
      \thetastar}{2}^2 + \stepsize^2 T (v_b^2 d + v_A^2
    \vecnorm{\thetastar}{2}^2) \right).
\end{align*}
Since $\stepsize \in \Big( 0, \frac{1}{\sqrt{T} (\rho (\Amatbar) + 3
  \asymconditioning (\Umat) v_A)} \Big)$, we have:
\begin{align*}
  \Exs \vecnorm{M_T}{2}^2 \leq 3 T v_A^2 e \asymconditioning^2
  (\Umat) \Exs \vecnorm{\theta_0 - \thetastar}{2}^2 + (3 + e) T
  (v_b^2 d + v_A^2 \vecnorm{\thetastar}{2}^2).
\end{align*}

Putting together the pieces yields
\begin{align*}
\Exs \vecnorm{\Amatbar (\thetabar_T - \thetastar)}{2}^2 \leq C \left(
\frac{\asymconditioning^2 (\Umat)}{\stepsize^2 T^2} \Exs
\vecnorm{\theta_0 - \thetastar}{2}^2 + \frac{v_b^2 d + v_A^2
  \vecnorm{\thetastar}{2}^2}{T} + \frac{v_A^2 \asymconditioning^2
  (\Umat)}{T} \Exs \vecnorm{\theta_0 - \thetastar}{2}^2 \right).
\end{align*}
Setting the step size as $\stepsize = \frac{1}{(\rho(\Amatbar) + 3
  \asymconditioning (\Umat) v_A) \sqrt{T}}$ yields the claim.


\section{Discussion}
\label{SecDiscussion}

In this paper, we established several new results for constant
step-size linear stochastic approximation combined with Polyak-Ruppert
averaging.  In the case where $\Amatbar$ is a Hurwitz matrix, a
central limit theorem is proven, with asymptotic covariance
characterizing the effect of the constant step size.  Non-asymptotically,
we derive high-probability concentration bounds for the averaged iterates
in any direction, whose leading term matches the asymptotic variance
and has poly-logarithmic dependence on the failure probability.  We
also study the critical case where the real part of eigenvalues are
only guaranteed to be non-negative, and establish a gap-independent
$\order{1/T}$ rate in mean-squared error. We illustrate the
effectiveness of our abstract results by considering momentum SGD for
linear regression and TD learning, and uncover new aspects of the LSA
approach to these problems.


\subsection*{Acknowledgements}

This research was partially supported by Office of Naval Research
Grant DOD ONR-N00014-18-1-2640 and National Science Foundation Grant
DMS-1612948 to MJW, and a grant from the DARPA program on Lifelong Learning Machines to MIJ.


\bibliographystyle{plainnat}
\bibliography{references}

\begin{thebibliography}{10}

\bibitem{agarwal2012information}
Alekh Agarwal, Peter~L Bartlett, Pradeep Ravikumar, and Martin~J Wainwright.
\newblock Information-theoretic lower bounds on the oracle complexity of
  stochastic convex optimization.
\newblock {\em IEEE Transactions on Information Theory}, 58(5):3235--3249,
  2012.

\bibitem{BENVENISTE-METIVIER-PRIOURET}
Albert Benveniste, Michel M{\'e}tivier, and Pierre Priouret.
\newblock {\em Adaptive Algorithms and Stochastic Approximations}.
\newblock Springer, 2012.

\bibitem{Bertsekas_dyn1}
Dimitris~P. Bertsekas.
\newblock {\em Dynamic Programming and Stochastic Control}, volume~1.
\newblock Athena Scientific, Belmont, MA, 1995.

\bibitem{BERTSEKAS-TSITSIKLIS}
Dimitris~P Bertsekas and John~N Tsitsiklis.
\newblock {\em Parallel and Distributed Computation: Numerical Methods}.
\newblock Athena Scientific, 1989.

\bibitem{bhandari2018finite}
Jalaj Bhandari, Daniel Russo, and Raghav Singal.
\newblock A finite time analysis of temporal difference learning with linear
  function approximation.
\newblock {\em arXiv preprint arXiv:1806.02450}, 2018.

\bibitem{BORKAR}
Vivek~S Borkar.
\newblock {\em Stochastic Approximation: A Dynamical Systems Viewpoint}.
\newblock Cambridge University Press, 2008.

\bibitem{bottou2016optimization}
L{\'e}on Bottou, Frank~E Curtis, and Jorge Nocedal.
\newblock Optimization methods for large-scale machine learning.
\newblock {\em arXiv preprint arXiv:1606.04838}, 2016.

\bibitem{burkholder1972integral}
Donald~L Burkholder, Burgess~J Davis, and Richard~F Gundy.
\newblock Integral inequalities for convex functions of operators on
  martingales.
\newblock In {\em Proc. Sixth Berkeley Symp. Math. Statist. Prob}, volume~2,
  pages 223--240, 1972.

\bibitem{chen2016statistical}
Xi~Chen, Jason~D Lee, Xin~T Tong, and Yichen Zhang.
\newblock Statistical inference for model parameters in stochastic gradient
  descent.
\newblock {\em arXiv preprint arXiv:1610.08637}, 2016.

\bibitem{defazio2014saga}
Aaron Defazio, Francis Bach, and Simon Lacoste-Julien.
\newblock Saga: A fast incremental gradient method with support for
  non-strongly convex composite objectives.
\newblock In {\em Advances in Neural Information Processing Systems}, pages
  1646--1654, 2014.

\bibitem{dieuleveut2017bridging}
Aymeric Dieuleveut, Alain Durmus, and Francis Bach.
\newblock Bridging the gap between constant step size stochastic gradient
  descent and {M}arkov chains.
\newblock {\em arXiv preprint arXiv:1707.06386}, 2017.

\bibitem{dieuleveut2017harder}
Aymeric Dieuleveut, Nicolas Flammarion, and Francis Bach.
\newblock Harder, better, faster, stronger convergence rates for least-squares
  regression.
\newblock {\em The Journal of Machine Learning Research}, 18(1):3520--3570,
  2017.

\bibitem{fan2018statistical}
Jianqing Fan, Wenyan Gong, Chris~Junchi Li, and Qiang Sun.
\newblock Statistical sparse online regression: A diffusion approximation
  perspective.
\newblock In {\em International Conference on Artificial Intelligence and
  Statistics}, pages 1017--1026, 2018.

\bibitem{freedman1975tail}
David~A Freedman et~al.
\newblock On tail probabilities for martingales.
\newblock {\em the Annals of Probability}, 3(1):100--118, 1975.

\bibitem{HALL-HEYDE}
Peter Hall and Christopher~C Heyde.
\newblock {\em Martingale Limit Theory and Its Application}.
\newblock Academic Press, 1980.

\bibitem{jain2018accelerating}
Prateek Jain, Sham~M Kakade, Rahul Kidambi, Praneeth Netrapalli, and Aaron
  Sidford.
\newblock Accelerating stochastic gradient descent for least squares
  regression.
\newblock In {\em Conference On Learning Theory}, pages 545--604, 2018.

\bibitem{jain2019making}
Prateek Jain, Dheeraj Nagaraj, and Praneeth Netrapalli.
\newblock Making the last iterate of {SGD} information theoretically optimal.
\newblock {\em arXiv preprint arXiv:1904.12443}, 2019.

\bibitem{jain2017parallelizing}
Prateek Jain, Praneeth Netrapalli, Sham~M Kakade, Rahul Kidambi, and Aaron
  Sidford.
\newblock Parallelizing stochastic gradient descent for least squares
  regression: mini-batching, averaging, and model misspecification.
\newblock {\em The Journal of Machine Learning Research}, 18(1):8258--8299,
  2017.

\bibitem{johnson2013accelerating}
Rie Johnson and Tong Zhang.
\newblock Accelerating stochastic gradient descent using predictive variance
  reduction.
\newblock In {\em Advances in Neural Information Processing Systems}, pages
  315--323, 2013.

\bibitem{joulin2010curvature}
Ald{\'e}ric Joulin and Yann Ollivier.
\newblock Curvature, concentration and error estimates for {M}arkov chain
  {M}onte {C}arlo.
\newblock {\em The Annals of Probability}, 38(6):2418--2442, 2010.

\bibitem{kallenberg2006foundations}
Olav Kallenberg.
\newblock {\em Foundations of Modern Probability}.
\newblock Springer Science \& Business Media, 2006.

\bibitem{karimi2019non}
Belhal Karimi, Blazej Miasojedow, {\'E}ric Moulines, and Hoi-To Wai.
\newblock Non-asymptotic analysis of biased stochastic approximation scheme.
\newblock {\em arXiv preprint arXiv:1902.00629}, 2019.

\bibitem{KUSHNER-YIN}
Harold Kushner and G~George Yin.
\newblock {\em Stochastic Approximation and Recursive Algorithms and
  Applications}, volume~35.
\newblock Springer, 2003.

\bibitem{lai2003stochastic}
Tze~Leung Lai.
\newblock Stochastic approximation.
\newblock {\em The Annals of Statistics}, 31(2):391--406, 2003.

\bibitem{lakshminarayanan2018linear}
Chandrashekar Lakshminarayanan and Csaba Szepesvari.
\newblock Linear stochastic approximation: How far does constant step-size and
  iterate averaging go?
\newblock In {\em International Conference on Artificial Intelligence and
  Statistics}, pages 1347--1355, 2018.

\bibitem{li2018statistical}
Tianyang Li, Liu Liu, Anastasios Kyrillidis, and Constantine Caramanis.
\newblock Statistical inference using {SGD}.
\newblock In {\em Thirty-Second AAAI Conference on Artificial Intelligence},
  2018.

\bibitem{liang2019statistical}
Tengyuan Liang and Weijie~J Su.
\newblock Statistical inference for the population landscape via
  moment-adjusted stochastic gradients.
\newblock {\em Journal of the Royal Statistical Society: Series B (Statistical
  Methodology)}, 81(2):431--456, 2019.

\bibitem{mandt2017stochastic}
Stephan Mandt, Matthew~D Hoffman, and David~M Blei.
\newblock Stochastic gradient descent as approximate {B}ayesian inference.
\newblock {\em The Journal of Machine Learning Research}, 18(1):4873--4907,
  2017.

\bibitem{moulines2011non}
Eric Moulines and Francis~R Bach.
\newblock Non-asymptotic analysis of stochastic approximation algorithms for
  machine learning.
\newblock In {\em Advances in Neural Information Processing Systems}, pages
  451--459, 2011.

\bibitem{nemirovski2009robust}
Arkadi Nemirovski, Anatoli Juditsky, Guanghui Lan, and Alexander Shapiro.
\newblock Robust stochastic approximation approach to stochastic programming.
\newblock {\em SIAM Journal on optimization}, 19(4):1574--1609, 2009.

\bibitem{NEMIROVSKII-YUDIN}
Arkadii Nemirovskii and David~Borisovich Yudin.
\newblock {\em Problem Complexity and Method Efficiency in Optimization}.
\newblock Wiley, 1983.

\bibitem{palaniappan2016stochastic}
Balamurugan Palaniappan and Francis Bach.
\newblock Stochastic variance reduction methods for saddle-point problems.
\newblock In {\em Advances in Neural Information Processing Systems}, pages
  1416--1424, 2016.

\bibitem{pananjady2019value}
Ashwin Pananjady and Martin~J Wainwright.
\newblock Value function estimation in {M}arkov reward processes:
  {I}nstance-dependent $\ell_\infty$-bounds for policy evaluation.
\newblock {\em arXiv preprint arXiv:1909.08749}, 2019.

\bibitem{perko2013differential}
Lawrence Perko.
\newblock {\em Differential Equations and Dynamical Systems}, volume~7.
\newblock Springer Science \& Business Media, 2013.

\bibitem{polyak1992acceleration}
Boris~T Polyak and Anatoli~B Juditsky.
\newblock Acceleration of stochastic approximation by averaging.
\newblock {\em SIAM Journal on Control and Optimization}, 30(4):838--855, 1992.

\bibitem{Puterman05}
Mark~L. Puterman.
\newblock {\em Markov Decision Processes: Discrete Stochastic Dynamic
  Programming}.
\newblock Wiley, 2005.

\bibitem{rakhlin2012making}
Alexander Rakhlin, Ohad Shamir, and Karthik Sridharan.
\newblock Making gradient descent optimal for strongly convex stochastic
  optimization.
\newblock In {\em Proceedings of the 29th International Conference on Machine
  Learning}, pages 449--456, 2012.

\bibitem{robbins1951stochastic}
Herbert Robbins and Sutton Monro.
\newblock A stochastic approximation method.
\newblock {\em The Annals of Mathematical Statistics}, pages 400--407, 1951.

\bibitem{rockafellar1970monotone}
R~Tyrrell Rockafellar.
\newblock Monotone operators associated with saddle-functions and minimax
  problems.
\newblock {\em Nonlinear Functional Analysis}, 18(part 1):397--407, 1970.

\bibitem{roux2012stochastic}
Nicolas~L Roux, Mark Schmidt, and Francis~R Bach.
\newblock A stochastic gradient method with an exponential convergence rate for
  finite training sets.
\newblock In {\em Advances in Neural Information Processing Systems}, pages
  2663--2671, 2012.

\bibitem{ruppert1988efficient}
David Ruppert.
\newblock Efficient estimations from a slowly convergent robbins-monro process.
\newblock {\em Technical Report, Cornell University Operations Research and
  Industrial Engineering}, 1988.

\bibitem{shamir2013stochastic}
Ohad Shamir and Tong Zhang.
\newblock Stochastic gradient descent for non-smooth optimization: Convergence
  results and optimal averaging schemes.
\newblock In {\em Proceedings of the 30th International Conference on Machine
  Learning}, pages 71--79, 2013.

\bibitem{su2018uncertainty}
Weijie~J Su and Yuancheng Zhu.
\newblock Uncertainty quantification for online learning and stochastic
  approximation via hierarchical incremental gradient descent.
\newblock {\em arXiv preprint arXiv:1802.04876}, 2018.

\bibitem{sutton1988learning}
Richard~S Sutton.
\newblock Learning to predict by the methods of temporal differences.
\newblock {\em Machine learning}, 3(1):9--44, 1988.

\bibitem{SutBar18}
Richard~S. Sutton and Andrew~G. Barto.
\newblock {\em {R}einforcement {L}earning: {A}n {I}ntroduction}.
\newblock {M}{I}{T} {P}ress, Cambridge, {M}{A}, 2nd edition, 2018.

\bibitem{tsitsiklis2002average}
John~N Tsitsiklis and Benjamin Van~Roy.
\newblock On average versus discounted reward temporal-difference learning.
\newblock {\em Machine Learning}, 49(2-3):179--191, 2002.

\bibitem{villani2008optimal}
C{\'e}dric Villani.
\newblock {\em Optimal Transport: Old and New}, volume 338.
\newblock Springer Science \& Business Media, 2008.

\bibitem{wainwright2019high}
Martin~J Wainwright.
\newblock {\em High-Dimensional Statistics: A Non-Asymptotic Viewpoint},
  volume~48.
\newblock Cambridge University Press, 2019.

\bibitem{wainwright2019stochastic}
Martin~J Wainwright.
\newblock Stochastic approximation with cone-contractive operators: Sharp
  $\ell_\infty$-bounds for {Q}-learning.
\newblock {\em arXiv preprint arXiv:1905.06265}, 2019.

\bibitem{wainwright2019variance}
Martin~J Wainwright.
\newblock Variance-reduced {Q}-learning is minimax optimal.
\newblock {\em arXiv preprint arXiv:1906.04697}, 2019.

\bibitem{wang2016stochastic}
Mengdi Wang and Dimitri~P Bertsekas.
\newblock Stochastic first-order methods with random constraint projection.
\newblock {\em SIAM Journal on Optimization}, 26(1):681--717, 2016.

\bibitem{watkins1992q}
Christopher~JCH Watkins and Peter Dayan.
\newblock Q-learning.
\newblock {\em Machine Learning}, 8(3-4):279--292, 1992.

\end{thebibliography}

\appendix

\section{Proof of Lemma~\ref{LemHurwitz}}
\label{AppHurwitz}

In this appendix, we prove Lemma~\ref{LemHurwitz}.  This lemma is a standard fact in linear algebra; for instance, see Section
1.8 in the book~\cite{perko2013differential}.  We include the proof for completeness and so as to extract the behavior of $\spectralgap$.

When the matrix $\Amatbar$ is diagonalizable, we can write $\Amatbar = \Umat
D \Umat^{-1}$, which implies the stronger lower bound \mbox{$D + D^\hc
  \succeq 2 \min_{i \in [d]} \RealPart(\lambda_i(\Amatbar))$.}  For a
non-diagonalizable matrix $\Amatbar$, we instead write $\Amatbar = \Umat J
\Umat^{-1}$, where the matrix $J = \mathrm{diag} (\lambda_i I_{d_i} +
J_{d_i})_{i = 1}^k$ contains the Jordan decomposition. For each Jordan
block, we note that for $Q_i \mydefn \diag (1, \RealPart (\lambda_i /
2), \cdots, \RealPart (\lambda_i / 2)^{d_i - 1})$, we have
\begin{align*}
  Q_i^{-1} (\lambda_i I_{d_i} + J_{d_i}) Q_i = \lambda_i I_{d_i} +
  \RealPart (\lambda_i / 2) J_{d_i} \mydefn B_i.
\end{align*}

We note that $A$ is similar to $\diag (B_1, B_2, \cdots, B_k)$. We
only need to study the eigenvalues of $B_i + B_i^\hc$. A straightforward
calculation yields:
\begin{align*}
  B_i + B_i^\hc = \frac{1}{2} \RealPart (\lambda_i)
  \begin{bmatrix}
      4 & 1 & 0 & \cdots & 0 \\
      1 & 4 & 1 & \cdots & 0 \\
      & & \cdots & & \\
      0 & \cdots & 1 & 4 & 1 \\
      0 & \cdots & 0 & 1 & 4
  \end{bmatrix}
& \mydefn \RealPart(\lambda_i) T_{d_i}.
\end{align*}
Note that the matrix $T_{d_i}$ is a symmetric tridiagonal Toeplitz
matrix, whose eigenvalues are given by the formula $\lambda_j
(T_{d_i}) = 4 + 2 \cos \left( \frac{j \pi}{(d_i + 1)} \right) \geq
2$. Therefore, we have $B_i + B_i^\hc \succeq \RealPart (\lambda_i)$,
which completes the proof.


\section{Properties of the process $\{\theta_t\}_{t \geq 0}$}
\label{AppProcessProperties}

In this appendix, we prove a number of claims about the basic
properties of the process $\{\theta_t\}_{t \geq 0}$.

\subsection{Proof of Lemma~\ref{lemma:simple-l2-estimate}}
\label{sec:proof:lemma:simple-l2-estimate}

Recall that we use $r_t = \theta_t - \thetastar$ to denote the error
in the process at time $t$.  We make use of the function $f(r) = \Exs
\vecnorm{U^{-1} r}{2}^2$ for a Lyapunov-type analysis.  Observe that
the error satisfies the recursion
\begin{align*}
  r_{t + 1} = r_t - \stepsize (\Amat_{t + 1} \theta_t - \bvec_{t + 1}) =
  (\IdMat - \stepsize \Amatbar) r_t - \stepsize \NoiseAplain_{t + 1} \theta_t +
  \stepsize \noisebplain_{t + 1}.
\end{align*}
Turning to the squared Euclidean norm, we have
\begin{align*}
  \Exs \vecnorm{U^{-1}r_{t + 1}}{2}^2 = \Exs \vecnorm{ U^{-1} (\IdMat -
    \stepsize \Amatbar) r_t}{2}^2 + \stepsize^2 \Exs
  \vecnorm{U^{-1}(\NoiseAplain_{t + 1} \theta_t + \noisebplain_{t + 1})}{2}^2,
\end{align*}
where we have expanded the quadratic term and used the i.i.d.\ condition
(Assumption~\ref{assume-indp}).  Examining the first term, we have
\begin{align*}
\vecnorm{ U^{-1} (\IdMat - \stepsize \Amatbar) r_t}{2}^2 & = \vecnorm{
  (\IdMat - \stepsize U^{-1} A U) U^{-1} r_t}{2}^2 \\
& = \vecnorm{ U^{-1} r_t}{2}^2 - \stepsize (U^{-1} r_t)^{\hc} (D +
D^\hc) U^{-1} r_t + \opnorm{D^\hc D} \vecnorm{U^{-1} r_t}{2}^2 \\
& \leq \Big \{ 1 - 2 \stepsize \spectralgap + \stepsize^2
\rho^2(\Amatbar) \Big \} \vecnorm{U^{-1} r_t}{2}^2.
\end{align*}
For the second term, by Assumption~\ref{assume-second-moment} and
Assumption~\ref{assume-indp}, we have:
\begin{multline*}
  \Exs \vecnorm{U^{-1}(\NoiseAplain_{t + 1} \theta_t + \noisebplain_{t + 1})}{2}^2
  \leq \opnorm{U^{-1}}^2 \Exs \vecnorm{\NoiseAplain_{t + 1} (\thetastar +
    r_t) + \noisebplain_{t + 1} }{2}^2\\ = \opnorm{U^{-1}}^2 \left( \Exs
  \vecnorm{\NoiseAplain_{t + 1} (\thetastar+r_t)}{2}^2 + \Exs
  \vecnorm{\noisebplain_{t + 1} }{2}^2\right) \leq \opnorm{U^{-1}}^2 \left(
  v_A^2 (\vecnorm{\thetastar}{2} + \Exs \vecnorm{r_t}{2}^2) + v_b^2
  d\right).
\end{multline*}
Putting the pieces together and using the fact that $\stepsize \in
\Big(0 , \frac{\spectralgap}{\rho^2 (\Amatbar) + \asymconditioning^2 (\Umat)
  v_A^2} \Big)$, we find that
\begin{align*}
  \Exs \vecnorm{U^{-1} r_{t + 1}}{2}^2 &\leq (1 - 2 \stepsize
  \spectralgap + \stepsize^2 (\rho^2 (\Amatbar) + \asymconditioning^2 (\Umat)
  v_A^2)) \Exs \vecnorm{U^{-1} r_t}{2}^2 + \stepsize^2
  \opnorm{U^{-1}}^2 (v_A^2 \vecnorm{\thetastar}{2}^2 + v_b^2
  d)\\ &\leq (1 - \stepsize \spectralgap) \Exs \vecnorm{U^{-1}
    r_t}{2}^2 + \stepsize^2 \opnorm{U^{-1}}^2 (v_A^2
  \vecnorm{\thetastar}{2}^2 + v_b^2 d).
\end{align*}
By induction, it is easy to show that for any $t \geq 0$,
\begin{align*}
  \Exs \vecnorm{U^{-1} r_t}{2}^2 \leq \Exs \vecnorm{U^{-1} (\theta_0 -
    \thetastar)}{2}^2 + \frac{\stepsize}{\spectralgap}
  \opnorm{U^{-1}}^2 (v_A^2 \vecnorm{\thetastar}{2}^2 + v_b^2 d),
     \end{align*}
and consequently, we have the bound
\begin{align*}
  \Exs \vecnorm{r_t}{2}^2 \leq \asymconditioning^2 (\Umat) \left( \Exs
  \vecnorm{\theta_0 - \thetastar}{2}^2 +
  \frac{\stepsize}{\spectralgap} (v_A^2 \vecnorm{\thetastar}{2}^2 +
  v_b^2 d) \right).
\end{align*}


\paragraph{Proof of the bound~\eqref{Eqn2plusMoment}:}

In establishing this bound, we use the fact that for scalars $A > 0$,
$z \in (-A, +\infty)$ and  $\alpha \in (0, 1)$, we have
\begin{align*}
  (A + z)^{1 + \alpha} \leq A^{1 + \alpha} + (1 + \alpha) A^\alpha z +
  |z|^{1 + \alpha}.
\end{align*}
The proof of this inequality is straightforward: by homogeneity, we
only need to prove for the case of $A = 1$. Let $f (z) \mydefn 1 + (1
+ \alpha) z + |z|^{1 + \alpha} - (1 + z)^{1 + \alpha}$ for $z \in (-1,
+ \infty)$. It is easy to see that $f' (z) > 0$ for $z > 0$ and $f'
(z) < 0$ for $z < 0$.
    
By Assumption~\ref{assume-hurwitz}, we have
\begin{align*}
  \vecnorm{U^{-1} r_{t + 1}}{2}^2 \leq (1 - 2 \stepsize \spectralgap)
  \vecnorm{U^{-1} r_t}{2}^2 + 2 \stepsize \RealPart (\inprod{U^{-1}
    (1 - \stepsize \Amatbar) r_t}{U^{-1} e_{t + 1} (\theta_t)}) +
  \stepsize^2 \vecnorm{U^{-1} e_t}{2}^2.
\end{align*}
Taking the $(1 + \alpha / 2)$-order moment, by the scalar inequality, we obtain:
\begin{align*}
  \Exs \vecnorm{U^{-1} r_{t + 1}}{2}^{2 + \alpha} &\leq (1 - 2
  \stepsize \spectralgap) \Exs \vecnorm{U^{-1} r_t}{2}^{2 + \alpha} +
  \Exs \abss{2 \stepsize \RealPart (\inprod{U^{-1} (1 - \stepsize
      \Amatbar) r_t}{U^{-1} e_{t + 1} (\theta_t)}) + \stepsize^2
    \vecnorm{U^{-1} e_t}{2}^2 }^{1 + \alpha} \\
& + \Exs \left[\left( (1 - 2 \stepsize \spectralgap) \vecnorm{U^{-1}
      r_t}{2}^2\right)^{\frac{\alpha}{2}} \left(2 \stepsize
    \RealPart (\inprod{U^{-1} (1 - \stepsize \Amatbar) r_t}{U^{-1} e_{t
        + 1} (\theta_t)}) + \stepsize^2 \vecnorm{U^{-1} e_t}{2}^2
    \right)\right].
\end{align*}
Note that $\Exs (e_{t + 1} (\theta_t)| \filtration_t) = 0$. The last
term equals $\Exs \left[\left( (1 - 2 \stepsize \spectralgap)
  \vecnorm{U^{-1} r_t}{2}^2\right)^{\frac{\alpha}{2}} \stepsize^2
  \vecnorm{U^{-1} e_t}{2}^2 \right]$.
    
By the existence of $(2 + \alpha)$-order moment, there exists constant
$M_1, M_2 > 0$ such that:
\begin{align*}
  \Exs \abss{2 \stepsize \RealPart (\inprod{U^{-1} (1 - \stepsize
      \Amatbar) r_t}{U^{-1} e_{t + 1} (\theta_t)}) + \stepsize^2
    \vecnorm{U^{-1} e_t}{2}^2 }^{1 + \alpha} &\leq \stepsize^{1 +
    \alpha}\left( M_1 + M_2 \Exs \vecnorm{U^{-1} r_t}{2}^{2 + \alpha}
  \right) \\
  \Exs \left[\left( (1 - 2 \stepsize \spectralgap) \vecnorm{U^{-1}
  r_t}{2}^2\right)^{\frac{\alpha}{2}} \stepsize^2 \vecnorm{U^{-1}
      e_t}{2}^2 \right] &\leq \stepsize^2 \left( M_1 + M_2 \Exs
      \vecnorm{U^{-1} r_t}{2}^{2 + \alpha} \right).
\end{align*}
Thus we obtain:
\begin{align*}
  \Exs \vecnorm{U^{-1}r_{t + 1}}{2}^{2 + \alpha} \leq (1 - 2 \stepsize
  \spectralgap) \Exs \vecnorm{U^{-1}r_t}{2}^{2 + \alpha} +
  (\stepsize^{1 + \alpha} + \stepsize^2) \left( M_1 + M_2 \Exs
  \vecnorm{U^{-1} r_t}{2}^{2 + \alpha} \right).
\end{align*}
For $\stepsize < \stepsize_0 = \frac{1}{2} (\spectralgap /
M_2)^{\frac{1}{\alpha}}$, we have: $\Exs \vecnorm{U^{-1}r_{t +
    1}}{2}^{2 + \alpha} \leq (1 - \stepsize \spectralgap) \Exs
\vecnorm{U^{-1}r_t}{2}^{2 + \alpha} + \stepsize^{1 + \alpha} M_1$. An
induction proof argument leads to $\Exs \vecnorm{U^{-1}r_{t}}{2}^{2 +
  \alpha} \leq \Exs \vecnorm{U^{-1}r_0}{2}^{2 + \alpha} +
\frac{\stepsize^\alpha}{\spectralgap} M_1$ for any $t \geq 0$.

\subsection{Proof of Lemma~\ref{lemma:stationary-existence-uniqueness}}
\label{sec:stationary-distribution}

In proving this lemma, we make use of
Lemma~\ref{lemma:mixing-hurwitz}; for $z_t \mydefn U^{-1} r_t$, there
exists a pathwise coupling such that for any starting points
$z_0^{(1)}, z_0^{(2)}$, we have $\Exs \vecnorm{z_{t + 1}^{(1)} - z_{t
    + 1}^{(2)}}{2}^2 \leq e^{ - \spectralgap \stepsize} \Exs
\vecnorm{z_{t}^{(1)} - z_{t}^{(2)}}{2}^2$.  (Note that the proof of
Lemma~\ref{lemma:mixing-hurwitz} does not use any results from this
proof.)

We first show the existence and uniqueness of the stationary
distribution, as well as the existence of the second moment. Then we
calculate the first and second moment under the stationary
distribution.

\subsubsection{Proof of existence}

Since $\real^d$ is separable and complete, the Wasserstein space
$\Wass^2$ is complete~\cite{villani2008optimal}. Therefore, it
suffices to show that $\{\law(\theta_t)\}_{t = 0}^{+\infty}$ is a
Cauchy sequence in this space.
    
Given $\mu \in \Wass^2$ and taking $\theta_0 \sim \mu$, take any
positive integer $N > 0$, for any $k \geq N$ and $m \geq 0$, and we seek
to upper bound $\Wass_2 (\law (\theta_k), \law (\theta_{k +
  m}))$. Consider the process with two different initial points
$\theta_0^{(1)} \sim \mu$ and $\theta_0^{(2)} \sim \law (\theta_m)$,
coupled in an arbitrary way. By Lemma~\ref{lemma:mixing-hurwitz}, we
have:
\begin{align*}
\Wass_2 \left( \law (\theta_k^{(1)}), \law (\theta_k^{(2)})\right)
\leq e^{- \frac{\spectralgap \stepsize k}{2}} \asymconditioning
(\Umat) \sqrt{\Exs \vecnorm{\theta_0^{(1)} - \theta_0^{(2)}}{2}^2}
\leq e^{- \frac{\spectralgap \stepsize N}{2}} \asymconditioning
(\Umat) \sqrt{2 \sup_{t \geq 0} \Exs \vecnorm{\theta_t -
    \thetastar}{2}^2}.
    \end{align*}
Moreover, by Lemma~\ref{lemma:simple-l2-estimate}, we have $\sup_{t
  \geq 0} \Exs \vecnorm{\theta_t - \thetastar}{2}^2 \leq
\asymconditioning^2 (\Umat) \left( \Exs \vecnorm{\theta_0 -
  \thetastar}{2}^2 + \frac{\stepsize}{\spectralgap} (v_A^2
\vecnorm{\thetastar}{2}^2 + v_b^2 d) \right)$ is a finite constant
independent of $N$. Therefore, $(\law (\theta_t))_{t \geq 0}$ is a
Cauchy sequence in the space $\Wass^2$. The limit exists in $\Wass^2$.

\subsubsection{Proof of uniqueness}

Suppose that there were two stationary measures $\pi^{(1)}$ and
$\pi^{(2)}$, let $\theta_t^{(i)} \sim \pi^{(i)}$ for $i = 1,2$, with
an optimal coupling such that:
\begin{align*}
  \Exs \vecnorm{\theta_t^{(1)} - \theta_t^{(2)}}{2}^2 = \Wass_2^2
  (\pi^{(1)}, \pi^{(2)}).
\end{align*}
By stationarity, we have $\theta_{t + 1}^{(i)} \sim \pi^{(i)}$, and
consequently:
\begin{align*}
  \Wass_2^2 (\pi^{(1)}, \pi^{(2)}) \leq \Exs \vecnorm{\theta_{t +
      1}^{(1)} - \theta_{t + 1}^{(2)}}{2}^2 \leq e^{- \stepsize
    \spectralgap} \Exs \vecnorm{\theta_t^{(1)} -
    \theta_t^{(2)}}{2}^2 = e^{- \stepsize \spectralgap} \Wass_2^2
  (\pi^{(1)}, \pi^{(2)}),
\end{align*}
which implies $\Wass_2 (\pi^{(1)}, \pi^{(2)}) = 0$ and therefore
$\pi^{(1)} = \pi^{(2)}$.


\subsubsection{First moment under the stationary distribution}
 
Let $\theta_t \sim \pi_\stepsize$. Consider a stationary chain
$(\theta_t)_{t \geq 0}$ starting at $\theta_0$. By stationarity, we
have $\law(\theta_{t + 1}) = \law(\theta_t) = \pi_\stepsize$. Note
that $\theta_{t + 1} = \theta - \stepsize (\Amat_{t + 1} \theta_t -
\bvec_{t + 1})$, taking expectations, we have:
  \begin{align*}
    \Exs (\theta_t) = \Exs (\theta_{t + 1}) = \Exs \left( \theta_t -
    \stepsize (\Amat_{t + 1} \theta_t - \bvec_t) \right) = \Exs \left(
    \theta_t - \stepsize \Exs (\Amat_{t + 1} \theta_t - \bvec_{t + 1}
    | \filtration_t) \right) = \Exs \left(\theta_t - \stepsize (A
    \theta_t - b) \right).
  \end{align*}
  Therefore, we have $\Amatbar \Exs_{\pi_\stepsize} (\theta) - \bvec =
  0$, which implies $\theta = \thetastar$ since $\Amatbar$ is
  non-degenerate.

  
\subsubsection{Second moment under the stationary distribution}

Let $\theta_t \sim \pi_\stepsize$. Consider a stationary chain
$(\theta_t)_{t \geq 0}$ starting at $\theta_0$. By stationarity, we
have $\law(\theta_{t + 1}) = \law(\theta_t) = \pi_\stepsize$. Note
that $\theta_{t + 1} = \theta - \stepsize (\Amat_{t + 1} \theta_t -
\bvec_{t + 1})$, and consequently, we have:
  \begin{align*}
    (\theta_{t + 1} - \thetastar) = (I - \stepsize \Amatbar) (\theta_t -
    \thetastar) - \stepsize \NoiseAplain_{t + 1} (\theta_t -
    \thetastar) + \stepsize \noisebplain_{t + 1} - \stepsize
    \NoiseAplain_{t + 1} \thetastar.
  \end{align*}
  As we have shown,
  $\Exs_{\pi_\stepsize} \theta = \thetastar$. Let $r_t \mydefn
  \theta_t - \thetastar$, taking conditional second moments of both
  sides of the equation, we obtain:
\begin{multline*}
\Exs \left( r_{t + 1} r_{t + 1}^\top \mid \filtration_{t} \right) =
(\IdMat - \stepsize \Amatbar) r_t r_t^\top (\IdMat - \stepsize
\Amatbar)^\top + \stepsize^2 \Exs (\NoiseAplain_{t + 1} r_t r_t^\top
\NoiseAplain_{t + 1}^\top |\filtration_t) \\
+ \stepsize^2 \Exs \left( \NoiseAplain_{t + 1} r_t (\noisebplain_{t +
  1} + \NoiseAplain_{t + 1} \thetastar )^\top + (\noisebplain_{t + 1}
+ \NoiseAplain_{t + 1} \thetastar) r_t^\top \NoiseAplain_{t + 1}^\top
\mid \filtration_t \right) \\
+ \stepsize^2 \Exs ((\noisebplain_{t + 1} + \NoiseAplain_{t + 1}
\thetastar) (\noisebplain_{t + 1} + \NoiseAplain_{t + 1}
\thetastar)^\top \mid \filtration_t).
\end{multline*}
Let $\Lambda \mydefn \Exs_{\pi_\stepsize} \left( r_t
r_t^\top\right)$. Taking the expectation of both sides, note that by
Assumption~\ref{assume-indp}:
\begin{align*}
    &\Exs \left( \NoiseAplain_{t + 1} r_t \noisebplain_{t + 1}^\top
  \mid \filtration_t \right) = 0, \quad \Exs ((\noisebplain_{t + 1} +
  \NoiseAplain_{t + 1} \thetastar) (\noisebplain_{t + 1} +
  \NoiseAplain_{t + 1} \thetastar)^\top \mid \filtration_t) =
  \Sigma_\xi + \Exs (\NoiseA \thetastar \thetastar^\top \NoiseA^\top
  ),\\ &\Exs \left( \NoiseAplain_{t + 1} r_t (\NoiseAplain_{t + 1}
  \thetastar)^\top \right) = \Exs \left( \NoiseA \otimes \NoiseA
  \right) \cdot \Vec (\Exs (r_t) \thetastar^\top) = \Exs \left(
  \NoiseA \otimes \NoiseA \right) \cdot \Vec (0 \cdot \thetastar^\top)
  = 0.
\end{align*}
Simplifying this equation yields
\begin{align*}
\Lambda = (\IdMat - \stepsize \Amatbar) \Lambda (\IdMat - \stepsize
\Amatbar)^\top + \stepsize^2 \Exs (\NoiseA \Lambda \NoiseA^\top) +
\stepsize^2 \Sigma_\xi + \stepsize^2 \Exs (\NoiseA \thetastar
\thetastar^\top \NoiseA^\top ),
\end{align*}
which means:
\begin{align*}
\Amatbar \Lambda + \Lambda \Amatbar^\top = \stepsize \Amatbar \Lambda
\Amatbar^\top + \stepsize \Exs (\NoiseA \Lambda \NoiseA^\top) + \stepsize
\SigStar.
\end{align*}

By flattening the tensors, we can write the equation in a
matrix-vector form:
\begin{align*}
\left(\IdMat \otimes \Amatbar + \Amatbar^\top \otimes \IdMat - \stepsize
\Amatbar \otimes \Amatbar - \stepsize \Exs (\NoiseA \otimes \NoiseA) \right)
\Vec (\Lambda) = \stepsize \Vec (\SigStar),
\end{align*}
where $\oplus$ denotes the Kronecker sum and $\otimes$ denotes the
Kronecker product.

To provide an upper bound on the trace of the solution to this matrix equation, which is the covariance under the stationary distribution, we note that in the proof of Lemma~\ref{lemma:simple-l2-estimate},
we use a contraction inequality:
\begin{align*}
  \Exs \vecnorm{U^{-1} r_{t + 1}}{2}^2 \leq (1 - \spectralgap
  \stepsize) \Exs \vecnorm{U^{-1} r_t}{2}^2 + \stepsize^2
  \opnorm{U^{-1}}^2 (v_A^2 \vecnorm{\thetastar}{2}^2 + v_b^2 d).
    \end{align*}
If $\theta_t \sim \pi_\stepsize$, we have $\theta_{t + 1} \sim
\pi_\stepsize$, and hence
\begin{align*}
  \Exs_{\pi_\stepsize} \vecnorm{U^{-1} (\theta - \thetastar)}{2}^2
  \leq (1 - \spectralgap \stepsize) \Exs_{\pi_\stepsize}
  \vecnorm{U^{-1} (\theta - \thetastar)}{2}^2 + \stepsize^2
  \opnorm{U^{-1}}^2 (v_A^2 \vecnorm{\thetastar}{2}^2 + v_b^2 d),
    \end{align*}
which implies the claimed bound:
\begin{align*}
    \Exs_{\pi_\stepsize} \vecnorm{\theta - \thetastar}{2}^2 \leq \frac{\stepsize}{\spectralgap} \asymconditioning^2 (U) (v_A^2 \vecnorm{\thetastar}{2}^2 + v_b^2 d).
\end{align*}


\subsection{Proof of Lemma~\ref{lemma:mixing-hurwitz}}
\label{AppLemMixingHurwitz}

Given two different starting points $x^{(i)} \in \real^d$ for $i = 1,
2$, let $\{ \theta_t^{(i)}\}_{t \geq 0}$ be the process starting at
$x^{(i)}$, and let the two processes to be driven by the same
sequences of noise variables $\noiseb$ and $\NoiseA$, so that
$\Amat_t^{(1)} = \Amat_t^{(2)}$ and $\bvec_t^{(1)} = \bvec_t^{(2)}$
almost surely.
  
By Lemma~\ref{LemHurwitz}, we can write $\Amatbar = U D^\top U^{-1}$,
such that $D + D^\hc \succeq \spectralgap \IdMat$. Introducing the
shorthand $r_t \mydefn \theta_t^{(1)} - \theta_t^{(2)}$, some algebra
leads to the recursive relation
\begin{multline*}
  r_{t + 1} = \theta_{t + 1}^{(1)} - \theta_{t + 1}^{(2)} =
  \theta_t^{(1)} - \stepsize \left( \Amatbar \theta_t^{(1)} - \bvec +
  \NoiseAplain_{t + 1} \theta_t^{(1)} - \noisebplain_{t + 1} \right) - \theta_t^{(2)}
  + \stepsize \left( \Amatbar \theta_t^{(2)} - \bvec + \NoiseAplain_{t +1 }
  \theta_t^{(2)} - \noisebplain_{t + 1} \right) \\
= (\IdMat - \stepsize \Amatbar - \stepsize \NoiseAplain_{t + 1}) r_t.
  \end{multline*}
Define the Lyapunov function $f(r) = \Exs \vecnorm{U^{-1} r}{2}^2$.
By Assumptions~\ref{assume-hurwitz} and~\ref{assume-second-moment},
note that $\rho(\Amatbar) = \sqrt{\opnorm{D^\hc D}}$ and
$\asymconditioning(\Umat) = \opnorm{\Umat} \opnorm{\Umat^{-1}}$, we
have:
     \begin{align*}
       &\Exs \vecnorm{U^{-1} r_{t + 1}}{2}^2 \\ &= \Exs \left(
       r_t^\hc (\IdMat - \stepsize \Amatbar - \stepsize
       \NoiseAplain_{t + 1})^\top (U^{-1})^\hc U^{-1} (\IdMat -
       \stepsize \Amatbar - \stepsize \NoiseAt) r_t \right)\\ & = \Exs
       \left( (U^{-1} r_t)^\hc (\IdMat - \stepsize D -
       \stepsize U^{-1} \NoiseAplain_{t + 1} U)^\hc (\IdMat -
       \stepsize D - \stepsize U^{-1} \NoiseAplain_{t + 1} U) (U^{-1} r_t)
       \right)\\ &= \Exs \vecnorm{(\IdMat - \stepsize D) U^{-1} r_t}{2}^2
       + \stepsize^2 \Exs \vecnorm{U \NoiseAplain_{t + 1} r_t}{2}^2 \\ &\leq
       \Exs \vecnorm{ U^{-1} r_t}{2}^2 - \stepsize \Exs
       (U^{-1}r_t)^\hc (D + D^\hc)
       (U^{-1}r_t) + \stepsize^2 \opnorm{D^\hc D} \Exs
       \vecnorm{ U^{-1} r_t}{2}^2 + \stepsize^2 \opnorm{U}^2 \Exs
       \vecnorm{ \NoiseAplain_{t + 1} r_t}{2}^2\\ & \leq \Exs
       \vecnorm{U^{-1} r_t}{2}^2 - 2 \stepsize \spectralgap \Exs
       \vecnorm{U^{-1} r_t}{2}^2 + \stepsize^2 \rho(\Amatbar)^2 \Exs
       \vecnorm{U^{-1} r_t}{2}^2 + \asymconditioning^2 (\Umat) v_A^2 \Exs
       \vecnorm{U^{-1} r_t}{2}^2.
     \end{align*}
 For $\stepsize \in \Big(0, \frac{\spectralgap}{\rho (\Amatbar)^2 +
   \asymconditioning (\Umat)^2 v_A^2} \Big)$, we have $\Exs
 \vecnorm{U^{-1} r_{t + 1}}{2}^2 \leq (1 - \stepsize \spectralgap)\Exs
 \vecnorm{U^{-1} r_t}{2}^2$ for any $t \geq 0$. Consequently, we have
 the coupling estimate:
 \begin{align*}
   \Exs \vecnorm{r_{T}}{2}^2 \leq \opnorm{U}^2 \vecnorm{U^{-1}
     r_T}{2}^2 \leq \opnorm{U}^2 e^{- \stepsize \spectralgap T}
   \vecnorm{U^{-1} r_0}{2}^2 \leq e^{- \stepsize \spectralgap T}
   \asymconditioning^2 (\Umat) \Exs \vecnorm{r_0}{2}^2,
 \end{align*}
 which completes the proof of the lemma.

\subsection{Proof of Lemma~\ref{lemma:linfty-estimate-and-contraction}}
\label{Appendix:linfty-estimates}
We first prove the almost-sure upper bounds on the iterates. Note that for $\theta_t \in [- \linftygap^{-1}, \linftygap^{-1}]^d$, we have the following sequence of inequalities almost surely:
\begin{multline*}
    \vecnorm{\theta_{t + 1}}{\infty} = \vecnorm{\theta_t - \stepsize (\Amat_t \theta_t - \bvec_t)}{\infty} \leq \vecnorm{(1 - \stepsize) \theta_t}{\infty} + \stepsize \vecnorm{(I_d - \Amat_t) \theta_t}{\infty} + \stepsize \vecnorm{\bvec_t}{\infty}\\
    \leq (1 - \stepsize) \vecnorm{\theta_t}{\infty} + \stepsize (1 - \linftygap) \vecnorm{\theta_t}{\infty} + \stepsize \leq (1 - \stepsize \linftygap) \linftygap^{-1} + \stepsize = \linftygap^{-1}.
\end{multline*}
The result then follows by induction.

We then prove the $\ell_\infty$ contraction bound. We take a synchronous coupling where the two processes use the same sequence of stochastic oracles. We have:
\begin{multline*}
    \vecnorm{\theta_{t + 1}^{(1)} - \theta_{t + 1}^{(2)}}{\infty} = \vecnorm{(I - \stepsize \Amat_t) (\theta_t^{(1)} - \theta_t^{(2)})}{\infty} \\
    \leq (1 - \stepsize) \vecnorm{\theta_t^{(1)} - \theta_t^{(2)}}{\infty} + \stepsize \vecnorm{(I - \Amat) (\theta_t^{(1)} - \theta_t^{(2)})}{\infty} \leq (1 - \stepsize \linftygap) \vecnorm{\theta_t^{(1)} - \theta_t^{(2)}}{\infty},
\end{multline*}
which proves the coupling bound.

\section{Proof of Lemma~\ref{lemma:theta-norm-hurwitz}}
\label{AppThetaNormHurwitz}

We decompose $\Amatbar$ in the form $\Amatbar = U D U^{-1}$ that is
guaranteed by Lemma~\ref{LemHurwitz}.  We study the dynamics of
$\vecnorm{U^{-1} (\theta_t - \thetastar)}{2}$. Defining the residual
term $r_t \mydefn \theta_t - \thetastar$, we observe that
\begin{align*}
& \vecnorm{U^{-1} r_{t + 1}}{2}^2 \\ &= (r_t - \stepsize (A +
  \NoiseAplain_{t + 1}) (r_t + \thetastar) - \stepsize \noisebplain_{t
    + 1})^\hc (U^{-1})^\hc U^{-1} (r_t - \stepsize (A +
  \NoiseAplain_{t + 1}) (r_t + \thetastar) - \stepsize \noisebplain_{t
    + 1})\\ &= (U^{-1} r_t)^\hc (I - \stepsize (D + D^\hc) +
  \stepsize^2 D^\hc D) (U^{-1} r_t) - 2 \stepsize \RealPart \left(
  (\NoiseAplain_{t + 1} (r_t + \thetastar) + \noisebplain_{t + 1})^\hc
  (U^{-1})^\hc (I - \stepsize D) U^{-1} r_t \right)\\ & \quad \quad +
  \stepsize^2 \vecnorm{U^{-1} (\NoiseAplain_{t + 1} r_t +
    \NoiseAplain_{t + 1} \thetastar + \noisebplain_{t +
      1})}{2}^2\\ &\leq (1 - \stepsize \spectralgap + \stepsize^2
  \rho^2 (\Amatbar)) \vecnorm{U^{-1} r_t}{2}^2 - 2 \stepsize \RealPart
  \left( (\NoiseAplain_{t + 1} (r_t + \thetastar) + \noisebplain_{t +
    1})^\hc (U^{-1})^\hc (I - \stepsize D) U^{-1} r_t \right)\\ &
  \quad \quad + 3 \stepsize^2 \opnorm{U^{-1}}^2 \left( \vecnorm{
    \NoiseAplain_{t + 1} r_t}{2}^2 + \vecnorm{\NoiseAplain_{t + 1}
    \thetastar}{2}^2 + \vecnorm{\noisebplain_{t + 1})}{2}^2 \right).
    \end{align*}
Telescoping this expression, for $\stepsize \in \Big(0,
\frac{\spectralgap}{\rho^2 (\Amatbar)} \Big)$, we have:
\begin{multline*}
  e^{\stepsize \spectralgap T} \vecnorm{U^{-1} r_T}{2}^2 \leq
  \vecnorm{U^{-1} r_0}{2}^2 - 2 \stepsize \underbrace{\sum_{t = 0}^{T
      - 1} e^{\stepsize \spectralgap t}\RealPart \left( (\NoiseAplain_{t +
      1} (r_t + \thetastar) + \noisebplain_{t + 1})^\hc
    (U^{-1})^\hc (I - \stepsize D) U^{-1} r_t
    \right)}_{\mydefn S_1 (T)} \\ + 3 \stepsize^2 \underbrace{\sum_{t
      = 0}^{T - 1} e^{\stepsize \spectralgap t} \opnorm{U^{-1}}^2
    \left(\vecnorm{ \NoiseAplain_{t + 1} r_t}{2}^2 + \vecnorm{\NoiseAplain_{t + 1}
      \thetastar}{2}^2 + \vecnorm{\noisebplain_{t + 1})}{2}^2 \right)}_{\mydefn
    S_2 (T)}.
    \end{multline*}
Note that the process $\{S_1(T)\}$ is a martingale and the process $\{
S_2(T) \}$ is non-decreasing.
    
Let us adopt $\Exs \sup \limits_{0 \leq t \leq T} \left(
e^{\spectralgap \stepsize t} \vecnorm{U^{-1} r_t}{2}^2
\right)^{\frac{p}{2}}$ as a Lyapunov function. By Young's inequality we obtain:
\begin{align*}
  \Exs \sup_{0 \leq t \leq T} \left( e^{\spectralgap \stepsize
    t} \vecnorm{U^{-1} r_t}{2}^2 \right)^{\frac{p}{2}} \leq
  3^{\frac{p}{2}} \Exs \vecnorm{U^{-1} r_0}{2}^p +
  6^{\frac{p}{2}} \stepsize^{\frac{p}{2}} \Exs \sup_{1 \leq t
    \leq T } |S_1 (t)|^{\frac{p}{2}} + 9^{\frac{p}{2}}
  \stepsize^p \Exs (S_2 (T))^{\frac{p}{2}}.
\end{align*}
We upper bound the two terms respectively.

\paragraph{Upper bound for $|S_1|$:}
Note that:
\begin{multline*}
  \abss{(\NoiseAplain_{t + 1} (r_t + \thetastar) + \noisebplain_{t +
      1})^\hc (U^{-1})^\hc (I -
    \stepsize D) U^{-1} r_t } \\ \leq \vecnorm{ (U^{-1}
    \NoiseAplain_{t + 1} r_t) + (U^{-1} \noisebplain_{t + 1}) + U^{-1}
    \NoiseAplain_{t + 1} \thetastar }{2} \cdot \opnorm{I - \stepsize
    D} \cdot \vecnorm{ U^{-1} r_t}{2}\\ \leq 2 \opnorm{U^{-1}}
  \left( \vecnorm{\NoiseAplain_{t + 1} r_t}{2} + \vecnorm{\noisebplain_{t +
      1}}{2} + \vecnorm{\NoiseAplain_{t + 1} r_t}{2} \right)
  \vecnorm{U^{-1} r_t}{2}.
\end{multline*}

Applying the Burkholder-Davis-Gundy inequality to the martingale $S_1
(t)$, we have:
\begin{align*}
  &\Exs \sup_{1 \leq t \leq T} |S_1 (t)|^{\frac{p}{2}} \leq \left(C
  p\right)^{\frac{p}{4}} \Exs \langle S_1 \rangle_T^{\frac{p}{4}}\\ &=
  \left(C p\right)^{\frac{p}{4}} \Exs \left( \sum_{t = 0}^{T - 1} e^{2
    \stepsize \spectralgap t} \abss{(\NoiseAplain_{t + 1} (r_t + \thetastar)
    + \noisebplain_{t + 1})^\hc (U^{-1})^\hc (I
    - \stepsize D) U^{-1} r_t }^2 \right)^{\frac{p}{4}}\\ &\leq
  \left(C p\right)^{\frac{p}{4}} \opnorm{U^{-1}}^{\frac{p}{2}} \Exs
  \left( \sum_{t = 0}^{T - 1} e^{2 \stepsize \spectralgap t} \left(
  \vecnorm{\NoiseAplain_{t + 1} r_t}{2}^2 \vecnorm{U^{-1} r_t}{2}^2 +
  (\vecnorm{\noisebplain_{t + 1}}{2}^2 + \vecnorm{\NoiseAplain_{t + 1}
    \thetastar}{2}^2 )\vecnorm{U^{-1} r_t}{2}^2\right)
  \right)^{\frac{p}{4}}.
    \end{align*}
By H\"{o}lder's inequality, we have:
\begin{multline*}
  \left( \sum_{t = 0}^{T - 1} e^{2 \stepsize \spectralgap t} \left(
  (\vecnorm{\NoiseAplain_{t + 1} r_t}{2}^2 + \vecnorm{\noisebplain_{t + 1}}{2}^2 +
  \vecnorm{\NoiseAplain_{t + 1} \thetastar}{2}^2) \vecnorm{U^{-1}
    r_t}{2}^2\right) \right)^{\frac{p}{4}} \\ \leq \left( \sum_{t =
    0}^{T - 1} e^{\frac{2p}{p - 4} \stepsize \spectralgap t}
  \right)^{\frac{p}{4} - 1} \left( 3 \sum_{t = 0}^{T - 1}
  (\vecnorm{\NoiseAplain_{t + 1} r_t}{2}^{\frac{p}{2}} + \vecnorm{\noisebplain_{t +
      1}}{2}^{\frac{p}{2}} + \vecnorm{\NoiseAplain_{t + 1}
    \thetastar}{2}^{\frac{p}{2}} ) \vecnorm{U^{-1}
    r_t}{2}^{\frac{p}{2}} \right).
    \end{multline*}
For the geometric series, we have $\left( \sum_{t = 0}^{T - 1}
e^{\frac{2p}{p - 4} \stepsize \spectralgap t} \right)^{ \frac{p}{4} -
  1} \leq \frac{1}{(\stepsize \spectralgap)^{\frac{p}{4} - 1}}
e^{\stepsize \spectralgap p T}$.
    
By Assumption~\ref{assume-noise-subgaussian}, we have:
\begin{align*}
  \Exs \vecnorm{\noisebplain_{t + 1}}{2}^{\frac{p}{2}} \leq p^{p \beta / 2}
  (\sigma_b \sqrt{d})^{p / 2}, \quad \Exs \vecnorm{\NoiseAplain_{t + 1}
    v}{2}^{\frac{p}{2}} \leq p^{p \alpha / 2} \sigma_A^{p / 2}
  \vecnorm{v}{2}^{p/ 2}.
\end{align*}
Putting together the pieces, we obtain:
\begin{multline*}
  \Exs \sup_{1 \leq t \leq T} |S_1 (t)|^{\frac{p}{2}} \leq
  \frac{(Cp)^{\frac{p}{4}} e^{\stepsize \spectralgap p T / 2}
  }{(\spectralgap \stepsize)^{\frac{p}{4}}} \sum_{t = 0}^{T - 1} \Big(
  p^{\frac{p \beta}{2}} (\sigma_b \sqrt{d})^{\frac{p}{2}}
  \opnorm{U^{-1}}^{\frac{p}{2}} \Exs \vecnorm{U^{-1}
    r_t}{2}^{\frac{p}{2}}\\ + p^{\frac{p \alpha}{2}} \sigma_A^{
    \frac{p}{2}} \asymconditioning (\Umat)^{\frac{p}{2}} \Exs
  \vecnorm{U^{-1} r_t}{2}^p + p^{\frac{p \alpha}{2}} \sigma_A^{
    \frac{p}{2}} \opnorm{U^{-1}}^{\frac{p}{2}} \Exs
  \vecnorm{\thetastar}{2}^p \Big).
\end{multline*}
    

\paragraph{Upper bounds on $S_2$:} By Young's inequality, we have:
\begin{align*}
(S_2 (T))^{\frac{p}{2}}& = \left( \sum_{t = 0}^{T - 1} e^{\stepsize
    \spectralgap t} \opnorm{U^{-1}}^2 \left( \vecnorm{\NoiseAplain_{t + 1}
    r_t}{2}^2 + \vecnorm{\noisebplain_{t + 1}}{2}^2 + \vecnorm{\NoiseAplain_{t + 1}
    \thetastar}{2}^2 \right) \right)^{\frac{p}{2}} \\
& \leq \opnorm{U^{-1}}^p \left[ \left( 3 \sum_{t = 0}^{T - 1}
    e^{\stepsize \spectralgap t} \vecnorm{\noisebplain_{t + 1}}{2}^2
    \right)^{\frac{p}{2}} + \left( 3 \sum_{t = 0}^{T - 1} e^{\stepsize
      \spectralgap t} \vecnorm{\NoiseAplain_{t + 1} r_t}{2}^2
    \right)^{\frac{p}{2}} + \left( 3 \sum_{t = 0}^{T - 1} e^{\stepsize
      \spectralgap t} \vecnorm{\NoiseAplain_{t + 1} \thetastar}{2}^2
    \right)^{\frac{p}{2}} \right].
    \end{align*}
By H\"{o}lder's inequality, we obtain:
\begin{align*}
  \left( \sum_{t = 0}^{T - 1} e^{\stepsize \spectralgap t}
  \vecnorm{\noisebplain_{t + 1}}{2}^2 \right)^{\frac{p}{2}} &\leq \left(
  \sum_{t = 0}^{T - 1} e^{\frac{p }{p - 2}\stepsize \spectralgap t}
  \right)^{\frac{p}{2} - 1} \left( \sum_{t = 0}^{T - 1}
  \vecnorm{\noisebplain_{t + 1}}{2}^p \right),\\ \left( \sum_{t = 0}^{T - 1}
  e^{\stepsize \spectralgap t} \vecnorm{\NoiseAplain_{t + 1}
    \thetastar}{2}^2 \right)^{\frac{p}{2}} &\leq \left( \sum_{t =
    0}^{T - 1} e^{\frac{p }{p - 2}\stepsize \spectralgap t}
  \right)^{\frac{p}{2} - 1} \left( \sum_{t = 0}^{T - 1}
  \vecnorm{\NoiseAplain_{t + 1} \thetastar}{2}^p \right),\\ \left( \sum_{t =
    0}^{T - 1} e^{\stepsize \spectralgap t} \vecnorm{\NoiseAplain_{t + 1}
    r_t}{2}^2 \right)^{\frac{p}{2}} &\leq \left( \sum_{t = 0}^{T - 1}
  e^{\frac{p }{p - 2}\stepsize \spectralgap t} \right)^{\frac{p}{2} -
    1} \left( \sum_{t = 0}^{T - 1} \vecnorm{\NoiseAplain_{t + 1} r_t}{2}^p
  \right).
\end{align*}
 For the geometric series, it is easy to see that $\left( \sum_{t =
   0}^{T - 1} e^{\frac{p }{p - 2}\stepsize \spectralgap t}
 \right)^{\frac{p}{2} - 1} \leq \frac{1}{(\stepsize
   \spectralgap)^{\frac{p}{2} - 1}} e^{\stepsize \spectralgap p T /
   2}$.

This yields:
\begin{align*}
\Exs (S_2 (T))^{\frac{p}{2}} \leq \opnorm{U^{-1}}^p
\frac{3^{\frac{p}{2}}}{(\stepsize \spectralgap)^{\frac{p}{2} - 1}}
e^{\stepsize \spectralgap p T} \left( \sum_{t = 0}^{T - 1} \Exs
\vecnorm{\noisebplain_{t + 1}}{2}^p + \sum_{t = 0}^{T - 1} \Exs
\vecnorm{\NoiseAplain_{t + 1} r_t}{2}^p + \sum_{t = 0}^{T - 1} \Exs
\vecnorm{\NoiseAplain_{t + 1} \thetastar}{2}^p \right).
\end{align*}
By Assumption~\ref{assume-noise-subgaussian}, we have:
\begin{align*}
  \Exs \vecnorm{\noisebplain_{t + 1}}{2}^{p} \leq p^{p \beta} (\sigma_b
  \sqrt{d})^{p}, \quad \Exs \vecnorm{\NoiseAplain_{t + 1} v}{2}^{p}
  \leq p^{p \alpha} \sigma_A^{p} \vecnorm{v}{2}^{p}.
\end{align*}
Putting the pieces together, we have:
\begin{align*}
  \Exs (S_2(T))^p \leq \frac{e^{\stepsize \spectralgap p T /
      2}}{(\stepsize \spectralgap)^{\frac{p}{2}}} \opnorm{U^{-1}}^p
  \left( T p^{p \beta} (\sigma_b \sqrt{d})^{p} + T p^{p \alpha}
  (\sigma_A \vecnorm{\thetastar}{2})^p + p^{p \alpha} \sigma_A^p
  \opnorm{U}^{p} \sum_{t = 0}^{T - 1} \Exs \vecnorm{U^{-1} r_t}{2}^p
  \right).
\end{align*}

Defining $H_T \mydefn e^{- \frac{\spectralgap \stepsize T}{2}} \left(
\Exs \sup_{0 \leq t \leq T} \left( e^{\spectralgap \stepsize t}
\vecnorm{U^{-1} r_t}{2}^2 \right)^{\frac{p}{2}}
\right)^{\frac{2}{p}}$, clearly we have the upper bound $(\Exs
\vecnorm{U^{-1}r_T}{2}^p)^{\frac{p}{2}} \leq H_T$. By the
decomposition of the Lyapunov function, we get:
\begin{align*}
  H_T \leq (\Exs \vecnorm{U^{-1}r_0}{2}^p)^{\frac{2}{p}} + 6 \stepsize
  e^{- \frac{\stepsize \stepsize T}{2}} (\Exs \sup_{1 \leq t \leq T}
  |S_1 (t)|^{\frac{p}{2}})^{\frac{2}{p}} + 6 \stepsize^2 e^{-
    \frac{\stepsize \stepsize T}{2}}(\Exs S_2
  (T)^{\frac{p}{2}})^{\frac{2}{p}}.
\end{align*}
Based on the upper bounds for $S_1$ and $S_2$, we have
\begin{align*}
  \stepsize^2 e^{- \frac{\stepsize \stepsize T}{2}} (\Exs S_2
  (T)^{\frac{p}{2}})^{\frac{2}{p}} &\leq C
  \frac{\stepsize}{\spectralgap} \left(\opnorm{U^{-1}}^2
  T^{\frac{2}{p}} (p^{2 \beta} \sigma_b^2 d + p^{2 \alpha} \sigma_A^2
  \vecnorm{\thetastar}{2}^2) + p^{2 \alpha} \asymconditioning^2 (\Umat)
  \sigma_A^2 \left( \sum_{t = 0}^{T - 1} H_t^{\frac{p}{2}}
  \right)^{\frac{2}{p}} \right),\\
  \stepsize e^{- \frac{\stepsize \stepsize T}{2}} (\Exs \sup_{1 \leq t
    \leq T} |S_1 (t)|^{\frac{p}{2}})^{\frac{2}{p}} & \leq C
  \sqrt{\frac{p \stepsize}{\spectralgap}} \left( \sum_{t = 0}^{T - 1}
  ((p^{\beta} \sigma_b \sqrt{d} + p^\alpha \sigma_A
  \vecnorm{\thetastar}{2}) \opnorm{U^{-1}} H_t)^{\frac{p}{4}} +
  (p^\alpha \sigma_A \asymconditioning(\Umat) H_t)^{\frac{p}{2}}
  \right)^{\frac{2}{p}}.
\end{align*}
Letting $R_T \mydefn \sup_{0 \leq t \leq T} H_t$, and noting that the upper
bounds above are non-decreasing in $T$, we have:
\begin{multline*}
  R_T \leq H_0 + C \frac{\stepsize}{\spectralgap}
  T^{\frac{2}{p}} \left(\opnorm{U^{-1}}^2 (p^{2 \beta}
  \sigma_b^2 d + p^{2 \alpha} \vecnorm{\thetastar}{2}^2)+ p^{2
    \alpha} \asymconditioning^2 (\Umat) \sigma_A^2 R_T \right)\\ + C
  \sqrt{\frac{p \stepsize}{\spectralgap}} T^{\frac{2}{p}} \left(
  \opnorm{U^{-1}} (p^\beta \sigma_b \sqrt{d} + p^\alpha \sigma_A
  \vecnorm{\thetastar}{2}) \sqrt{R_T} + p^\alpha \sigma_A
  \asymconditioning (\Umat) R_T \right).
\end{multline*}
Take $p \geq 2 \log T$ and $\stepsize \leq \frac{\spectralgap}{18 C^2
  e^2 p^{2 \alpha + 1} \asymconditioning^2 (\Umat) \sigma_A^2}$, we obtain
that:
\begin{align*}
  R_T \leq H_0 + C e\frac{\stepsize}{\spectralgap}
  \opnorm{U^{-1}}^2(p^\beta \sigma_b \sqrt{d} + p^\alpha \sigma_A
  \vecnorm{\thetastar}{2})^2 + C e \sqrt{\frac{p
      \stepsize}{\spectralgap}} \opnorm{U^{-1}} (p^\beta \sigma_b
  \sqrt{d} + p^\alpha \sigma_A \vecnorm{\thetastar}{2}) + \frac{1}{2}
  R_T,
\end{align*}
and therefore:
\begin{align*}
  \max_{0 \leq t \leq T} (\Exs \vecnorm{r_t}{2}^p)^{\frac{2}{p}} \leq
  \opnorm{U^{-1}}^2 R_T \lesssim \asymconditioning^2 (\Umat) \left( (\Exs
  \vecnorm{\theta_0 - \thetastar}{2}^p)^{\frac{2}{p}} +
  \frac{\stepsize}{\spectralgap} (p^{2 \beta + 1} \sigma_b^2 d + p^{2
    \alpha + 1} \sigma_A^2 \vecnorm{\thetastar}{2}^2 ) \right).
\end{align*}
Thus, we have completed the proof of
Lemma~\ref{lemma:theta-norm-hurwitz}.


\section{Proof of Lemma~\ref{lemma-bound-on-theta-critical}}
\label{subsubsec:proof_of_lemma_lemma-bound-on-theta-critical}

By Assumption~\ref{assume-critical}, the matrix $\Amatbar$ is
diagonalizable. Accordingly, we can write $\Amatbar = U D U^{-1}$, and
the remaining part of Assumption~\ref{assume-critical} implies that $D
+ D^\hc \succeq 0$.
  
We use the function $f(\theta) = \vecnorm{U^{-1} (\theta -
  \thetastar)}{2}^2$ as a Lyapunov function. From the process
dynamics~\eqref{eq:lsa}, we can write
\begin{align*}
U^{-1} (\theta_{t + 1} - \thetastar) & = U^{-1} (\IdMat - \stepsize
\Amatbar) (\theta_t - \thetastar) + \stepsize U^{-1} \NoiseAplain_{t + 1}
(\theta_t - \thetastar) + \stepsize U^{-1} \noisebplain_{t + 1} -
\stepsize U^{-1} \NoiseAplain_{t + 1} \thetastar.
\end{align*}
Using this decomposition, we can write
\begin{align*}
\Exs[ \vecnorm{U^{-1} (\theta_{t + 1} - \thetastar)}{2}^2] & =
\Term_1 + \stepsize^2 \Term_2 + 2 \stepsize \Term_3,
\end{align*}
where
\begin{subequations}
  \begin{align}
\Term_1 & \defn \Exs \vecnorm{ U^{-1} (I - \stepsize \Amatbar) (\theta_t
  - \thetastar)}{2}^2 \\
\Term_2 & \defn \Exs \vecnorm{ U^{-1} (\NoiseAplain_{t + 1} (\theta_t
  - \thetastar) + \noisebplain_{t + 1} - \NoiseAplain_{t + 1}
  \thetastar) }{2}^2 \\
\Term_3 & \defn
    \Exs \left( \inprod{ U^{-1} (I - \stepsize \Amatbar) (\theta_t -
      \thetastar)}{U^{-1} (\NoiseAplain_{t + 1} (\theta_t -
      \thetastar) + \noisebplain_{t + 1} - \NoiseAplain_{t + 1}
      \thetastar)} \right).
  \end{align}
\end{subequations}
We upper bound each these three terms in succession.


\paragraph{Bounding $\Term_1$:}

Using Assumption~\ref{assume-critical}, we have:
\begin{align*}
\Term_1 &= \Exs (U^{-1}(\theta_t - \thetastar))^\hc \left( \IdMat - 2
\stepsize \left( U^{-1} \Amatbar U + (U^{-1} \Amatbar U)^\hc \right) +
\stepsize^2 (U^{-1} \Amatbar U)^\hc (U^{-1} \Amatbar U) \right) U^{-1}
(\theta_t - \thetastar) \\
& \leq \Exs \vecnorm{U^{-1} (\theta_t - \thetastar)}{2}^2 + \stepsize
\rho^2 (\Amatbar) \Exs \vecnorm{U^{-1} (\theta_t - \thetastar)}{2}^2.
\end{align*}

\paragraph{Bounding $\Term_2$:}
By Young's inequality and Assumption~\ref{assume-second-moment}, we
find that
\begin{align*}
\Term_2 & = \Exs \vecnorm{ U^{-1} (\NoiseAplain_{t + 1} (\theta_t -
  \thetastar) + \noisebplain_{t + 1} - \NoiseAplain_{t + 1}
  \thetastar) }{2}^2\\ &\leq 3 \opnorm{U^{-1}}^2 \Exs \left(
\vecnorm{\NoiseAplain_{t + 1} (\theta_t - \thetastar)}{2}^2 +
\vecnorm{\noisebplain_{t + 1}}{2}^2 + \vecnorm{\NoiseAplain_{t + 1}
  \thetastar}{2}^2 \right)\\ &\leq 3 \opnorm{U^{-1}}^2 \left(
\opnorm{U}^2 v_A^2 \Exs \vecnorm{U^{-1} (\theta_t - \thetastar)}{2}^2
+ v_b^2 d + v_A^2 \vecnorm{\thetastar}{2}^2 \right).
\end{align*}


\paragraph{Bounding $\Term_3$:}
In this case, we have
\begin{align*}
  \Term_3 = \Exs \left( \inprod{ \Umat^{-1} (\IdMat - \stepsize \Amatbar)
    (\theta_t - \thetastar)}{\Umat^{-1} \Exs \left(\NoiseAplain_{t +
      1} (\theta_t - \thetastar) + \noisebplain_{t + 1} -
    \NoiseAplain_{t + 1} \thetastar \mid \filtration_{t} \right)}
  \right) = 0.
\end{align*}
This yields:
\begin{align*}
  \Exs \vecnorm{U^{-1} (\theta_{t + 1} - \thetastar)}{2}^2 \leq (1 +
  \stepsize^2 \rho^2 (\Amatbar) + 3 \stepsize^2 \asymconditioning^2 (\Umat)
  v_A^2) \Exs \vecnorm{U^{-1} (\theta_t - \thetastar)}{2}^2 + 3
  \opnorm{U^{-1}}^2 (v_b^2 d + v_A^2 \vecnorm{\thetastar}{2}^2).
\end{align*}
Solving the recursion, for $\stepsize \leq \frac{1}{(\rho (\Amatbar) + 3
  \asymconditioning(\Umat)  v_A) \sqrt{T}}$, we obtain:
\begin{align*}
  &\Exs \vecnorm{U^{-1} (\theta_{T} - \thetastar)}{2}^2 \\ &\leq \exp
  \left( \stepsize^2 T (\rho^2 (\Amatbar) + 3 \asymconditioning^2 (\Umat) v_A^2)
  \right) \Exs \vecnorm{U^{-1} (\theta_0 - \thetastar)}{2}^2 \\ &\quad
  \quad + 3\stepsize^2 \opnorm{U^{-1}}^2 (v_b^2 d + v_A^2
  \vecnorm{\thetastar}{2}^2) \sum_{t = 0}^{T - 1} \exp \left(
  \stepsize^2 t (\rho^2 (\Amatbar) + 3 \asymconditioning^2 (\Umat) v_A^2)
  \right)\\ &\leq e \left( \Exs \vecnorm{U^{-1} (\theta_0 -
    \thetastar)}{2}^2 + 3 \stepsize^2 T \opnorm{U^{-1}}^2 (v_b^2 d +
  v_A^2 \vecnorm{\thetastar}{2}^2)\right).
\end{align*}
Noting that $\vecnorm{\theta_T - \thetastar}{2} \leq \opnorm{U}\cdot
\vecnorm{U^{-1} (\theta_T - \thetastar)}{2}$, we obtain the final
result.

\section{Concentration inequalities involving metric ergodocity}
\label{AppMetric}

In this section, we state and prove two concentration inequalities
that play an important role in our analysis.  We first state these
results and then prove them.

\begin{lemma}
\label{lemma:ergodic-concentration}
Under Assumption~\ref{assume-indp}, Assumption~\ref{assume-hurwitz} and
Assumption~\ref{assume-second-moment}, for given $T > 0$, if for any
$\delta > 0$, there exists $R (\delta), r (\delta) > 0$ such that:
\begin{itemize}
\item $\Prob \left( \max_{0\leq t \leq T} \vecnorm{U^{-1} \theta_t}{2}
  > R(\delta) \right) < \delta$.
\item $\Prob \left( \max_{0 \leq t \leq T} \vecnorm{U^{-1} (\NoiseAplain_{t
    + 1} \theta_t - \noisebplain_{t + 1})}{2} > r(\delta) \right) < \delta$,
\end{itemize}
then, for any matrix $L \in \real^{d \times d}$ and any $\delta \in
\left( 0, (T^2 \opnorm{L}^2 \max_{t \leq T}\Exs
\vecnorm{\theta_t}{2}^4)^{-1} \right)$, we have:
\begin{align*}
  \Prob \left( \abss{\frac{1}{T} \sum_{t = 1}^T (\theta_t^\top L
    \theta_t - \Exs \theta_t^\top L \theta_t )} > C \opnorm{L}
  \opnorm{U}^2 \frac{ R (\delta) r (\delta)}{\spectralgap}
  \left(\sqrt{\frac{\log \delta^{-1}}{T}} + \frac{ \log \delta^{-1}}{
    T} \right) \right) \leq 3 \delta.
\end{align*}
\end{lemma}

Our second lemma is a variant of the first, in which we replace the
second-moment condition (Assumption~\ref{assume-second-moment}) on the
noise variables with a stronger tail condition
(Assumption~\ref{assume-noise-subgaussian}).
\begin{lemma}
\label{LemErgconcTwo}
Under Assumption~\ref{assume-indp}, Assumption~\ref{assume-hurwitz}
and Assumption~\ref{assume-noise-subgaussian}, for a given initial
point $\theta_0$, for a matrix $L$ and given $\delta > 0, T > \log
\delta^{-1}$, if $\stepsize < \frac{\spectralgap}{\rho^2(\Amatbar) + C
  \asymconditioning^2 (\Umat) \sigma_A^2 \log^{2 \alpha + 1} (T d /
  \delta)}$, with probability $1 - \delta$, we have:
\begin{align*}
  \abss{\frac{1}{T} \sum_{t = 1}^T (\theta_t^\top L \theta_t - \Exs
    \theta_t^\top L \theta_t )} \leq C \opnorm{L}
  \frac{\asymconditioning^2 (\Umat)}{\spectralgap} B \left( \sigma_A (B +
  \vecnorm{\thetastar}{2})\log^{\alpha}\frac{T}{\delta} + \sigma_b
  \sqrt{d} \log^{\beta}\frac{T}{\delta} \right) \sqrt{\frac{\log
      \delta^{-1}}{T}},
    \end{align*}
where $B \mydefn \vecnorm{\theta_0 - \thetastar}{2} +
\frac{\stepsize}{\spectralgap} (\sigma_b \sqrt{d} \log^{\beta +
  1/2} \frac{T}{\delta} + \sigma_A \vecnorm{\thetastar}{2}
    \log^{\alpha + 1/2} \frac{T}{\delta} ) $.
\end{lemma}


\subsection{Proof of Lemma~\ref{lemma:ergodic-concentration}}

In order to prove this lemma, we make use of the following known result
due to Joulin and Ollivier~\cite{joulin2010curvature}:
\begin{proposition}[Theorem 4~\cite{joulin2010curvature}, special case]
\label{prop:joulin-ollivier}
Let $(X_t)_{t \geq 1}$ be a discrete-time Markov chain with transition
kernel $P$, defined on a space $\mathcal{X}$ equipped with the metric $d
(\cdot, \cdot)$. Assume that $\forall x, y \in \mathcal{X}, ~\Wass_{1,
  d} (P_x, P_y) \leq (1 - \kappa) d (x, y)$ for some $\kappa >
0$. Assume furthermore that $\sigma_\infty \mydefn \sup_{x \in
  \mathcal{X}} \mathrm{diam} (\mathrm{supp} (P_x))$. For any function
$f$ that is $1$-Lipschitz on $\mathcal{X}$ with respect to $d (\cdot,
\cdot)$, given a trajectory $(X_t)_{1 \leq t \leq T}$ of the Markov
chain, we have:
\begin{align*}
  \Prob \left( \abss{\frac{1}{T} \sum_{t = 1}^T (f (X_t) - \Exs
    f(X_t)) } > r\right) \leq \begin{cases} 2 \exp \left( - \frac{r^2
      T}{32} \cdot \frac{\kappa^2}{\sigma_\infty^2} \right) & r <
    \frac{4 \sigma_\infty}{3 \kappa}\\ 2 \exp \left( - \frac{r \kappa
      T}{12 \sigma_\infty} \right) & r \geq \frac{4 \sigma_\infty}{3
      \kappa}
  \end{cases}.
\end{align*}
\end{proposition}
Proposition~\ref{prop:joulin-ollivier} requires bounded noise and
global Lipschitzness, neither of which is satisfied by the process
$\theta_t$ with a quadratic function $f$.  In order to circumvent this
limitation, we use a standard truncation argument.

Under the assumptions of Lemma~\ref{lemma:ergodic-concentration}, for
any $\delta > 0$, define a stopping time
\begin{align*}
\tau(\delta) \mydefn \inf \left\{t \geq 1: \vecnorm{U^{-1}
  \theta_t}{2} > R (\delta) ~\text{or}~ \vecnorm{U^{-1}(\NoiseAt
  \theta_t - \noisebplain_t)}{2} > r (\delta) \right\}.
\end{align*}
Let $A = U D U^{-1}$ be its eigendecomposition. By the proof of
Lemma~\ref{lemma:mixing-hurwitz}, when $\stepsize <
\frac{\spectralgap}{2(\rho^2 (\Amatbar) + \asymconditioning^2(\Umat)
  v_A^2)}$, the Markov process $(U^{-1} \theta_t)_{t \geq 0}$
satisfies:
\begin{align*}
  \Wass_1 (P_x, P_y) \leq \Wass_2 (P_x, P_y) \leq (1 - \stepsize
  \spectralgap/2) \vecnorm{x - y}{2}, \quad \forall x, y \in \real^d.
\end{align*}
We define a killed Markov process $\vartheta_t \mydefn U^{-1}
\theta_t$ for $t < \tau (\delta)$, which gets killed at time $\tau
(\delta)$.  The one-step transition of the process $\vartheta_t$ is
defined as $\vartheta_t \mapsto \vartheta_t - \stepsize U^{-1} (A U
\vartheta_t - b) - U^{-1} (\NoiseAt U \vartheta_t - \noisebplain_t)$,
whose support has a diameter bounded by $2 \stepsize r(\delta)$ before
being killed.  Note that the Wasserstein contraction property remains
true for the killed process. The assumptions in
Lemma~\ref{lemma:ergodic-concentration} guarantee that $\Prob
(\tau(\delta) \leq T) < 2 \delta$.  By definition, we have
$\vecnorm{\vartheta_t}{2} \leq R (\delta)$.  Finally, for the function
$f: \ball (0, R(\delta)) \rightarrow \real$ with $f (\vartheta)
\mydefn \vartheta^\top U^\top L U \vartheta$, we have:
\begin{align*}
  \vecnorm{\nabla f (\vartheta)}{2} \leq 2 \opnorm{L} \opnorm{U}^2
  \vecnorm{\vartheta}{2} \leq 2 \opnorm{U}^2 \opnorm{L} R(\delta).
\end{align*}
Applying Proposition~\ref{prop:joulin-ollivier}, for any $\varepsilon
> 0$, we obtain:
\begin{multline*}
  \Prob \left( \abss{\frac{1}{T} \sum_{t = 1}^T (\vartheta_t^\top
    U^\top L U \vartheta_t \bm{1}_{t < \tau (\delta)} - \Exs
    \vartheta_t^\top U^\top L U \vartheta_t)} > 2 \varepsilon
  \opnorm{L} \cdot \opnorm{U}^2 R(\delta) \right) \\
\leq \begin{cases} 2 \exp \left( - \frac{\varepsilon^2 T}{32} \cdot
  \frac{(\spectralgap)^2}{16 ( r (\delta))^2 } \right),& \varepsilon <
  \frac{16 r (\delta)}{ 3\spectralgap} \\
2 \exp \left( - \frac{\varepsilon \spectralgap T}{48 r (\delta)}
\right), & \varepsilon \geq \frac{16 r (\delta)}{ 3\spectralgap}.
    \end{cases}
\end{multline*}
On the event $\{T < \tau (\delta)\}$, we have $\vartheta_t = U^{-1}
\theta_t$ for $t = 1,2, \cdots, T$. It remains to bound the difference
between $\Exs \vartheta_t^\top U^\top L U \vartheta_t$ and $\Exs
\theta_t^\top L \theta_t$. Note that:
\begin{align*}
  |\Exs \vartheta_t^\top U^\top L U \vartheta_t - \Exs \theta_t^\top L
  \theta_t| &= |\Exs ( \theta_t^\top L \theta_t \bm{1}_{t < \tau}) -
  \Exs \theta_t^\top L \theta_t| \leq \opnorm{L} \Exs
  (\vecnorm{\theta_t}{2}^2 \bm{1}_{\tau < t})\\ &\leq \opnorm{L}
  \sqrt{\Exs (\vecnorm{\theta_t}{2}^4) \Exs (\bm{1}_{\tau < t}^2)}
  \leq \opnorm{L} \sqrt{\delta\Exs \vecnorm{\theta_t}{2}^4}.
\end{align*}
Putting together the pieces yields the claimed result.


\subsection{Proof of Lemma~\ref{LemErgconcTwo}}

The proof involves verifying the assumptions in
Lemma~\ref{lemma:ergodic-concentration}. For the high-probability
bound on $\max_{0 \leq t \leq T} \vecnorm{U^{-1} \theta_t}{2}$, we
note that by the proof of Lemma~\ref{lemma:theta-norm-hurwitz}, for $p
\geq 2\log T$ we have:
\begin{align*}
  \Exs \max_{0 \leq t \leq T} \vecnorm{U^{-1} \theta_t}{2}^p & \leq
  \sum_{t = 1}^T \Exs \vecnorm{U^{-1} \theta_t}{2}^p \\
& \leq T
  \opnorm{U^{-1}}^{p} \left( \vecnorm{\theta_0 - \thetastar}{2} +
  \frac{\stepsize}{\spectralgap} (\sigma_b \sqrt{d} p^{\beta + 1/2} +
  \sigma_A \vecnorm{\thetastar}{2} p^{\alpha + 1/2}) \right)^p.
    \end{align*}
Taking $p = C \log \frac{T}{\delta}$ for a universal constant $C > 0$ and
applying Markov inequality, we have:
\begin{align*}
  \Prob \left( \max_{0 \leq t \leq T} \vecnorm{U^{-1} \theta_t}{2} > B\right) < \delta.
\end{align*}

In order to verify the second condition, we note that by
Assumption~\ref{assume-noise-subgaussian}, conditionally on
$\filtration_t$, the Markov inequality yields:
\begin{align*}
  \Prob \left( \vecnorm{\NoiseAplain_{t + 1} \theta_t }{2} > \sigma_A
  \vecnorm{\theta_t}{2} \log^\alpha \delta^{-1} | \filtration_t
  \right) < \delta, \quad \Prob \left( \vecnorm{\noisebplain_{t + 1}}{2} >
  \sigma_b \sqrt{d} \log^\beta \delta^{-1} | \filtration_t \right) <
  \delta.
    \end{align*}
Combined with high probability bounds on $\theta_t$ and take union
bound over $t \in \{1,2, \cdots, T\}$, we obtain the final result.



\section{A concentration inequality for heavy-tailed martingales}

In this appendix, we state and prove a useful concentration inequality
for heavy-tailed martingales.

\begin{lemma}
\label{lemma:heavy-tail-concentration}
 For a (scalar) martingale difference sequence $(X_t: t \geq 1)$
 adapted to filtration $(\filtration_t)_{t \geq 0}$, if we have $
 \forall p \geq 2, ~\Exs (|X_t|^p | \filtration_{t - 1})^{\frac{1}{p}}
 \leq p^\discount \sigma$ almost surely for some $\discount, \sigma > 0$,
 for any $\delta > 0$, we have
 \begin{align*}
   \Prob \left( \abss{\frac{1}{T}\sum_{t = 1}^T X_t } > C_\discount
   \sigma \left( \sqrt{\frac{\log \delta^{-1}}{T}} +
   \frac{\log^{1 + \discount} T/\delta }{T} \right) \right) < \delta.
 \end{align*}
\end{lemma}

\begin{proof}
For a constant $M > 0$ which will be determined later, define
$\tilde{X}_t \mydefn X_t \bm{1}_{|X_t| \leq M}$ be the truncated
version of the process. By the Bernstein inequality for
martingales~\citep{freedman1975tail}, for any $K > 0$, we have:
\begin{align*}
  \forall \varepsilon > 0,~\Prob \left( \abss{\sum_{t = 1}^T
    \tilde{X}_t - \Exs (\tilde{X}_t \mid \filtration_{t - 1} )} >
  \varepsilon, ~\sum_{t = 1}^T \mathrm{var} \left( \tilde{X}_t |
  \filtration_{t - 1}\right) < K \right) \leq 2 \exp \left( -
  \frac{\varepsilon^2}{2K + 2 M \varepsilon / 3} \right).
    \end{align*}
On the other hand, note that for $M > (2e)^\discount \sigma$, we have
\begin{align*}
  \Prob \left( X_t \neq \tilde{X}_t | \filtration_{t - 1} \right) \leq
  \inf_{p \geq 2} \frac{p^{p \discount} \sigma^p }{M^p} = \exp \left( -
  \frac{\discount}{e} \left( \frac{M}{\sigma} \right)^{\frac{1}{\discount}}
  \right),
\end{align*}
and note that:
\begin{align*}
  \abss{ \Exs \left( \tilde{X}_t | \filtration_{t - 1} \right) }
  \leq \Exs \left( |X_t - \tilde{X}_t| \big| \filtration_{t - 1}
  \right) \leq 2 \int_M^{+ \infty} \exp \left( - \frac{\discount}{e}
  \left( \frac{z}{\sigma} \right)^{\frac{1}{\discount}} \right) dz
  \leq C_\discount \left( \frac{M}{\sigma} \right)^{1 -
    \frac{1}{\discount}} \exp \left( - \frac{\discount}{e} \left(
  \frac{M}{\sigma} \right)^{\frac{1}{\discount}} \right).
\end{align*}
For the conditional second moment, we have:
\begin{align*}
  \mathrm{var} (\tilde{X}_t | \filtration_{t - 1}) \leq \Exs
  (\tilde{X}_t^2 | \filtration_{t - 1}) \leq \Exs (X_t^2 |
  \filtration_{t - 1}) \leq 2^{2\gamma} \sigma^2, \quad \mathrm{a.s.}
\end{align*}
Choosing $K = 2^{2\gamma} \sigma^2 T$, we have:
\begin{align*}
  \forall \varepsilon > 0,~\Prob \left( \abss{\sum_{t = 1}^T
    \tilde{X}_t - \Exs (\tilde{X}_t | \filtration_{t - 1} )} >
  \varepsilon \right) \leq 2 \exp \left( -
  \frac{\varepsilon^2}{C_\gamma \sigma^2 T + 2 M \varepsilon / 3}
  \right)
\end{align*}

Putting together the pieces, we find that
\begin{align*}
  \Prob \left( \abss{\frac{1}{T} \sum_{t = 1}^T X_t} > C_\gamma \sigma
  \sqrt{\frac{\log \delta^{-1}}{T}} + \frac{M\log \delta^{-1} }{T} +
  C_\gamma \left(\frac{M}{\sigma} \right)^{1 - \frac{1}{\gamma}} e^{-
    \frac{\gamma}{e} (\frac{M}{\sigma})^{1/ \gamma}} \right) \leq
  \delta + T \exp \left( - \frac{\gamma}{e} \left( \frac{M}{\sigma}
  \right) \right)^{\frac{1}{\gamma}}.
\end{align*}
Setting $M = C_\gamma \sigma \log^\gamma (\tfrac{T}{\delta})$ yields
the claim.
\end{proof}


\section{Necessity of diagonalizable $\Amatbar$ in the critical case}
\label{AppDiagNeeded}

In this appendix, we demonstrate that the diagonalizability condition
in Assumption~\ref{assume-critical} cannot be removed.  More
precisely, we show that even in the case of deterministic observations
(i.e., $\Amat_t = \Amatbar$ and $\bvec_t = \bvec$ for all iterations
$t$), there is a choice of matrix $\Amatbar$ and initial vector
$\theta_0$ for which the Polyak-Ruppert iterates behave badly.
\begin{proposition}
\label{prop:bad-non-diag-matrix}
For any dimension $d \geq 2$ and given initial vector $\theta_0 = [0,
  0, \cdots,0, 1]^\top$, there exists a matrix $\Amatbar \in \complex^{d
  \times d}$ with $\min_{i \in [d]} \RealPart (\lambda_i(\Amatbar)) \geq
0$ and $\min_i |\lambda_i (\Amatbar)| \geq 1$ such that for any positive
step size $\stepsize$ and any iteration $T \geq 4$, the Polyak-Ruppert
averaged iterate satisfies the lower bound
\begin{align}
  \label{EqnBadBound}
  \vecnorm{\thetabar_T - \thetastar}{2} \geq \frac{1}{2}.
\end{align}
\end{proposition}

The proof is based on an explicit construction.  Consider the
$d$-dimensional matrix
\begin{align*}
  J_d \mydefn \begin{bmatrix} 0 & 1 & 0 & \cdots & 0 \\
    0 & 0 & 1 & \cdots & 0 \\
    & & \cdots & & \\
    0 & 0 & \cdots & 0 & 1 \\
    0 & 0 & \cdots & 0 & 0
  \end{bmatrix}.
\end{align*}
Define the matrix $\Amatbar = - i \IdMat - J_d$.  In this deterministic
setting, we have:
\begin{align*}
  \theta_T - \thetastar = (\IdMat - \stepsize \Amatbar)^T (\theta_0 -
  \thetastar) = ((1 + \stepsize i) \IdMat + \stepsize J_d)^T (\theta_0
  - \thetastar) = \sum_{\ell = 0}^{\min (d, T)} \stepsize^\ell (1 +
  \stepsize i)^{T - \ell} \binom{T}{\ell} J_d^\ell (\theta_0 -
  \thetastar) .
\end{align*}
Take $\thetastar = 0$.  Given our initialization $\theta_0 = [0, 0,
  \cdots,0, 1]^\top$.  for all $T \geq d - 1$, we have $\theta_T =
\sum_{\ell = 0}^{d - 1} \stepsize^\ell (1 + \stepsize i)^{T - \ell}
\binom{T}{\ell} e_{d - \ell}$, and consequently, we have:
\begin{align*}
  - (\thetabar_T - \thetastar) = \frac{1}{T} \sum_{t = 1}^T
  \sum_{\ell = 0}^{d - 1} \stepsize^\ell \binom{t}{\ell} e_{d - \ell}
  = \sum_{\ell = 0}^{d - 2} e_{d - \ell} \stepsize^\ell
  \frac{1}{T}\sum_{t = \ell}^T (1 + \stepsize i)^{t - \ell}
  \binom{t}{\ell}.
\end{align*}
Consider the coefficient in the $(d - 1)$-th coordinate, which
corresponds to the case with $\ell = 1$, we have:
\begin{align*}
  - e_{d - 1}^\hc (\thetabar_T - \thetastar) = \frac{\stepsize}{T} \sum_{t = 1}^T (1 + \stepsize i)^{t - 1} t = \left(- i + \frac{1}{T}\right) (1 + \stepsize i)^T +  \frac{i - 1}{T}
\end{align*}
Therefore, for $T \geq 4$, we have:
\begin{align*}
\vecnorm{\thetabar_T - \thetastar}{2} \geq |e_{d - 1}^\hc (\thetabar_T
- \thetastar)_{d - 1}) | \geq \abss{\left(i + \frac{1}{T}\right) (1 +
  \stepsize i)^T} - \frac{\sqrt{2}}{T}\geq (1 +
\stepsize^2)^\frac{T}{2} - \frac{\sqrt{2}}{T} \geq \frac{1}{2},
\end{align*}
which completes the proof.


\section{Eigenvalue computation for momentum SGD}
\label{AppSpectrumMomentum}

Since $\Amatbar$ is real symmetric and positive definite, it is
guaranteed to have a spectral decomposition of the form $\Amatbar = U
D U^{-1}$, where $U$ is a orthonormal matrix and $D = \diag \{
\lambda_i(\Amatbar) \}_{i=1}^d$.   Using this fact, we can write
\begin{align*}
\AmatTil &= \begin{bmatrix} U & 0 \\ 0 & U
  \end{bmatrix} \;
\begin{bmatrix} 
    0 & \IdMat \\ - D & \alpha \IdMat + \stepsize D
\end{bmatrix} \;
\begin{bmatrix}
    U & 0\\ 0 & U
\end{bmatrix}^{-1} \\
&= \left(
\begin{bmatrix}
    U & 0 \\ 0 & U
\end{bmatrix}
  P_0 \right) \mathrm{diag} \left(
  \left[ \begin{matrix}0&1\\ -\lambda_i&\alpha + \stepsize
      \lambda_i \end{matrix} \right] \right)_{i = 1}^d \left( \left[
    \begin{matrix}
    U &0\\
    0& U
    \end{matrix}
    \right] P_0 \right)^{-1},
\end{align*}
where $P_0$ is a permutation matrix which turns the order $(1, 2,
\cdots, 2d)$ into $(1, d + 1, 2, d+ 2, \cdots, d, 2d)$. It can be seen
that $P_0$ is orthonormal.

For $\alpha \in \real_+ \setminus \{2 \sqrt{\lambda_i} - \stepsize
\lambda_i\}_{i = 1}^d$, each $2 \times 2$ block has distinct
eigenvalues, which makes it diagonalizable. In particular, we have:
\begin{align*}
  \left[
    \begin{matrix}
      0&1\\-\lambda_i & \alpha + \stepsize \lambda_i
    \end{matrix}
    \right] = \left[\begin{matrix}
    \lambda_i&-\nu_i^+\\
    \lambda_i&-\nu_i^-
    \end{matrix}
    \right]\cdot
    \left[\begin{matrix}
    \nu_i^+&0\\0&\nu_i^-
    \end{matrix}
    \right] \cdot \left[\begin{matrix}
    \lambda_i&-\nu_i^+\\
    \lambda_i&-\nu_i^-
    \end{matrix}
    \right]^{-1},
\end{align*}
where $\nu_i^\pm = \frac{ (\alpha + \stepsize \lambda_i) \pm
  \sqrt{(\alpha + \stepsize \lambda_i)^2 - 4 \lambda_i}}{2}$.


\end{document}